\documentclass{article}

     \usepackage[preprint]{neurips_2025}

\usepackage[utf8]{inputenc} 
\usepackage[T1]{fontenc}    
\usepackage{amsmath,dsfont,enumitem}
\usepackage{amsfonts, amssymb, amsthm}
\usepackage{tocloft}

\usepackage{url}            
\usepackage{booktabs}       
\usepackage{nicefrac}       
\usepackage{microtype}      
\usepackage{xcolor}         
\usepackage{tcolorbox}
\usepackage{standalone} 
\usepackage{tikz}
\usetikzlibrary{decorations.pathreplacing,fit, backgrounds} 

\usepackage{graphicx}
\usepackage{mathrsfs}
\usepackage{tcolorbox} 
\usepackage{bm}
\usepackage{framed}

\usepackage{mathtools}
\usepackage{hyperref}       
\usepackage{url}            

\usepackage[toc,page]{appendix}
\usepackage{minitoc}

\usepackage{tocloft}
\usepackage{algorithm}
\usepackage{algorithmic}
\usepackage{pdfpages}

\usepackage{cleveref}
\usepackage{autonum}
\usepackage{cancel}
\usepackage{subcaption}

\definecolor{darkmidnightblue}{rgb}{0.0, 0.2, 0.4}
\definecolor{darkpowderblue}{rgb}{0.0, 0.2, 0.6}
\definecolor{dukeblue}{rgb}{0.0, 0.0, 0.61}
\definecolor{noonblue}{HTML}{e5eef7}
\definecolor{chromered}{HTML}{f14233}
\definecolor{midnightblue}{HTML}{0059b3}
\definecolor{darkgreen}{HTML}{0e6029}

\hypersetup{
    colorlinks = true,
    citecolor=blue,
    urlcolor=blue,
    breaklinks=true,
    linkcolor = dukeblue,
    linkbordercolor = {white},
}

\newcommand{\Var}{\operatorname{Var}}
\newcommand{\cov}{\operatorname{Cov}}

\renewcommand{\leq}{\leqslant}
\renewcommand{\geq}{\geqslant}
\renewcommand{\le}{\leqslant}

\newcommand{\argmin}{\mathop{\mathrm{arg}\,\mathrm{min}}}

\newcommand{\wt}{\widetilde}

\renewcommand{\ln}{\log}

\def\drift{\boldsymbol{b}}
\newcommand{\cond}{\,|\,}
\newcommand{\indep}{\perp\!\!\!\perp} 
\def\hat{\widehat}

\renewcommand{\top}{\mathsf{T}}

\usepackage[textwidth=2.0cm, textsize=small]{todonotes} 

\newtheorem{proposition}{Proposition}
\newtheorem{theorem}{Theorem}
\newtheorem*{theorem*}{Theorem}
\newtheorem{lemma}{Lemma}

\theoremstyle{definition}
\newtheorem{definition}{Definition}

\newtheorem{assumption}{Assumption}
\newtheorem{remark}{Remark}

\newtheorem*{consequence*}{Consequence}

\crefname{assumption}{assumption}{assumptions}
\Crefname{assumption}{Assumption}{Assumptions}

\newcommand{\brr}[1]{\left( {#1} \right)}
\newcommand{\brs}[1]{\left[ {#1} \right]}

\newcommand*\numcircled[1]{\raisebox{2pt}{\tikz[baseline=(char.base)]{
            \node[shape=circle,draw,inner sep=1.2pt] (char) {\tiny #1};}}}

\def\bX{\boldsymbol{X}}
\def\bY{\boldsymbol{Y}}
\def\bZ{\boldsymbol{Z}}
\def\bW{\boldsymbol{W}}
\def\bU{\boldsymbol{U}}
\def\bV{\boldsymbol{V}}
\def\bB{\boldsymbol{B}}
\def\bx{\boldsymbol{x}}
\def\by{\boldsymbol{y}}
\def\bz{\boldsymbol{z}}
\def\bw{\boldsymbol{w}}

\def\bmu{\boldsymbol{\mu}}
\def\bfeta{\boldsymbol{\eta}}
\def\bzeta{\boldsymbol{\zeta}}
\def\btheta{\boldsymbol{\theta}}
\def\bphi{\boldsymbol{\phi}}

\def\bxi{\boldsymbol{\xi}}
\def\bSigma{\boldsymbol{\Sigma}}
\def\bDelta{\boldsymbol{\Delta}}

\def\score{\boldsymbol{s}}
\def\diamX{\mathfrak D_{\!\boldsymbol{X}}}

\def\wass{{\sf W}}
\def\rmd{{\rm d}}

\def\sfd{{\sf d}}
\def\forwM{\mathscr M_{\!{}_\rightarrow}}
\def\backM{\mathscr M_{{}_\leftarrow}}
\def\hatbackM{\bar{\mathscr M}_{{}_\leftarrow}}

\def\bfA{\mathbf A}
\def\bfB{\mathbf B}

\def\bfI{\mathbf I}
\def\bfH{\mathbf H}
\def\bfM{\mathbf M}
\def\bfU{\mathbf U}

\title{Assessing the quality of denoising diffusion models
in Wasserstein distance: noisy score and optimal bounds }

\author{%
  Vahan Arsenyan\footnotemark[1]\thanks{Equal Contribution} 
  \quad Elen Vardanyan\footnotemark[1] \quad 
  Arnak S.~Dalalyan
  \\ 
  CREST, ENSAE, IP Paris\\
  5 av. H. Le Chatelier, 
  91120 Palaiseau, France\\
}

\begin{document}

\vspace*{-15pt}
\maketitle

\vspace*{-20pt}
\begin{abstract}\vspace*{-10pt}
Generative modeling aims to produce new random examples 
from an unknown target distribution, given access to a 
finite collection of examples. Among the leading approaches, 
denoising diffusion probabilistic models (DDPMs) construct 
such examples by mapping a Brownian motion via a diffusion 
process driven by an estimated score function. In this work, 
we first provide empirical evidence that DDPMs are robust 
to constant-variance noise in the score evaluations. We then 
establish finite-sample guarantees in Wasserstein-2 distance 
that exhibit two key features: (i) they characterize and 
quantify the robustness of DDPMs to noisy score estimates, 
and (ii) they achieve faster convergence rates than previously 
known results. Furthermore, we observe that the obtained rates 
match those known in the Gaussian case, implying their 
optimality. 
\end{abstract}

\vspace*{-50pt}

\doparttoc 
\faketableofcontents 

\setlength\cftbeforesecskip{5pt}
\setcounter{parttocdepth}{2}
\part{} 
\parttoc 
\setcounter{parttocdepth}{2}

\newpage
\section{Introduction}

We study the problem of generative modeling, which aims 
to construct a mechanism capable of producing synthetic 
samples that mimic a target distribution $P^*$, given 
access to independent observations from $P^*$. This 
fundamental task lies at the core of numerous applications, 
including image, text, music, and molecule generation. Among 
the recent advances in this domain, Denoising Diffusion 
Probabilistic Models (DDPMs)---introduced in 
\cite{HoJA20}---have emerged as a 
remarkably effective class of generative models; see, 
\textit{e.g.}, \cite{Overview_diff1,Overview_diff2,tutorial_diff} 
for comprehensive overviews. In this work, we contribute to 
the growing theoretical understanding of DDPMs by analyzing 
several of their key properties and performance guarantees.

The central idea underlying DDPMs is to construct a 
transport map that transforms a simple source of randomness 
into a sample from the target distribution~$P^*$. More 
precisely, for any distribution~$P^*$, there exists a 
map---defined via a stochastic differential equation 
(SDE)---that takes as input a standard Gaussian 
vector~$\bxi_0$ and a standard Brownian motion~$\bW$, 
and outputs a random vector whose distribution coincides 
with~$P^*$. Importantly, only the drift term of the SDE 
depends on~$P^*$, and this dependence occurs through 
the score function, that is, the gradient of the log-density 
of a Gaussian-smoothed version of~$P^*$. This formulation 
reduces the generative modeling task to that of score 
estimation: one can estimate the score function from data and 
substitute this estimate into the SDE to approximately sample 
from~$P^*$.

For many commonly used datasets, such as CIFAR-10 and 
CelebA-HQ considered in \Cref{sec:numerical}, accurate 
estimators of the score function are available.
Consequently, generating a synthetic sample reduces to 
drawing a standard Gaussian vector together with the 
increments of a Brownian motion, and simulating the SDE 
defined by the pretrained score. This procedure requires 
multiple evaluations of the score estimator at different 
inputs. The first question we address in this paper is: 
what happens if each evaluation returns a value corrupted 
by additive centered noise? Such a scenario may arise, 
for instance, when the pretrained model is hosted on a 
remote server and communication introduces random 
perturbations, or when the score values are compressed 
using stochastic rounding to reduce transmission costs. 
Anticipating our main findings, we emphasize that---perhaps 
counterintuitively---we observe that adding even a constant 
level of noise to each score evaluation has only a limited 
effect on the quality of the generated samples; see 
\Cref{fig:noisy_gen_images} for an illustration.


The second question we investigate concerns the 
accuracy of DDPMs when performance is measured 
in terms of the Wasserstein distance. A natural 
criterion in this setting is the number of score 
function queries $K$ required to achieve a prescribed 
level of accuracy $\varepsilon$. For the Gaussian 
target distribution, elementary computations show 
that  $K = \mathcal{O}(\sqrt{D} / \varepsilon)$, 
where $D$ denotes the ambient dimension. Surprisingly, 
however, it remains unclear whether DDPMs maintain 
this level of accuracy for broader classes of 
distributions beyond the Gaussian case. 
\vspace*{-9pt}

\paragraph{Contributions.}
\label{contributions}
The main contributions of this work can be 
summarized as follows:
\vspace*{-7pt}
\begin{itemize}[leftmargin=20pt]\itemsep=-2pt
    \item We provide empirical evidence, based on 
    experiments with the CIFAR-10 and CelebA-HQ 
    datasets, that DDPMs are remarkably robust to 
    noise in the evaluation of the score function.
    \item We derive non-asymptotic upper bounds 
    on the Wasserstein-2 distance between the 
    target distribution and the distribution 
    induced by the DDPM with noisy score evaluations, 
    thus offering a theoretical explanation for the 
    observed robustness.
    \item Our bounds match—up to a multiplicative 
    constant—the rate $\sqrt{D}/\varepsilon$ of the 
    case of a Gaussian target. Moreover, our results 
    extend to a significantly broader class of 
    distributions, including compactly supported 
    semi-log-concave measures supported on 
    low-dimensional subspaces.
\end{itemize}
\vspace*{-10pt}

\begin{figure}[t]
    \centering
    \includegraphics[height = 0.24\linewidth]{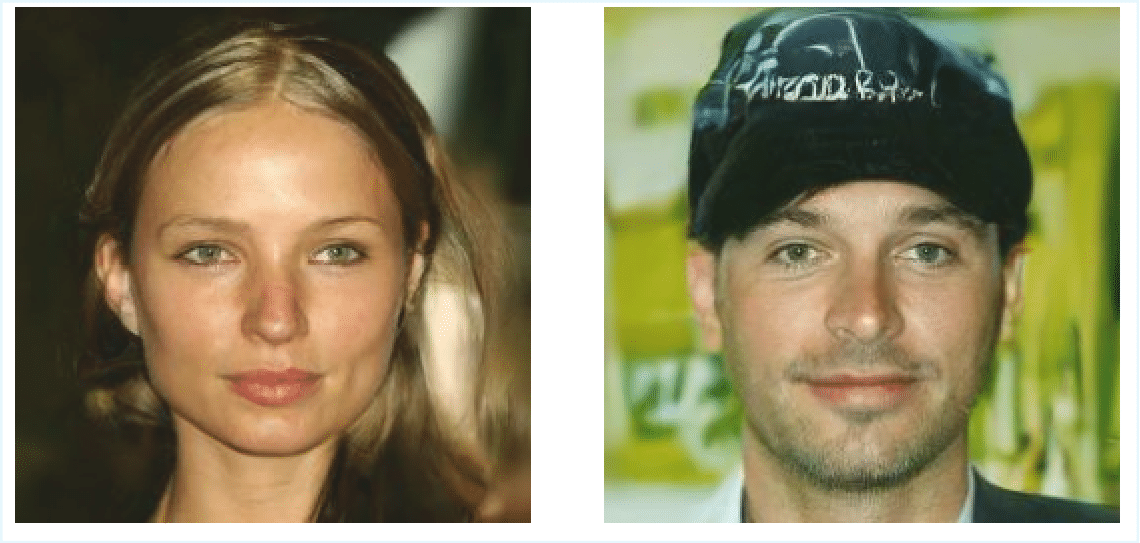}
    \includegraphics[height = 0.24\linewidth]{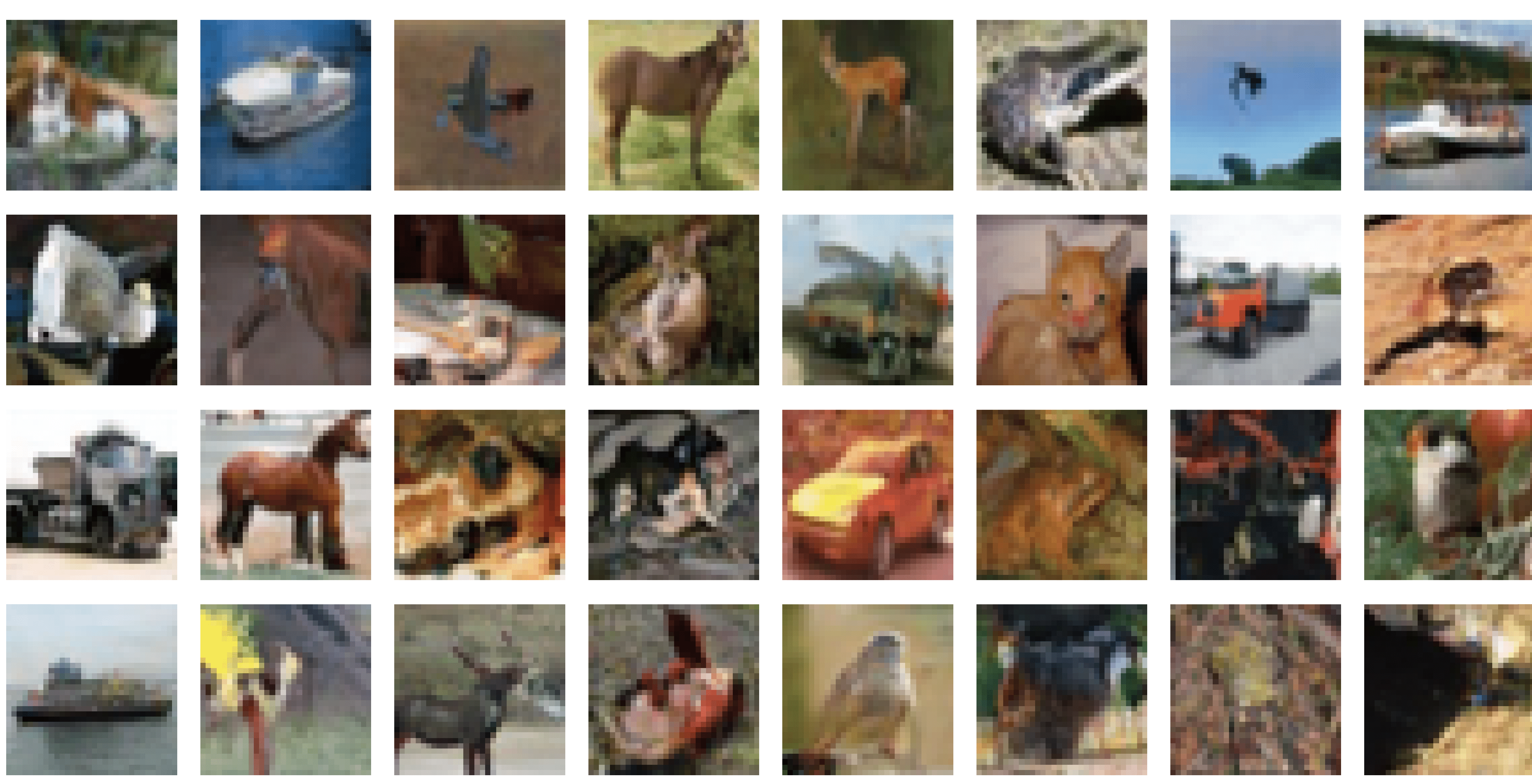}
    \caption{Generated images obtained by DDPM with a constant-level 
    noise added to the estimated score. Left: CelebA-HQ. 
    Right: CIFAR10. The result is visually as good as the noiseless one.}
    \label{fig:noisy_gen_images}
\end{figure}

\paragraph{Related work}
\cite{kwon2022scorebased} highlighted the connection 
between DDPMs and the Wasserstein distance. The first 
quantitative bounds---polynomial in the dimension and 
valid for a broad class of $P^*$---were established 
in \cite{ChenC0LSZ23}, covering several metrics. Unlike 
their result in total variation (TV) distance, their 
bound in Wasserstein distance has the poor scaling 
$D^5/\varepsilon^{12}$. Subsequent work significantly 
improved this rate: \cite{ChenL023} achieved $D^4
/\varepsilon^2$ under minimal assumptions, while 
\cite{Sabanis23, JMLR:v26:24-0902, Lu25, strasman2025, 
Ocello25} reduced it further to $D/\varepsilon^2$, 
assuming stronger conditions on $P^*$. Our paper 
closes the loop by proving that the optimal rate 
$\sqrt{D}/\varepsilon$ is achieved by the standard 
DDPM procedure. A related result by \cite{GaoZhu} 
establishes similar bounds for the probability flow 
ODE, but under more restrictive assumptions, such as 
strong log-concavity of $P^*$.

Over the past three years, substantial progress has also been made in establishing guarantees for DDPMs in total variation and Kullback–Leibler divergence under weak assumptions on $P^*$ \cite{ConfortiDS25, liang2025broadening, li2025odt, BentonBDD24, HEWC024}, including acceleration techniques such as parallel sampling, randomized midpoint, and Runge–Kutta methods \cite{ChenRYR24, Gupta, Wu24}. In parallel, a growing body of work investigates the statistical optimality of score-based models \cite{Suzuki, WibisonoWY24,holk2025}, as well as their ability to adapt to low-dimensional structure in high-dimensional ambient space \cite{debortoli2022, TangY24, li2024adapting, Wei2024, Azangulov24, Potaptchik}.
Analogous results for flow matching have been established
in \cite{kunkel2025}.

\vspace*{-7pt}
\paragraph{Notation} For $D\in\mathbb N$, $\bfI_D$ is the 
$D\times D$ identity matrix. We use notation $\bfA\prec \bfB$, 
$\bfA\preccurlyeq\bfB$, $\bfA\succ \bfB$, $\bfA\succcurlyeq 
\bfB$ to design that the matrix $\bfA-\bfB$ is, respectively, 
negative definite, negative semi-definite, positive definite 
and positive semi-definite. We denote by $\mathcal N_D (\bmu, 
\bSigma)$ the $D$-dimensional Gaussian distribution with mean 
$\bmu$ and covariance matrix $\bSigma$. We set $\gamma^D = 
\mathcal N_D(0, \bfI_D)$, and use the same notation $\gamma^D$ 
for the probability density function of $\mathcal N_D(0, 
\bfI_D)$. The norm of a vector is always understood as the 
Euclidean norm, whereas the norm of a matrix is the operator 
norm (the largest singular value). For two random vectors 
$\bX$ and $\bY$, $\bX\indep\bY$ means that $\bX$ and $\bY$ 
are independent. The Wasserstein-$q$ distance between two 
distributions $P$ and $Q$ is defined by
\begin{align}
    \wass_q^q(P,Q) = \inf_{\varrho\in\Gamma(P,Q)} \mathbf
    E_{(\bX,\bY)\sim\varrho}[\|\bX - \bY\|^q],
\end{align}
where $q\geqslant 1$ and $\Gamma(P,Q)$ is the set of all 
joint distributions with marginals $P$ and $Q$. For any 
function $g:[0,T]\times \mathbb R^D\to \mathbb R$, we will 
write $\nabla g$ and $\nabla^2 g$ for the gradient and the 
Hessian of $g$ with respect to its second variable. If 
$g:[0,T]\times \mathbb R^D \to\mathbb R^D$, we write 
$\mathrm{D}g$ for the differential of $g$ with respect to
its second variable. For each random vector $\bX$, we write
$\|\bX\|_{\mathbb L_2} = (\mathbf E[\|\bX\|_2^2])^{1/2}$. 

\section{Problem statement and conditions}
\label{sec:problem_statement}

The goal of this section is to set the framework of
denoising diffusion probabilistic models with randomized
score estimators and to state the conditions imposed
on the unknown target distribution. 

\subsection{The setting of randomized score estimators}
\label{ssec:2.1}
The setting considered in this paper slightly differs from 
those previously studied in the literature. For a fixed,
but unknown, distribution $P^*$ on $\mathbb{R}^D$, and 
for any $t>0$, we define $P^*_t$ to be the distribution 
of $\alpha_t\bX + \beta_t \bxi$, where $(\bX,\bxi)\sim P^*
\otimes \gamma^D$, $\alpha_t = e^{-t}$, and
$\beta_t = \sqrt{1-\alpha_t^2}$. The set $(P^*_t)_{t\geqslant 
0}$ can be seen as a curve in the space of probability measures 
interpolating between $P^*$ and $\gamma^D$, since  $P^*_0 = P^*$
and $P^*_\infty = \gamma^D$. For $t>0$, $P^*_t$ is 
absolutely continuous with respect to the Lebesgue measure
$\lambda^D$ on $\mathbb{R}^D$ with an infinitely differentiable
density. Therefore, we can define the score function $\score$ by
\begin{align}\label{eq:pi_t}
    \pi(t,\bx) = \frac{\rmd P^*_t}{\rmd\lambda^D}\, (\bx),
    \qquad \score(t,\bx) = \nabla \log \pi(t,\bx).
\end{align}
Since $P^*_t$ is unknown, we cannot access $\score(t,\bx)$. 
Instead, we have access to randomized and noisy evaluations 
of this function: 
for each query $(t,\bx)\in[0,\infty)\times \mathbb{R}^D$, we 
can observe a random vector $\widetilde\score(t,\bx)$ such that 
$\|\tilde \score(t,\bx) - \score(t,\bx)\|_{\mathbb{L}_2}$ is 
small. Our goal is to combine independent Gaussian random 
vectors and queries to the approximate score $\widetilde\score$ 
to build a random vector $\bZ$ in $\mathbb{R}^D$
having a distribution $P_Z$ close to $P^*$. To this end, we 
focus on the DDPM algorithm presented in \Cref{algo:1}.

\begin{center}
    \begin{minipage}{0.8\textwidth}
    \begin{tcolorbox}[
    colback=cyan!05,
    colframe=cyan!80!black,
    arc=4mm,
    boxrule=0.8pt,
    left=2mm, right=2mm, top=-1.8mm, bottom=-1mm
    ]
    \vspace{-10pt}
        \begin{algorithm}[H]
        \caption{Generation of $\bZ$ by the denoising diffusion 
        probabilistic model} \label{algo:1}
        \begin{algorithmic}[1]
        \REQUIRE Sequence $(t_1, \dots, t_{K+1})$ for some integer 
        $K\geqslant 1$
        \ENSURE Vector $\bZ = \bZ_{K+1}$
        \STATE Set $t_0 = 0$, $T = t_{K+1}$, and $\bZ_0 \sim\gamma^D$
        \FOR{$k = 0$ \TO $K$}
            \STATE Set $h_k = t_{k+1} - t_k$
            \STATE Generate $\bxi_{k+1} \sim \gamma^D$, 
            independent of all previous randomness
            \STATE Query $\tilde{\score}$ at $(t_k,\bZ_k)$
            \STATE Set $\bZ_{k+1} = (1 + h_k) \bZ_k + 2 h_k 
            \tilde{\score}(T-t_k, \bZ_k) + \sqrt{2h_k}\,
            \bxi_{k+1}$
        \ENDFOR
        \STATE \textbf{Output} $\bZ_{K+1}$
        \end{algorithmic}
        \end{algorithm}
    \end{tcolorbox}
    \end{minipage}
\end{center}

We postpone the discussion of the origin of this algorithm 
to \Cref{sec:3}. The main difference between our setting 
and prior work lies in the randomness of $\widetilde\score$, 
which goes beyond the randomness arising from the sample 
used to estimate the score. Let us provide concrete examples 
to illustrate our setting.

\textbf{Example 1 (Noisy score estimator).} Assume that 
the true score $\score(t,\bx)$ has been estimated by 
$\hat \score(t,\bx)$, for instance by fitting a deep 
neural network using a score matching algorithm. 
Due to issues such as communication constraints
or privacy concerns, we do not observe $\hat
\score(t,\bx)$ directly, but rather a noisy version 
$\widetilde\score(t,\bx) = \widehat\score(t,\bx) + \bzeta$, 
where $\bzeta$ is a random vector, typically with zero mean 
and bounded variance.

\textbf{Example 2 (Compressed score estimator).} 
Assume again that an estimator $\hat \score$ is 
available, but only one of its coordinates can be queried 
at a time. At each iteration, we randomly choose  
$i\in \{1,\ldots,D\}$ uniformly and set
$\widetilde\score(t,\bx) = D\times\big(\widehat\score (t, 
\bx)^\top\boldsymbol{e}_i\big) \boldsymbol{e}_i$, where 
$\boldsymbol{e}_i$ is the $i$-th canonical basis vector.

\textbf{Example 3 (Randomized network weights).} 
The conventional approach fits the weights $\btheta$
of a neural net $\bphi(t,\bx;\btheta)$ to the unknown
score $\score(t,\bx)$ by minimizing the (estimated)
prediction error:
\begin{align}
    \min_{\btheta\in\mathbb{R}^p} R(P^*,\btheta),\quad 
    \text{where}\quad R(P^*,\btheta):= \int_0^T 
    \int_{\mathbb{R}^D} \|\bphi(t,\bx,\btheta) - 
    \score(t,\bx)\|^2\,\pi(t,\bx)\,\rmd \bx\,\rmd t.
\end{align}
One can instead minimize an estimator of the integrated 
error under a Gaussian prior by solving
\begin{align}
    \hat\bmu\in \argmin_{\bmu\in\mathbb{R}^p} \int_{\mathbb{R}^p} 
    R(P^*,\bmu + \sigma \bz)\,\gamma^p(\bz)\,\rmd \bz,
\end{align}
where $\sigma>0$ is a hyperparameter. This may lead
to a more robust score estimator. In this setting, 
the randomized estimator of the score at each query point
$(t,\bx)$ is $\bphi(t,\bx,\hat\bmu + \sigma \bzeta)$, with 
$\bzeta\sim \gamma^p$ generated independently by the user.

\subsection{Conditions on the target distribution}
\label{ssec:2.2}

The guarantees on the precision of the DDPM that we will 
state in the next section depend on the properties of
the target $P^*$. We will express these properties 
in terms of a function $\varphi$. 

\begin{assumption} \label{ass:1}
For a function $\varphi:\mathbb R_{>0}\to\mathbb R_{>0}$, 
we say that $P^*$ or $\bX$ satisfies \Cref{ass:1} 
with function $\varphi$ if, for $(\bX,\bxi) \sim 
P^* \otimes \gamma^D$,
it holds that $\Var\big(\bX \cond \bX + \sigma \bxi 
= \by \big)\preccurlyeq \varphi(\sigma)\,\bfI_D$ 
for all $\sigma > 0$.
\end{assumption}

Many distributions satisfy this assumption (see 
\Cref{app:A} for the proofs):
\vspace{-7pt}
\begin{itemize}[leftmargin = 20pt]\itemsep=-1pt
    \item[(a)] If $\bX$ has bounded support 
    $\mathcal{K}$ with 
    $\mathrm{diam}(\mathcal{K}) = 2\diamX$, 
    \Cref{ass:1} holds with  $\varphi(\sigma)\equiv 
    \diamX^2$;
    
    \item[(b)] Any $m$-strongly log-concave 
    distribution $P^*$ satisfies \Cref{ass:1} 
    with $\varphi(\sigma) = \frac{\sigma^2}{1 
    + m\sigma^2}$;
    
    \item[(c)] If $\bX$ is semi-log-concave with 
    constant\footnote{We recall that $\bX$ is 
    semi-log-concave \cite{Clarke1983} with constant $M\in\mathbb{R}$ 
    if $\bX$ has a density $\pi_X$ wrt 
    the Lebesgue measure and $-\log \pi_X(\bx) 
    + \tfrac{M}2 \|\bx\|^2$ is convex; see \cite{Vacher25} for an 
    application in sampling.} $M\geqslant 0$ 
    and has bounded support of diameter $2\diamX$, 
    then $\bX$ satisfies \Cref{ass:1} with $\varphi
    (\sigma) = \diamX^2 \wedge \frac{\sigma^2}{(1 - 
    M\sigma^2)_+}$;
    
    \item[(d)] If $\bX$ satisfies \Cref{ass:1} with 
    some function $\varphi$, $\bfU$ is a $D\times D$ 
    orthonormal matrix and $\boldsymbol{b}\in\mathbb R^D$, 
    then $\bfU\bX + \boldsymbol{b}$ satisfies 
    \Cref{ass:1} with the same $\varphi$;
    
    \item[(e)] If $\bX$ is obtained by concatenating 
    two independent vectors $\bX_1$ and $\bX_2$ satisfying 
    \Cref{ass:1} with the same function $\varphi$, 
    then $\bX$ satisfies \Cref{ass:1} with $\varphi$.

    \item[(f)] If $(\bW,\bzeta)\sim P_0\otimes \gamma^D$ 
    such that $\bW$ satisfies \Cref{ass:1} with the function
    $\varphi_0$, then, $\bX=\bW + 
    \tau\bzeta$ satisfies \Cref{ass:1} with the function
    $\varphi_\tau(\sigma) = \frac{\tau^2\sigma^2}{\tau^2+\sigma^2} 
    + \frac{\sigma^4\varphi_0(\sqrt{\tau^2+\sigma^2})}{(\tau^2 + \sigma^2)^2}.$

    \item[(g)] If $\bW$ is 
    supported by a bounded set of diameter $2\mathfrak D$ 
    and $\bzeta\indep \bW$ is $m$-strongly 
    log-concave with an $M$-Lipschitz score function,
    then $\bX = \bW + \bzeta$ satisfies \Cref{ass:1}
    with $\varphi(\sigma) = \frac{\sigma^2}{1+m\sigma^2} 
    + \frac{(M\mathfrak D \sigma^2)^2}{(1+M\sigma^2)^2}$.
\end{itemize}

\vspace{-5pt}
The main purpose of \Cref{ass:1} is to ensure that the 
drift coefficient of the backward diffusion process is 
strongly convex when the noise level is large and 
semi-log-concave for all noise levels. Moreover, 
the drift coefficient is always gradient-Lipschitz, 
with a Lipschitz constant depending on the noise level. 
These properties are summarized in the following 
result\footnote{The formula relating the Hessian of the
log-density to the conditional variance, stated 
in \Cref{prop:1} is often referred to as the second-order
Tweedie formula.}.

\begin{proposition}\label{prop:1}
    Let $\bX$ and $\bxi$ be random vectors in $\mathbb{R}^D$ 
    drawn from $P^*\otimes \gamma^D$. 
    For any $\alpha,\beta>0$, the density $\pi_Y$ of $\bY=
    \alpha \bX + \beta \bxi$ is twice continuously 
    differentiable and satisfies
    \begin{align}
        \nabla^2\log \pi_Y(\by) = \frac{\alpha^2}{\beta^4}
        \Var(\bX\cond\bY=\by) -\frac{1}{\beta^2}\bfI_D
        \succcurlyeq - \frac{1}{\beta^2}\bfI_D,
        \qquad \text{for all } \by\in \mathbb{R}^D.
    \end{align}
    Thus,  \Cref{ass:1} is equivalent to $\nabla^2\log \pi_Y 
    (\by)\preccurlyeq \frac{(\alpha^2\varphi(\beta/\alpha) 
    - \beta^2)}{\beta^4} \,\bfI_D$, for all $\by\in
    \mathbb{R}^D$, $\alpha,\beta>0$.
\end{proposition}
The last inequality above implies that if $\varphi(\beta/\alpha)
\leqslant (\beta/\alpha)^2$ , the distribution of $\bY = 
\alpha\bX + \beta \bxi$ is log-concave, and it is strongly 
log-concave if the inequality is strict. 

\subsection{Conditions on the estimated score}
\label{ssec:score_ass}

As mentioned in \Cref{ssec:2.1}, we consider randomized estimators 
\(\widetilde\score\) of the true score function \(\score\).  
The mean squared error of such an estimator can be decomposed into a bias and a variance term:
\begin{align}
    \mathbf{E}\big[\left\|\widetilde\score(t,\bx) - \score(t,\bx)\right\|^2\big]
    = \left\|\mathbf{E}\left[\widetilde\score(t,\bx)\right] - \score(t,\bx)\right\|^2
    + \mathbf{E}\big[\left\|\widetilde\score(t,\bx) - \mathbf{E}\left[\widetilde\score(t,\bx)\right]\right\|^2\big].
\end{align}
In what follows, we analyze separately the impact of the bias 
and the variance on the overall error.  As we will see, the 
variance term has a much weaker influence on the final accuracy 
than the bias term.  To reflect this difference, we introduce 
the following assumption.

\begin{assumption}\label{ass:2}
    There are constants $\varepsilon^b_{\sf score}$ and  
    $\varepsilon_{\sf score}^v$ 
    such that for all $t\in \{t_k:k\leqslant K\}$ of \Cref{algo:1}, 
    \begin{align}
        \sup_{\bx\in\mathbb R^D} \left\| \mathbf{E}\left[ 
        \widetilde\score (t,\bx)\right] - \score(t,\bx)\right\| 
        \leqslant D^{1/2}\varepsilon^b_{\sf score},\quad
        \sup_{\bx\in\mathbb R^D} \left\|\widetilde\score(t,\bx) 
        - \mathbf{E}\left[\widetilde\score (t,\bx)\right]
        \right\|_{\mathbb L^2}\leqslant D^{1/2} \varepsilon^v_{\sf score}.
    \end{align}
\end{assumption}
It is worth emphasizing that our version of this 
assumption imposes uniformity over all $\bx \in
\mathbb{R}^D$ and $t \in {t_k : k \leqslant K}$. 
This is a stronger condition than the one used 
in previous work \cite{ChenL023}, where a similar 
assumption is stated in terms of an $\mathbb L_2$-norm
with respect to $P^*_t$, rather than 
a supremum, and involves a weighted average over $t$. While 
it may be possible to relax the requirement involving the 
maximum over the time grid, the uniformity with respect to 
$\bx$ appears to be more difficult to replace by the $\mathbb L_2$-norm with respect to $P^*_t$. It is 
important to note, however, that what is actually required 
for our proof is an $\mathbb{L}_2$ bound with respect to 
the distribution of the DDPM output at time $t$. 


\section{Score-Based Generative Modeling: preliminary 
considerations}\label{sec:3}


\subsection{Forward Process}

The starting point of a DDPM is the forward process given as a solution to a stochastic 
differential equation (SDE). The simplest and the most widespread choice is the 
Ornstein–Uhlenbeck process
\begin{align}
\label{eq:ou-forward}
    \rmd \bX_t = -\bX_t \,\rmd t + \sqrt{2}\, \rmd \bB_t, 
		\qquad t \geq 0,\quad \bX_0 \sim P^*,\qquad 
		(\bB_t)_{t\geqslant 0} \indep \bX_0,
\end{align}
where $(\bB_t)_{t\geqslant 0}$  is a standard Brownian 
motion in $\mathbb{R}^D$. The Ornstein–Uhlenbeck process is a 
time-homogeneous Markov process which is also a Gaussian 
process, with stationary distribution equal to the 
standard Gaussian distribution $\gamma^D$ on $\mathbb{R}^D$. 
The forward process has the interpretation of transforming 
samples from the data generating distribution $P^*$ into the 
latent distribution. From the classical theory of 
Markov diffusions, it is known that $P^*_t := \text{law}
(\bX_t)$ converges to $\gamma_D$ exponentially
fast in various divergences and metrics such as the 
2-Wasserstein metric $\wass_2$: $\wass_2 (P^*_t; \gamma^D) 
\leq e^{-t} \wass_2(P_0; \gamma^D)$, see for instance 
\cite{villani2008optimal}.


\subsection{Reverse Process: continuous-time version}\label{sec:reverse}

If we reverse the forward process in time, we obtain 
a process that transforms the latent distribution into the target 
distribution $P^*$, which is the aim of generative modeling. 
Fix some large time horizon $T>0$ and set $\bY_t:= \bX_{T-t}$, 
then $\text{law}(\bY_0) = \text{law}(\bX_T)$ is close to the 
Gaussian distribution $\gamma^D$. Notably, the
dynamics of the reverse process can also be described by a 
stochastic differential equation, as stated in the next result.

\begin{theorem}[\cite{anderson1982}]   
If $\bX$ is a solution to \eqref{eq:ou-forward} and $\bY_t = 
\bX_{T-t}$, then there exists a Brownian Motion 
$\wt{\bB} \indep \bY_0$ such that
\begin{align}\label{eq:backward}
    \rmd \bY_t = (\bY_t +2\nabla \log \pi(T-t, \bY_t)) 
    \,\rmd t + \sqrt{2}\, \rmd \wt{\bB}_t,\quad 
    0\leq t\leq T,
\end{align}
where $\pi(t, \bx) \propto \int_{\mathbb R^D} \gamma^D 
\big((\bx-\alpha_t\by)/\beta_t\big)\,P^*(\rmd \by)$, 
$\alpha_t = e^{-t}$ and $\beta_t^2 = 1-e^{-2t}$.
\end{theorem}

Note that $\pi(t, x)$ in this theorem coincides with the one
defined in \eqref{eq:pi_t} and $\nabla \log \pi(T-t, \bY_t)$ 
is the score function $\score$  evaluated at scale $T-t$ and 
state $\bY_t$. 

The forward process transforms a data point 
$\bX_0$ drawn from $P^*$ into a point which is very
close to being drawn from the latent distribution. 
The reverse process aims to transform a
point $\bY_0$ drawn from the latent distribution
into a point drawn from $P^*$. To this end, 
we replace the unknown score function by its estimate 
$\wt{\score}$ based on a training sample $\bX_1,
\ldots,\bX_n\sim P^*$. The resulting process is 
defined as the solution to the SDE
\begin{align}
    \rmd \wt\bY_t = (\wt\bY_t +2\wt{\score}
		(T-t, \wt\bY_t))\,\rmd t + \sqrt{2}\, \rmd 
    \wt{\bB}_t,\qquad \wt\bY_0\sim \gamma^D,
    \qquad t\in[0,T].\label{backward2}
\end{align}
Both $\wt{\bY}$ and $\bY$ are processes on the space 
$\mathbb{C}([0, T], \mathbb{R}^D)$, differing in their 
initial conditions and drift terms. We wish to assess the 
distance between the 
distributions of their states at time $T$.

\subsection{Reverse Process: discrete-time version}
\label{ssec:reverse}

To be able to sample the final state---or any intermediate
one---of the reverse process efficiently, we  have to
discretize SDE \eqref{backward2}. To this end, we introduce
a sequence $(h_k)_{k\in\mathbb N}$ of positive numbers and
set\footnote{By convention, $t_0=0$.} $t_k = h_0 + \ldots + 
h_{k-1}$. We then define
\begin{align}\label{discr}
    \bZ_{k+1} = (1+{h_k})\bZ_k + 2h_k\wt\score(T-t_k,\bZ_k)
    + \sqrt{2h_k}\,\bxi_{k+1};\qquad \bZ_0\sim \gamma^D,
\end{align}
where $(\bxi_k)_{k\in\mathbb N}$ is a sequence of independent
standard Gaussian random variables. The rationale behind this
definition is that $\bZ_k$ has approximately the same law as 
$\wt\bY_{t_k}$, for every $k$. 

\begin{definition}
    The denoising diffusion probabilistic model is the distribution
    $P^{\sf DDPM} $ of the random vector $\bZ_{K+1}$ defined by 
    \eqref{discr}. It requires the choice of $K\in\mathbb N$, 
    the sequence $(t_1,\ldots,t_{K+1})$ 
    and the score estimators $\big(\wt\score (T-t_k,\cdot)
    \big)_{k=0,\ldots,K}$. 
\end{definition}
In this paper, we are interested in quantifying the accuracy
of the denoising diffusion generative model when the
error is measured in terms of the Wasserstein distance, that
is to upper bound $\wass_2(P^*,P^{\sf DDPM})$. In the
rest of this section, we motivate the choice of the 
Wasserstein distance and discuss the challenges related
to it in the framework of denoising diffusions.

\subsection{Relevance of the Wasserstein distance}
\label{ssec:wass}

Recent work on assessing the error of denoising diffusion 
models mainly focuses on accuracy measured by
the total variation distance and the Kullback-Leibler
divergence.  However, we believe that for
statistical purposes, measuring the quality of a generative
model in the Wasserstein distance is highly appealing.

To justify this point of view, let us recall that
the closeness of two distributions in total variation
distance or KL-divergence does not guarantee the
closeness of their means or their covariance matrices.
In sharp contrast, the Wasserstein-2 distance
offers such a guarantee since for any pair of distributions
$P$ and $Q$ defined on the same space, it holds that
\begin{align}
    \|\mathbf E_P[\bX] - \mathbf E_Q[\bX]\|
    \leqslant \wass_2(P,Q);\quad
    |(\mathbf E_P[\bX^\top\mathbf A\bX])^{1/2} -
    (\mathbf E_Q[\bX^\top\mathbf A\bX])^{1/2}|
    \leqslant \wass_2(P,Q),
\end{align}
for any matrix $\mathbf A$ satisfying $0\preccurlyeq
\mathbf A \preccurlyeq \mathbf I$. The fact that
the TV-distance and the KL-divergence are not suitable
for controlling the moments of distributions can be
demonstrated by the following example. Let $P$
be the exponential distribution with parameter 1 and,
for every $n\in \mathbb N $, set $P_n= (1-\delta_n)P +
\delta_n Q_n$, where $\delta_n=1/\sqrt{n}$ and $Q_n$
is the uniform distribution on $[n,n+2]$. One can
easily check that $P_n$ is very close to $P$ both
in the TV-distance and in the KL-divergence:
\begin{align}
    \sfd_{\sf TV}(P_n;P)\leqslant \delta_n =
    n^{-1/2};\qquad \sfd_{\sf KL} (P|\!|P_n)
    &= -\log(1-\delta_n)\leqslant
    2n^{-1/2},\quad n\geqslant 2.
\end{align}
Therefore, one could expect that $P_n$ is an excellent
generative model for the target $P$. However, the
generated examples will have a mean and variance
that explode as $n\to\infty$, and will be
infinitely far away from the mean and the variance
of the target, since $\mathbf E_{P_n}[X] = 1+n\delta_n
\geqslant n^{1/2}$ and $\mathbf E_{P_n}[X^2] \geqslant
2(1-\delta_n) + \delta_n n^2\geqslant n^{3/2}$.

\subsection{Challenges inherent to Wasserstein 
distance}

When the distance $\wass_q$ is employed 
to assess the quality of a DDPM, a mathematical challenge 
arises in quantifying the error due to the absence of 
the data-processing inequality for $\wass_q$-distance. 
Let us clarify this point. Consider a forward 
mechanism $\forwM$ that transforms the target  $P^*$ into a distribution $P_1^*$ which is 
close to an easy-to-sample-from latent distribution $Q_0$: 
$P_1^* := \forwM(P^*) \approx Q_0$. Furthermore, 
assume we have knowledge of the ``inverse'' forward 
mechanism, termed backward mechanics, which maps 
$P_1^*$ back to $P^*$: $\backM(P_1^*) = P^*$. The 
forward-backward methods of generative modeling 
then define the generative model as $Q_1 = \hatbackM(Q_0)$, 
where $\hatbackM$ represents a suitably regularized 
estimator of $\backM$. In DDPM, $\backM$ and $\hatbackM$ are specified 
by Markov kernels.

In this context, denoting $\sfd_F$ as the $F$-divergence 
for some $F$, the following relationship holds:
\begin{align}
    \sfd_F(Q_1|\!|P^*) &= \sfd_F\big(\hatbackM(Q_0) 
    \big|\!\big| P^*\big) \approx \sfd_F
    \big(\backM(Q_0)\big|\!\big|P^*\big)\\ 
    &= \sfd_F\big(\backM(Q_0)\big|\!\big|
    \backM(P_1^*)\big) \stackrel{\rm DPI}{\leqslant} 
    \sfd_F(Q_0|\!|P_1^*),
\end{align}
where the final equality derives from the 
data-processing inequality. Thus, the error of the 
generative distribution is dominated by how well the 
forward mechanism approximates the latent distribution, 
provided that the error of $\backM$ approximation is 
suitably small. These arguments were central in prior 
work\footnote{See \cite{ChenC0LSZ23,BentonBDD24,
Wei2024,ConfortiDS25} 
and the references therein} establishing  bounds on the 
error of denoising diffusion models measured in total 
variation distance and KL-divergence, which are 
$F$-divergences with $F(x)= \tfrac12|x-1|$ and 
$F(x) = x\log x$ respectively. However, this approach 
breaks down for the Wasserstein distance $\wass_q$, 
for which no suitable equivalent of the data processing 
inequality exists. 

In the case of denoising diffusion models, the 
qualitative difference between the Wasserstein distance 
and $F$-divergences (such as TV-distances and KL-divergence) 
can be formally demonstrated even when the backward kernel 
is known. This is illustrated in the following lemma.

\begin{lemma}\label{lem:noDPI}
    For any  $T>0$, let $Q_1^{T,\score}$ 
    be the distribution of the backward process 
    \eqref{backward2} at time $T$ with $\wt\score$ replaced by
    the true score $\score$. Let $\mathcal N$ be the set
    of all the Gaussian distributions. It then holds that
    \begin{align}
        \sup_{P^*\in \mathcal N} \frac{\sfd_{\mathsf{TV}
        }^2(Q_1^{T,\score}; P^*)}{\sfd_{\mathsf{TV}
        }^2(P^*;\gamma^D)} \bigvee \frac{\sfd_{\mathsf{KL}
        }(Q_1^{T,\score}|\!| P^*)}{\sfd_{\mathsf{KL}
        }(P^*|\!|\gamma^D)} 
        \leqslant e^{-2T};\qquad
        \sup_{P^*\in \mathcal N} \frac{\wass_2(Q_1^{T,\score};
        P^*)}{\wass_2(P^*;\gamma^D)} = 1.
    \end{align}
\end{lemma}
This lemma reveals that when assessing accuracy through 
the rate of improvement in Wasserstein distance, the 
choice of parameter $T$ must be carefully tailored to 
the target distribution $P^*$. This might be less important
in the case of the TV-distance and the KL-divergence.

\section{Main results: bounds on the error in various settings}

In this section, we derive upper bounds on the Wasserstein-2 
distance between the distribution of the random vector generated 
by the DDPM (see \Cref{algo:1}) and the target distribution 
$P^*$. To this end, we employ a discretization scheme previously 
introduced in the literature \citep{ChenL023,BentonBDD24}. 
This scheme operates in two regimes: an arithmetic grid in the 
first half and a geometric grid in the second half; 
see \Cref{algo:2} for further details.
\begin{center}
    \begin{minipage}{0.85\textwidth}
    \begin{tcolorbox}[
    colback=cyan!05,
    colframe=cyan!80!black,
    arc=4mm,
    boxrule=0.8pt,
    left=2mm, right=2mm, top=-1.7mm, bottom=-1mm
    ]
    \vspace{-10pt}
        \begin{algorithm}[H]
        \caption{Definition of the discretization time steps} \label{algo:2}
        \begin{algorithmic}[1]
        \REQUIRE $\delta,a,T_1>0$, 
        and $K_0\in\mathbb N_{>1}$
        \ENSURE Sequence $t_0<t_1<\ldots<t_{K+1}$
        \STATE Set $t_0 = 0$, $K = 2K_0$, $t_{K+1}=T_1+\tfrac12\log(6a)$
        \FOR{$k = 1$ \TO $K_0$}
            \STATE Set $t_{k} =  (T_1/K_0)\,k$ 
            \qquad\qquad\qquad\qquad\hspace*{46pt}\COMMENT{arithmetic grid}
            \STATE Set $t_{K_0+k} = T_1 + \tfrac{\log(6a)}{2}
            \big[1 - \big(\frac{2\delta}{\log(6a)} \big)^{k/K_0}
            \big]$. \qquad\COMMENT{geometric grid}
        \ENDFOR
        \STATE \textbf{Output} $(t_0,\ldots,t_{K+1})$
        \end{algorithmic}
        \end{algorithm}
    \end{tcolorbox}
    \end{minipage}
\end{center}

\subsection[Strongly log-concave 
convolved with a bounded support  distribution]{Strongly log-concave distributions 
convolved with a distribution with bounded support}

In this section, we consider the case of a distribution $P^*$
satisfying \Cref{ass:1} with a function $\varphi$ that has
the following form: for some constants $m,M,b\geqslant 0$, 
\begin{align}\label{eq:phi_slg}
    \varphi(\sigma) = \frac{\sigma^2}{ 1 + m\sigma^2} +
    \frac{bM ^2\sigma^4}{(1 + M\sigma^2)^2},\qquad \forall 
    \sigma>0.
\end{align}
If $P^*$ is $m$-strongly log-concave, as discussed in \Cref{ssec:2.1}, then \eqref{eq:phi_slg} holds with $b = 0$ and any $M > 0$. Another class of distributions satisfying \eqref{eq:phi_slg} consists of convolutions $P^* = P_{\sf slc} \star P_{\sf cmpct}$, where $P_{\sf slc}$ is $m$-strongly 
log-concave with an $M$-Lipschitz score, and $P_{\sf cmpct}$ is supported on a bounded set of diameter $2\mathfrak D$, for 
some $M \geqslant m > 0$ and $\mathfrak D \geqslant 0$. In this case, \eqref{eq:phi_slg} holds with $b = \mathfrak D^2$.

Finally, there are distributions satisfying \Cref{ass:1} with $\varphi$ given by \eqref{eq:phi_slg} that are not absolutely continuous with respect to the Lebesgue measure on $\mathbb{R}^D$. For example, if $P^*$ is supported on a linear subspace $\mathcal S$ of $\mathbb{R}^D$, and its restriction to $\mathcal S$, viewed as a distribution on $\mathbb{R}^d$ for some $d \in {1, \ldots, D}$, satisfies \Cref{ass:1} with $\varphi$ given by \eqref{eq:phi_slg}, then $P^*$ also satisfies the assumption with the same $\varphi$. This is a consequence of properties (d) and (e)
presented in \Cref{ssec:2.2}.

\begin{theorem}\label{th:1}
    Let the target distribution $P^*$ satisfy $\mathbf E 
    [\|\bX\|_2^2]\leqslant D$ and \Cref{ass:1} with function
    $\varphi$ given by \eqref{eq:phi_slg} for some $m,M,b 
    \geqslant 0$. Let us choose $T_1>0$, 
    \begin{align}
        a = \tfrac1m + b,\qquad 
        K_0\geq 7T_1\log(6a) + 4\log(6a)\log\log(6a) 
        \qquad \delta = 0.5e^{-2T_1},
    \end{align}
    and define the sequence $(t_k)_{0\leq 
    k\leq K+1}$ by \Cref{algo:2}. Let $\widetilde \score$ be a 
    randomized estimator of the score satisfying \Cref{ass:2}.
    Then, the distribution $P^{\sf DDPM}$ of the output 
    of \Cref{algo:1} based on $2K_0$ queries to the score estimator
    $\widetilde \score$ satisfies
    \begin{align}\label{eq:wass_bound1}
        \wass_2(P^*,P^{\sf DDPM}) \leq e^{(4/3)bM}
        \Big\{2e^{-T_1} + 
        7\sqrt{6a}\,h_{\max} + 4\sqrt{6a}\big(2\varepsilon^b_{
        \sf score} + h_{\max}^{1/2}\,\varepsilon^v_{\sf score}\big)
        \Big\}\sqrt{D},
    \end{align}
    with $h_{\max} = \max_{k} (t_{k+1}-t_k) \leqslant 
    \frac{\log(6a)(\log\log(6a) + 2T_1)}{K_0}$.
\end{theorem}

There are several notable features in the upper bound stated 
in \Cref{th:1}, when we compare it to the previously known results. 

\begin{remark}[Optimality]\label{rem:1}
    The dependence of the discretization error (the second term in 
    \eqref{eq:wass_bound1}) on the largest step 
    size $h_{\max}$ is linear, 
    whereas it was of order $h_{\max}^{1/12}$ in \cite[Cor.~6]
    {ChenC0LSZ23}, $h_{\max}^{1/4}$ in \cite[Cor.~2.4]{ChenL023}, 
    and $h_{\max}^{1/2}$ in \cite[Remark 12]{Sabanis23}, \cite[Cor.~4.3]
    {strasman2025}, \cite{Ocello25, JMLR:v26:24-0902, Lu25}. Moreover, 
    \cite{JMLR:v26:24-0902} establishes that the lower bound on the 
    Wasserstein-2 error, achieved by the Gaussian distribution, 
    scales as $\sqrt{D}\,h_{\max}$, thereby implying the optimality 
    of the bound in \Cref{th:1}.
\end{remark}

\begin{remark}[Conditions] 
    The bound established in \Cref{th:1} is derived under a set of 
    conditions on the target distribution that may be regarded as 
    restrictive. Indeed, Corollary 2.4 in \cite{ChenL023}, for
    instance, applies to a broader class of distributions than
    those encompassed by \Cref{th:1}. Nevertheless, 
    stronger assumptions are typically indispensable for attaining 
    faster rates of convergence. In this regard, the assumptions 
    imposed in \Cref{th:1} are less stringent than those adopted in 
    earlier works such as \cite{Sabanis23,Lu25,JMLR:v26:24-0902}, 
    among others. In particular, in the case where $P^*$ is 
    $m$-strongly log-concave, we do not assume that the Hessian of 
    the log-density is bounded from below. Furthermore, \Cref{th:1} 
    covers a broad class of distributions obtained as convolutions 
    of a compactly supported distribution and a Gaussian—a framework 
    not addressed in previous studies achieving a convergence rate of 
    $h_{\max}^{1/2}$. In conclusion, the conditions required by 
    \Cref{th:1} are weaker than those previously associated with the 
    $h_{\max}^{1/2}$ rate, while enabling the faster convergence 
    rate of $h_{\max}$.
\end{remark}

\begin{remark}[Impact of noise]
\label{rem:2} All previously known bounds 
are proportional to $\|(\widetilde\score - \score) (\tau,\bX) 
\|_{\mathbb{L}_2}$, 
where the proportionality factor is often logarithmic in the 
number of queries, and the $\mathbb{L}_2$-norm can take 
different forms—the weakest being the case where $\tau \sim 
\text{Unif}([0,T])$ and the law of $\bX$ given $\tau = t$ is 
$P^*_{t}$. If $\widetilde\score(t,\bx) = \hat\score(t,\bx) + \bzeta$, 
with $\|\bzeta\|_{\mathbb{L}_2}^2 = \sigma_{\bzeta}^2 D$ 
as in Example 1 of \Cref{ssec:2.1}, then $\|\widetilde\score - 
\score\|_{\mathbb{L}_2}^2 \geq \sigma_{\bzeta}^2 D$. Thus, 
all known bounds include a term of constant order, 
independent of the number of queries. In contrast, the 
corresponding term in the bound of \Cref{th:1} is $O(\sqrt{D\, h_{\max}}\,\varepsilon^v_{\sf score})$, which scales as 
$\sigma_{\bzeta} \sqrt{DT_1/K}$ and thus vanishes as $K$, 
the number of queries, grows large. 
\end{remark}

\begin{remark}[Informal statement] To facilitate comparison 
with existing results, let us consider the strongly log-concave
case $b=0$ and denote by $L := a$ the surrogate of the Lipschitz 
norm of the score of $P^*$. For $T_1 = \log(K_0)$, our result 
implies that, after $K$ queries to the score estimator, 
\begin{align}
    \wass_2(P^*,P^{\sf DDPM}) \lesssim \sqrt{LD}\Big\{\frac{
    \log L\log K}{K} + \varepsilon_{\sf score}^b + 
    \frac{\sqrt{\log L\log K}}{\sqrt{K}}\,
    \varepsilon_{\sf score}^v\Big\}.
\end{align}
In particular,  $\wass_2(P^*,P^{\sf DDPM}) \lesssim \sqrt{LD}\, 
\varepsilon_{\sf score}^b$, provided that the number of queries
satisfies
\begin{align}
    \frac{K}{\log K} \geqslant \Big\{\frac{1}{\varepsilon_{\sf 
    score}^b}\bigvee \Big(\frac{\varepsilon_{\sf score}^v}{
    \varepsilon_{\sf score}^b}\Big)^2\Big\}\,\log L. 
\end{align}
As mentioned in \Cref{rem:2}, this improves on \cite{Sabanis23,Lu25,JMLR:v26:24-0902,Ocello25}, which 
require $K\gtrsim(\log L)/(\varepsilon_{\sf score}^b)^2$ and $\varepsilon_{\sf score}^v \lesssim\varepsilon^b_{\sf score}$ to achieve $\wass_2(P^*,P^{\sf DDPM}) \lesssim \sqrt{LD}\, 
\varepsilon_{\sf score}^b$. 
\end{remark}

\subsection{Semi log-concave distributions with bounded support}

In this section, we consider the case of a distribution $P^*$
satisfying \Cref{ass:1} with a function $\varphi$ that has
the following form: for some constants $b,M\geqslant 0$, 
\begin{align}\label{eq:str_conc_b:0}
    \varphi(\sigma) = b \wedge \frac{\sigma^2}{(1 - 
    M\sigma^2)_+},\qquad \forall 
    \sigma>0.
\end{align}
The typical example of $P^*$ satisfying this assumption
is a distribution by a bounded set $\mathcal K$ included in 
a linear subspace of $\mathbb R^D$, if in addition the
log-density wrt to the Lebesgue measure on the subspace
has a Hessian $\preccurlyeq M\bfI$. It then follows from
claims (c), (d), and (e) of \Cref{ssec:2.1} that 
$P^*$ satisfies \Cref{ass:1} with $\varphi$ as in
\eqref{eq:str_conc_b:0} with $b=\diamX^2$. 
\begin{theorem}\label{th:2}
    Let the target distribution $P^*$ satisfy $\mathbf E 
    [\|\bX\|_2^2]\leqslant D$ and \Cref{ass:1} with function
    $\varphi$ given by \eqref{eq:str_conc_b:0} for some $b,M
    \geqslant 0$. Let us choose $T_1>0$, 
    \begin{align}
        a = b\vee1,\qquad 
        K_0\geq 7T_1\log(6a) + 4\log(6a)\log\log(6a) 
        \qquad \delta = 0.5e^{-2T_1},
    \end{align}
    and define the sequence $(t_k)_{0\leq 
    k\leq K+1}$ by \Cref{algo:2}. Let $\widetilde \score$ be a 
    randomized estimator of the score satisfying \Cref{ass:2}.
    Then, the distribution $P^{\sf DDPM}$ of the output 
    of \Cref{algo:1} based on $2K_0$ queries to the score estimator
    $\widetilde \score$ satisfies
    \begin{align}\label{eq:wass_bound1b}
        \wass_2(P^*,P^{\sf DDPM}) \leq e^{2b M+1}
        \Big\{2e^{-T_1} + 
        7\sqrt{6a}\,h_{\max} + 4\sqrt{6a}\big(2\varepsilon^b_{
        \sf score} + h_{\max}^{1/2}\,\varepsilon^v_{\sf score}\big)
        \Big\}\sqrt{D},
    \end{align}
    with $h_{\max} = \max_{k} (t_{k+1}-t_k) \leqslant 
    \frac{\log(6a)(\log\log(6a) + 2T_1)}{K_0}$.
\end{theorem}

Since the conclusions of this theorem closely mirror 
those of \Cref{th:1}, the remarks provided after the 
latter apply here as well and will not be repeated. We merely emphasize two points. First, $P^*$ is 
not assumed to have a density wrt the Lebesgue 
measure on $\mathbb{R}^D$. Second, the number $K$ of queries 
to the score estimator required to achieve $\wass_2$ error $\varepsilon$ scales as $1/\varepsilon$, 
up to a factor that grows at most logarithmically 
in $1/\varepsilon$. For a 
log-concave distribution supported on a bounded domain, 
we have $(M,b) = (0, \diamX^2)$, so the exponential 
factor in the bound \eqref{eq:wass_bound1b} becomes a universal constant. This complements the result obtained in the 
strongly log-concave setting from \Cref{th:1}.

\section{Relation to prior work: extended discussion}

Given the wealth of recent work on Langevin algorithms  and  score-based generative models, it would be infeasible to provide an exhaustive account of all existing results. Instead, this section offers a synthetic and selective overview of prior work, to situate our contributions within the broader landscape rather than presenting a comprehensive literature review.

Theoretical guarantees for DDPMs have been largely inspired by techniques from the sampling literature, particularly those used for analyzing Langevin Monte Carlo and its variants; see the comprehensive overview in \cite{ChewiBook}. Prior work can be broadly grouped into three categories based on the underlying proof strategies.

The first category includes works that build on the approach initiated in \cite{DalalyanT12,Dalalyan17JRSS}, which combines the \textbf{Pinsker inequality with the Girsanov formula} to derive TV-distance bounds via the KL-divergence. This methodology has been developed further in \cite{ChenC0LSZ23,ChenL023,BentonBDD24,ConfortiDS25,li2025odt,liang2025broadening,HEWC024}. Its key strengths are:\vspace*{-8pt}
\begin{itemize}[leftmargin=20pt]\itemsep=-2pt
    \item it requires only a bound on the mean integrated squared error (MISE) of the score estimator—one of the weakest conditions in this framework;
    \item it relies on mild assumptions on the data-generating distribution~$P^*$.
\end{itemize}\vspace*{-7pt}
As noted in \cite{ChenC0LSZ23,ChenL023}, TV-distance bounds can be converted into Wasserstein bounds under additional assumptions, such as compact support or light-tailed $P^*$. If the support lies in a ball of radius~$R$, one can project the generated sample onto this ball and use that $\wass_2$ is bounded by $R $ times the TV distance. By the data-processing inequality, this projection does not increase the TV-error. 

However, this versatility comes at a price. Let 
$K_{\sf TV}(\wt\varepsilon)$ be the number of steps required to achieve an error smaller than $\wt\varepsilon$ in TV-distance. Then, to achieve $\wass_2$ error $\varepsilon$, one needs a TV-error $\wt\varepsilon = \varepsilon^2/R^2$, leading to a number of steps at least $K_{\sf TV}(\varepsilon^2/R^2)$. As a result, the rates derived from this strategy are suboptimal: ${O}(D^4/\varepsilon^2)$ in \cite{ChenL023}, ${O}(D/\varepsilon^4)$ in \cite{BentonBDD24,ConfortiDS25}, and ${O}(D^3/\varepsilon^2)$ in \cite{HEWC024}, ignoring logarithmic factors. Another limitation of this approach applied to DDPM
and other score-based generative models is that the resulting upper bound on the $\wass_2$ distance scales 
as the square root of the error of estimation of the score. As a consequence, to guarantee an error $\varepsilon$ in $\wass_2$, one needs a score estimator
that has an error $\varepsilon_{\sf score}$ 
bounded by $\varepsilon^2$. Our results, as well as
those of the third category below, typically require
the weaker condition $\varepsilon_{\sf score}\lesssim 
\varepsilon$. 

The second category comprises results that exploit the interpretation of Langevin dynamics as a \textbf{gradient flow in the space of probability measures}. This perspective was initiated in \cite{wibisono18a, Bernton18} and further developed in \cite{ChengB18,DurmusMM19,VempalaW19}. Interestingly, the first polynomial-in-dimension guarantees for DDPM fall within this framework, as shown in \cite{lee2022convergence,yang2022convergence}. These works evaluate the error in terms of $f$-divergences such as total variation, KL, or $\chi^2$ divergence. 
However, when translated to bounds in the $\wass_2$ distance, they suffer from the same limitations as the TV-based approaches discussed above. 
Moreover, this line of work typically relies on strong structural assumptions on the target distribution~$P^*$, notably the satisfaction of a log-Sobolev inequality. Another limitation---shared with our own analysis---is that the score estimation error is measured in the uniform norm. We believe, however, that this requirement could be relaxed, both in the gradient-flow framework and in the recursive method developed in our work.

The third category comprises works using the 
\textbf{recursive approach} to bound the error of 
iterative algorithms such as LMC or DDPM. This method,
widely used in optimization theory, was shown to yield
strong guarantees for sampling in 
\cite{Dalalyan17,DurmusM17,DurmusM19,DalalyanK19}. For 
DDPM, it underlies the analyses in 
\cite{Sabanis23,JMLR:v26:24-0902,strasman2025,Lu25}, 
which establish a $\wass_2$-error rate of order 
$D/\varepsilon^2$---an improvement over the bounds 
derived or derivable from the first two categories. 
However, despite having all the necessary ingredients, 
these works do not reach the faster rate 
$\sqrt{D}/\varepsilon$. This is somewhat surprising, 
especially since their assumptions on $P^*$ are often
quite strong, such as strong log-concavity. We believe 
this gap arises from not fully exploiting the 
smoothness of the score of the distribution obtained 
from $P^*$ by convolving with a Gaussian. Technically,
their recursive bounds relate the error at iteration 
$k$ to that at iteration $k-1$ via triangle inequalities,
which can be loose when the two terms
involved are weakly correlated. As we show, applying 
the recursive approach to the \textbf{squared} 
Wasserstein distance yields significantly tighter 
control and leads to optimal rates. We believe that 
this improvement can be further exploited to get even
faster rates using the randomized midpoint discretization 
\cite{ShenLee,Erdogdu20,yu2024langevin,Lu25} or to get a
faster algorithm exploiting parallelization 
\cite{ChenRYR24,anari24a, Gupta,yu2025parallel}. 



\section{Numerical experiments}\label{sec:numerical}

We supplement our theoretical results with a small-scale 
empirical study on CIFAR-10 \cite{Krizhevsky2009LearningML} 
CelebA-HQ \cite{karras2018progressivegrowinggansimproved}, and LSUN-Church \cite{yu15lsun},
evaluating the robustness of DDPMs to noise in the estimated 
score\footnote{Code is available at https://github.com/VahanArsenian/DiffusionWasserstein}.

\vspace*{-7pt}
\paragraph{Setup.}
We use pretrained DDPM models from the publicly available 
checkpoints \texttt{google/ddpm-cifar10-32}, 
\texttt{google/ddpm-celebahq-256}, and \texttt{google/ddpm-church-256}, all licensed under Apache license 2.0 and hosted on HuggingFace. 
For each model, we follow the standard DDPM sampling 
procedure, and then repeat the generation process while 
injecting noise into the score network~$s_\theta$ at 
every denoising step. Specifically, we replace the 
score function with a perturbed version $
\widetilde{s}_\theta(t, \bx) = s_\theta(t, \bx) + 
\bzeta$, where $\bzeta$ is a $D$-dimensional noise 
vector with independent and identically distributed 
components. We consider 4 noise distributions: centered \texttt{Uniform}, \texttt{Gaussian}, 
\texttt{Laplace}, and \texttt{Student's-t} with 3 
degrees of freedom. For each noise type, we evaluate 
6 values for the noise scale, $\sigma \in \{0.25, 
0.5, 1, 2, 3, 4\}$. All other elements of the generation 
pipeline---including the variance schedule, guidance 
scale, and number of sampling steps---are left unchanged. 
For each experimental setting, we generate 8192 CIFAR-10 
images and 8192 CelebA-HQ images. Additional implementation 
details can be found in \Cref{App:E}.

\vspace*{-7pt}
\paragraph{Qualitative results.}
\Cref{fig:noisy_gen_images} shows random generations 
for standard normal noise.  
We observe that injecting noise with constant 
variance into the score network has a negligible 
impact on the visual quality of the generated 
samples. As expected, the quality gradually 
degrades as the noise level increases. Additional 
qualitative results illustrating this phenomenon 
are provided in \Cref{App:E}.


\vspace*{-7pt}
\paragraph{FID sensitivity.}
The Fréchet Inception Distance (FID) is a widely 
used metric for assessing the quality of generative 
image models. In \Cref{fig:fid}, we plot the FID 
as a function of the noise scale~$\sigma$. On CelebA‑HQ, 
the FID increases only moderately up to $\sigma 
\approx 1$, while CIFAR‑10 exhibits robustness up 
to $\sigma \approx 2$. In agreement with our 
theoretical findings, the shape of the noise 
distribution has negligible impact, only its 
scale matters. We also observe a sharp degradation 
in quality beyond a certain noise threshold, 
a phenomenon not accounted for by our theoretical 
analysis.

    

\begin{figure}[h!]
    \centering
    \begin{subfigure}[t]{0.32\linewidth}
    \centering
    \includegraphics[width=\linewidth]{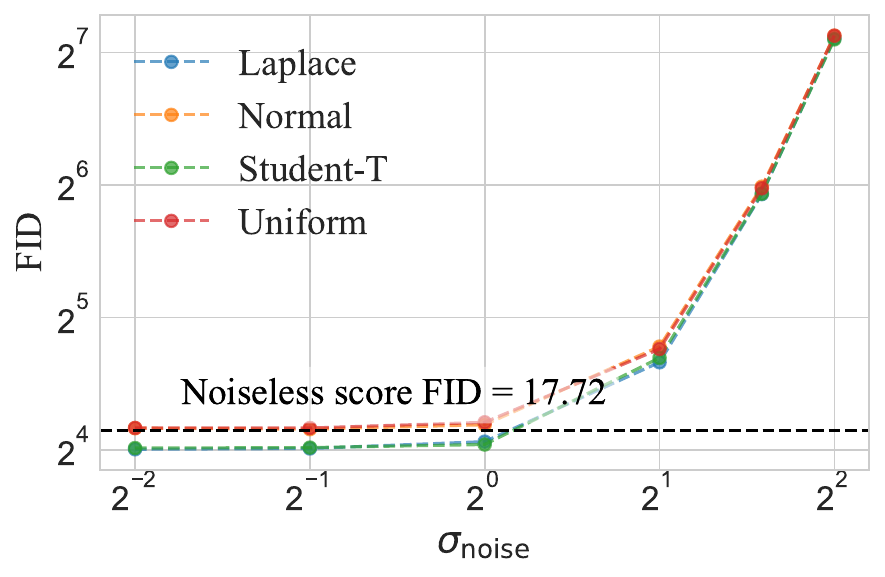}
    \caption{CIFAR-10}
    \end{subfigure}\hfill
    \begin{subfigure}[t]{0.32\linewidth}
    \centering
    \includegraphics[width=\linewidth]{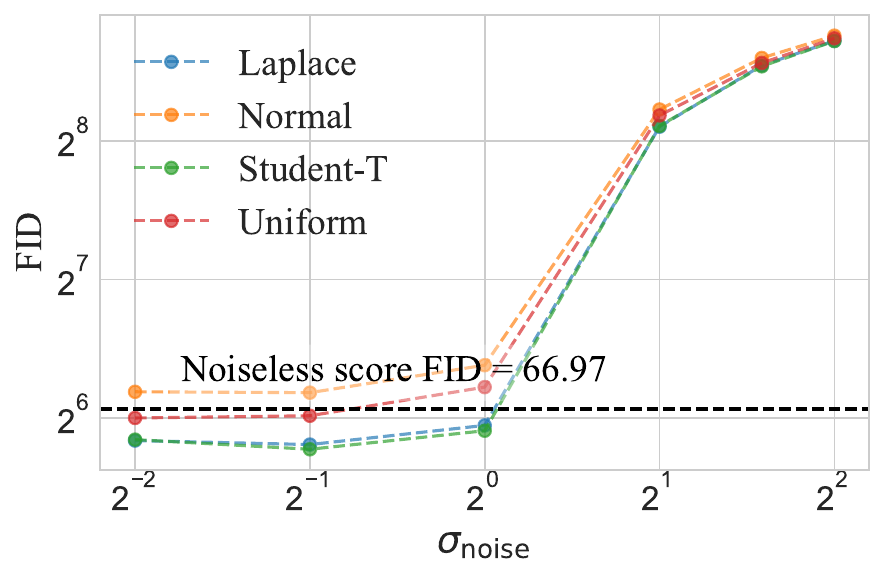}
    \caption{CelebA-HQ}
    \end{subfigure}\hfill
    \begin{subfigure}[t]{0.32\linewidth}
    \centering
    \includegraphics[width=\linewidth]{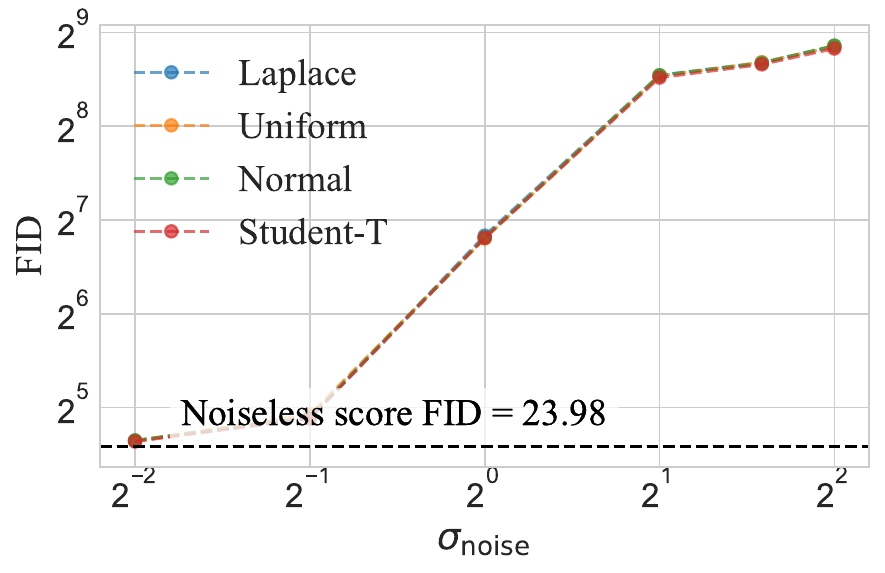}
    \caption{LSUN-Churches}
    \end{subfigure}
    
    \caption{FID as a function of noise level for four 
    distributions and different standard deviations.
    }
    \label{fig:fid}
\end{figure}



\section{Conclusion}
\label{conclusion}
In this paper, we provide a refined theoretical analysis of 
denoising diffusion probabilistic models (DDPMs), revealing two 
important features. First, we show that DDPMs exhibit robustness 
to noise in the estimated score function. Second, we establish 
that, when the true data-generating distribution belongs to a 
broad class---significantly larger than the class of log-concave distributions---DDPMs achieve fast convergence rates in the
Wasserstein distance.

Our findings open several avenues for future research.
One direction is the adaptation of our techniques to the 
analysis of kinetic Langevin diffusion-based DDPMs. It 
remains an open question whether such an extension would 
improve the dependence of the error bounds on the 
discretization step size. Additionally, the convergence 
rates we derive include terms that scale exponentially with 
certain parameters, such as the diameter of the support in 
the case of semi-log-concave targets. It is unclear whether 
this dependence is intrinsic to the problem or an artifact of 
our analysis. Finally, it would be of interest to assess  
the potential benefits of incorporating estimators of the 
Hessian of the log-density into the DDPM framework.


{\footnotesize
\bibliography{bibliography}
\bibliographystyle{alpha}
}
\newpage
\appendix	
\renewcommand\contentsname{}
\part{Appendix} 
\parttoc 

\newpage

\section[Classes of distributions satisfying Assumption~\ref{ass:1}]{Classes of distributions satisfying \Cref{ass:1}}
\label{app:A}

Throughout the paper we make use of Tweedie's formula \cite[Eq. 1.4]{tweedies_formula} which takes the following form using our notation: Let $\pi_{\bY}$ be the probability density function of $\bY = \alpha \bX+\beta \bxi$ where $(\bX,\bxi)\sim P^*\otimes \gamma^D$, then
\begin{align}\label{eq:grad_as_exp}
    \nabla\log \pi_{\bY}(\by)=\frac{\alpha}{\beta^2}\mathbf{E}\brs{\bX\mid\bY=\by}-\frac{\by}{\beta^2},\qquad \forall
    \by\in\mathbb R^D.
\end{align}

This section shows that distributions mentioned in  \Cref{sec:problem_statement} satisfy \Cref{ass:1}.

\subsection{Compactly supported distributions: property (a)}

\begin{lemma}\label{lemma:5}
    Let $P_{\bX,\bY}$ be a probability measure defined on $\mathcal{X}\times\mathcal{Y}$, $P_{\bX}$ and $P_{\bX\cond \bY=\by}$ be the marginal and the conditional distributions of $\bX$. Then 
    \begin{align}
        \operatorname{supp}(P_{\bX\cond \bY=\by})\subset\operatorname{supp}(P_{\bX}).
    \end{align}
\end{lemma}

\begin{proof}
    Let $S_{\bX}:=\operatorname{supp}(P_{\bX})$. Then by the definition of the marginal probability measure:
    \begin{align}
        P_{\bX}(S_{\bX})=P_{\bX,\bY}(S_{\bX}\times\mathcal{Y})=1.
    \end{align}

    On the other hand, by Bayes' theorem:
    \begin{align}\label{eq:eq:by}
        P_{\bX,\bY}(S_{\bX}\times\mathcal{Y})=P_{\bX\cond \bY=\by}(S_{\bX})P_{\bY}(\mathcal{Y}),
    \end{align}
    where $P_{\bY}$ is the marginal probability measure of $\bY$. The proof is completed by noting that \eqref{eq:eq:by}  yields $P_{\bX\cond \bY=\by}(S_{\bX})=1$.
\end{proof}

A simple consequence of \Cref{lemma:5} is that if $\operatorname{diam}(\operatorname{supp}(P_{\bX}))\leq C$ then $\operatorname{diam}(\operatorname{supp}(P_{\bX\cond \bY=\by}))\leq C$. Using this result, we show that a random vector $\bX$ with support diameter $2\diamX$ satisfies \Cref{ass:1} with $\varphi(\sigma)=\diamX^2$.

\begin{lemma}[Property (a) in \Cref{ssec:2.2}]\label{lem:bounded_supp}
    Let $\bX\sim P$ such that $\operatorname{diam}(\operatorname{supp}(P))\leq 2\diamX$ and let $\bY$ be any random variable defined on
    the same probability space. Then 
    \begin{align}
           \Var\big(\bX \cond \bY = \by \big) \preccurlyeq \diamX^2\bfI_D.
    \end{align}
\end{lemma}

\begin{proof}
    We need to prove that for any $\bm{v} \in \mathbb{R}^D$:
    \begin{align} \label{eq:scalar_ineq}
        \bm{v}^\top \Var(\bm{X}\cond \bY=\by) \bm{v} \le \bm{v}^\top \left( \diamX^2 \bfI_D \right) \bm{v},
    \end{align}
    which can be rewritten as:
    \begin{align}
        \Var(\bm{v}^\top\bm{X}\cond \bY = \by)\leq \|\bm{v}\|^2\diamX^2.
    \end{align}
    
    By dividing both sides by $\|\bm{v}\|^2$, we can rewrite the target inequality with respect to a unit vector $\bm{u}\in \mathbb{R}^D$:
    \begin{align}
        \Var(\bm{u}^\top\bm{X}\cond \bY = \by)\leq \diamX^2.
    \end{align}
    Denote $Z=\bm{u}^\top\bX$. The $\operatorname{supp}(P_Z)$ is contained in the set $\{\bm{u}^\top\bx\cond \bx\in \operatorname{supp}(P_{\bX\cond \bY=\by})\}$. By \Cref{lemma:5}, the $\operatorname{diam}(\operatorname{supp}(P_{\bX\cond \bY=\by}))\leq 2\diamX$. Let $z_1=\bm{u}^\top\bx_1$ and $z_2=\bm{u}^\top\bx_2$ for arbitrary $\bx_1, \bx_2\in\operatorname{supp}(P_{\bX\cond \bY=\by})$. The distance between them is:
    \begin{align}
        |z_1-z_2|=|\bm{u}^\top\bx_1-\bm{u}^\top\bx_2|=|\bm{u}^\top(\bx_1-\bx_2)|.
    \end{align}
    
    By the Cauchy-Schwarz inequality:
    \begin{align}
        |\bm{u}^\top (\bm{x}_1 - \bm{x}_2)| \le \|\bm{u}\|_2 \|\bm{x}_1 - \bm{x}_2\|_2.
    \end{align}

    Since $\|\bm{u}\|_2 = 1$, we write $|z_1 - z_2| \le \|\bm{x}_1 - \bm{x}_2\|_2.$
    The maximum possible value for $\|\bm{x}_1 - \bm{x}_2\|_2$ is the diameter $2\diamX$. Therefore, $|z_1 - z_2| \le 2\diamX$ for all $z_1, z_2$ in the support of $Z$. This implies that the support of $Z$ is contained within an interval $[a, b]$ such that the length of the interval $b - a \le 2\diamX$.
    We now apply Popoviciu's inequality on variances \cite{popovicius_ineq}, which yields that:
    \begin{align}
        \Var(Z\cond \bY=\by) \le \tfrac14\,(b - a)^2\le \diamX^2.
    \end{align}%
\end{proof}

\subsection{Log-concave and semi-log-concave distributions: properties 
(b) and (c)}
Random vectors with $m$-strongly log-concave densities also satisfy \Cref{ass:1}, as shown in the lemma below.
\begin{lemma}[Property (b) in \Cref{ssec:2.2}]\label{lem:lgc}
    Let $(\bX,\bxi) \sim P \otimes\gamma^D$, where the density of $P$, denoted as $\pi(\bx)$, is $m$-strongly log-concave. Then 
    \begin{align}
           \Var\big(\bX \cond \bX + \sigma \bxi = \by \big) \preccurlyeq \frac{\sigma^2}{1+m\sigma^2} \bfI_D.
    \end{align}
    In addition, if $\bx\mapsto \nabla \log \pi(\bx)$ is $M$-Lipschitz
    for some $M>0$, then 
        \begin{align}
           \Var\big(\bX \cond \bX + \sigma \bxi = \by \big) \succcurlyeq \frac{\sigma^2}{1+M\sigma^2} \bfI_D.
    \end{align} 
\end{lemma}

\begin{proof}
    By applying the preservation of strong log-concavity 
    \cite{log_convex}, we obtain that $\pi_{\bX+\sigma\bxi}(\by)$ 
    is $\frac{m}{1+m\sigma^2}$-strongly log-concave. We 
    then invoke \Cref{prop:1} with parameters $\alpha=1$ 
    and $\beta=\sigma$, which yields 
    \begin{align} 
        \frac{1}{\sigma^4}\Var(\bX\cond\bY=\by)-\frac{1}{\sigma^2}\bfI_D\preccurlyeq -
        \frac{m}{1+m\sigma^2}\bfI_D, 
    \end{align} 
    for $\bY = \bX+\sigma\bxi$, from which the first desired 
    result follows.

    For the second claim, set $\bY=\bX+\sigma\bxi$. The definition 
    of semi-log-concavity yields
    \begin{align}
        0\preccurlyeq -\nabla^2\log\pi(\bx)\preccurlyeq M\bfI_D.
    \end{align}
    The conditional density of $\bX$ given $\bY$ satisfies
    \begin{align}
        \pi_{\bX|\bY=\by}(\bx)\propto \pi_{\bX}(\bx)\pi_{\bY|\bX=\bx}(\by)
    \end{align}
    with $\pi_{\bY|\bX=\bx}(\by)\propto \exp(-\frac{\|\by-\bx\|^2}{2\sigma^2})$. Hence, the Hessian of $\pi_{\bX|\bY=\by}(\bx)$ is equal to:
    \begin{align}
        \nabla^2\log\pi_{\bX|\bY=\by}(\bx) = \nabla^2\ln\pi_{\bX}(\bx)-\frac{1}{\sigma^2}\bfI_D\succcurlyeq \brs{-M-\frac{1}{\sigma^2}}\bfI_D = -\frac{1+M\sigma^2}{\sigma^2}\,\bfI_D.
    \end{align}
    The Cramer-Rao inequality implies that
    \begin{align}
        \Var(\bX\cond\bY=\by) \succcurlyeq -\big(\mathbf{E}[\nabla^2 \log\pi_{\bX|\bY=\by}(\bX)\cond \bY=\by]\big)^{-1}
        \succcurlyeq \frac{\sigma^2}{1+M\sigma^2}\,\bfI_D
    \end{align}
    and the claim of the lemma follows.
\end{proof}
Similar results hold for for semi-log-concave distributions 
with a bounded support.
\begin{lemma}[Property (c) in \Cref{ssec:2.2}]
    Let $(\bX,\bxi) \sim P \otimes\gamma^D$ where $P$ has a density w.r.t. Lebesgue measure denoted as $\pi(\bx)$ and $\operatorname{diam}(\operatorname{supp}(P))\leq 2\diamX$. If  $\pi(x)$ is $M$-semi-log-concave for $M\in\mathbb R$, then:
    \begin{align}
         \Var\big(\bX \cond \bX + \sigma \bxi = \by \big) \preccurlyeq \diamX^2 \wedge \frac{\sigma^2}{(1 - 
    M\sigma^2)_+} \bfI_D.
    \end{align}
\end{lemma}

\begin{proof}
    Denote $\bY=\bX+\sigma\bxi$. We obtain from the definition of semi-log-concavity that:
    \begin{align}
        \nabla^2\log\pi(\bx)\preccurlyeq M\bfI_D.
    \end{align}
    The posterior of $\bX$ given $\bY$ is proportional to the joint:
    \begin{align}
        \pi(\bx\cond\by)\propto \pi(\bx)\pi(\by\cond\bx)
    \end{align}
    with $\pi(\by\cond\bx)\propto \exp(-\frac{\|\by-\bx\|^2}{2\sigma^2})$. Hence, the Hessian of $\log \pi(\bx\cond\by)$ 
    is equal to:
    \begin{align}
        \nabla^2\log\pi(\bx\cond\by)=\nabla^2\ln\pi(\bx) 
        - \frac{1}{\sigma^2}\bfI_D\preccurlyeq \brs{M-
        \frac{1}{\sigma^2}}\bfI_D,
    \end{align}
    where the last inequality follows from the semi-log-concavity of $\pi(\bx)$. By Brascamp-Lieb inequality \cite{brascamp1976best}, we have that:
    \begin{align}\label{eq:semi_log_conc}
        \Var(\bX\cond\bY=\by)\preccurlyeq\frac{\sigma^2}{1-M\sigma^2}\bfI_D
    \end{align}
    whenever $M\sigma^2\leqslant 1$. The conditional variance of $\bX$ can be bounded via \Cref{lem:bounded_supp}, as $P$ has a bounded support:
    \begin{align}
        \Var(\bX\cond\bY=\by)\preccurlyeq\diamX^2\bfI_D.
    \end{align}
    Combined with \eqref{eq:semi_log_conc}, we write:
    \begin{align}
        \Var\big(\bX \cond \bX + \sigma \bxi = \by \big) \preccurlyeq \diamX^2 \wedge \frac{\sigma^2}{(1 - 
    M\sigma^2)_+} \bfI_D.
    \end{align}
    This completes the proof of the lemma.
\end{proof}

\subsection{Stability by orthogonal transform and 
concatenation: properties (d) and (e)}

Afterwards, we prove that if $\bX$ satisfies \Cref{ass:1} then its rotation also satisfies \Cref{ass:1} with the same $\varphi(\sigma)$.

\begin{lemma}[Property (d) in \Cref{ssec:2.2}]
    Let $(\bX,\bxi) \sim P \otimes\gamma^D$ and 
    \begin{align}
        \Var\big(\bX \cond \bX + \sigma \bxi = \by \big)
        \preccurlyeq \varphi(\sigma)\,\bfI_D, 
        \qquad \forall \sigma > 0.
    \end{align}

    Then for any orthonormal matrix $\bfU$, we have that:
    \begin{align}
        \Var\big(\bfU\bX \cond \bfU\bX + \sigma \bxi' = \by' \big)
        \preccurlyeq \varphi(\sigma)\,\bfI_D, 
        \qquad \forall \sigma > 0.
    \end{align}

    for $\bxi'\sim\gamma^D$ and $\bxi'\indep \bfU\bX$.
\end{lemma}

\begin{proof}
    Consider $\Var\big(\bfU\bX \cond \bfU\bX + \sigma \bxi' = \by \big)$. We rewrite it as:
    \begin{align}
        \Var\big(\bfU\bX \cond \bfU\bX + \sigma \bxi' = \by' \big)&=\bfU\Var\big(\bX \cond \bfU\bX + \sigma \bxi' = \by' \big)\bfU^\top\\
        &=\bfU\Var\big(\bX \cond \bfU^\top\bfU\bX + \sigma \bfU^\top\bxi' = \bfU^\top\by' \big)\bfU^\top
    \end{align}
    Let $\by:=\bfU^\top\by'$ and $\bxi:=\bfU^\top\bxi'$. By using the properties that $\bfU^\top\bfU=\bfI_D$ as $\bfU$ is orthonormal and that $\bxi\sim\gamma^D$ independently from $\bX$ as $\bxi'\sim\gamma^D$ and $\bxi'\indep \bfU\bX$, we write:
    \begin{align}
        \Var\big(\bfU\bX \cond \bfU\bX + \sigma \bxi' = \by' \big)&=\bfU\Var\big(\bX \cond \bX + \sigma \bxi = \by \big)\bfU^\top
        \preccurlyeq\varphi(\sigma)\,\bfI_D,
    \end{align}
    and the claim of the lemma follows.
\end{proof}

We now show that the concatenation of two independent random vectors satisfying \Cref{ass:1}  also satisfies \Cref{ass:1}.

\begin{lemma}[Property (e) in \Cref{ssec:2.2}]
    Let $(\bX_1, \bX_2)\sim P_1\otimes P_2$, where $P_1$ and $P_2$ satisfy \Cref{ass:1} for some $\varphi$. Then the concatenation of $\bX_1$ and $\bX_2$ , denoted as $\bX_1\oplus\bX_2$ also satisfies \Cref{ass:1} for the same $\varphi$.
\end{lemma}

\begin{proof}
    Let $\bX_1$ be $d_1$-dimensional, $\bX_2$ be $d_2$-dimensional, and $D=d_1+d_2$. Consider $\bxi\sim\gamma^D$ and independent of $(\bX_1, \bX_2)$. We may write 
    \begin{align}
        \bY=\bX_1\oplus\bX_2+\sigma\bxi=\bX_1\oplus\bX_2+\sigma\brr{\bxi_1\oplus\bxi_2}=\underbrace{\brs{\bX_1+\sigma\bxi_1}}_{:=\bY_1}\oplus\underbrace{\brs{\bX_2+\sigma\bxi_2}}_{:=\bY_2}.
    \end{align}
    We have that $(\bX_1,\bX_2,\bxi_1,\bxi_2)$ are mutually independent as $(\bX_1,\bX_2,\bxi)$ are mutually independent and $\bxi_1$ and $\bxi_2$ are uncorrelated. From $\brr{\bX_1, \bxi_1}\indep(\bX_2, \bxi_2)$ we get that $\brr{\bX_1, \bY_1}\indep(\bX_2, \bY_2)$. Applying the weak union property of the conditional independence twice we get:
    \begin{align}
        \brr{\bX_1, \bY_1}\indep(\bX_2, \bY_2)\Rightarrow\bX_1\indep(\bX_2, \bY_2)\cond\bY_1\Rightarrow \bX_1\indep\bX_2\cond\brr{\bY_1, \bY_2}.
    \end{align}
    Hence the covariance of $\bX_1$ and $\bX_2$ given $(\bY_1, \bY_2)$ is $\bm{0}$. Finally,
    \begin{align}
        \Var(\bX_1\oplus\bX_2\cond \bY&=\by)=\Var(\bX_1\oplus\bX_2\cond \bY_1=\by_1, \bY_2=\by_2)\\
        &=\begin{bmatrix}
            \Var(\bX_1\cond \bY_1=\by_1, \bY_2=\by_2) &\cov(\bX_1, \bX_2\cond \bY_1=\by_1, \bY_2=\by_2) \\
            \cov(\bX_1,\bX_2\cond \bY_1=\by_1, \bY_2=\by_2) & \Var(\bX_2\cond \bY_1=\by_1, \bY_2=\by_2)
        \end{bmatrix}\\
        &=\begin{bmatrix}
            \Var(\bX_1\cond \bY_1=\by_1) &\bm{0}\\
            \bm{0} & \Var(\bX_2\cond \bY_2=\by_2)
        \end{bmatrix}\preccurlyeq \varphi(\sigma)\,\bfI_D
    \end{align}
    where the last inequality is due to $P_1$ and $P_2$ satisfying \Cref{ass:1}.
\end{proof}

\subsection{Convolution with a spherical Gaussian: property (f)}

\begin{lemma}[Property (f) in \Cref{ssec:2.2}]
    Let $(\bW,\bzeta)\sim P_0\otimes \gamma^D$. If $\bW$ satisfies
    \Cref{ass:1} with the function $\varphi_0$, then, for every 
    $\tau > 0 $, $\bX=\bW+\tau\bzeta$ satisfies \Cref{ass:1} with the
    function
    \begin{align}
        \varphi_\tau(\sigma) = \frac{\tau^2\sigma^2}{\tau^2+\sigma^2} 
        + \frac{\sigma^4\varphi_0(\sqrt{\tau^2+\sigma^2})}{(\tau^2 + \sigma^2)^2
        },\qquad \forall \sigma>0.
    \end{align}
\end{lemma}

\begin{proof}
    Let us define $\bY = \bX + \sigma\bxi = \bW + \tau\bzeta + \sigma\bxi$ and $\bfeta := \tau\bzeta + \sigma\bxi$. Since $\bxi, \bzeta \overset{\text{i.i.d.}}{\sim} \gamma^D$ are independent of $\bW$, 
    we have $\bfeta\sim\gamma^D$ with covariance $(\tau^2 + \sigma^2) 
    \bfI_D$ and $\bY = \bW + \bfeta$. Equivalently,
    \begin{align}
        \bY = \bW + \sqrt{\tau^2 + \sigma^2} \; \bxi', \quad \bxi' \sim \gamma^D, \bxi'  \indep \bW.
    \end{align}
    Using \Cref{ass:1} with noise level $\sqrt{\tau^2 + \sigma^2}$ 
    leads to
    \begin{align}
        \Var(\bW \cond \bY = \by) \preccurlyeq \varphi_0 (\sqrt{\tau^2 + \sigma^2}) \; \bfI_D.
    \end{align}
    To ease notation, we write $\mathbf E_{\by}$ and 
    $\Var_{\by}$ to refer to the conditional expectation 
    and conditional variance given $\bY=\by$, respectively. 
    By the law of total variance, we have
    \begin{align}
        \Var_{\by}(\bX) = \mathbf E_{\by} \left[ \Var(\bX \cond \bY=\by, \bW) \right] + \Var_{\by} \left( \mathbf E[\bX \cond \bY=\by, \bW] \right).
        \label{cond_var}
    \end{align}
    We know that $\tau \bzeta$ and $\bfeta = \tau \bzeta + \sigma \bxi$ are linear transforms of two independent standard Gaussians. Hence, the standard covariance calculation gives us
    \begin{align}
        \Var(\tau\bzeta\mid\bfeta) &=\tau^{2}\bfI_{D} - \tau^{2}\bfI_{D}(\tau^{2} + \sigma^{2})^{-1}\bfI_{D}\tau^{2}\bfI_{D}\\
        &= \frac{\tau^2 \sigma^2}{\tau^2 + \sigma^2} \bfI_D.
    \end{align}
    And since $\Var(\bX \cond \bY=\by, \bW) = \Var(\tau\bzeta\mid\bfeta)$, we get the first part of \eqref{cond_var} equal to
    \begin{align}
        \mathbf E_{\by}\left[\Var(\bX \cond \bY=\by, \bW) \right] = \frac{\tau^2 \sigma^2}{\tau^2 + \sigma^2} \bfI_D.
    \end{align}
    For the second term, since $\bzeta,\bxi\stackrel{\text{i.i.d.}}\sim\mathcal N(\mathbf0,\mathbf I_D)$, then the corresponding  $2D$-dimensional vector 
    \begin{align}
        \begin{pmatrix} \tau\bzeta \\ \bfeta \end{pmatrix} \sim \mathcal N \left( \mathbf 0, \bSigma \right), \quad \text{with } \bSigma =
        \begin{pmatrix} \Var(\tau\bzeta) & \cov(\tau\bzeta,\bfeta) \\
        \cov(\bfeta,\tau\bzeta) & \Var(\bfeta) \end{pmatrix} = \begin{pmatrix} \tau^{2}\bfI_D & \tau^{2}\bfI_D \\
        \tau^{2}\bfI_D & (\tau^{2}+\sigma^{2})\bfI_D \end{pmatrix}.
    \end{align}
    So, the conditional expectation that we are interested in will be equal to
    \begin{align}
        \mathbf E[\tau\bzeta\mid\bfeta] = \tau^{2}\bfI_D
     (\tau^{2}+\sigma^{2})^{-1}\bfI_D \bfeta = \frac{\tau^{2}}{\tau^{2}+\sigma^{2}}\bfeta.
    \end{align}
    Under the conditioning on both $\bY$ and $\bW$, the quantity $\bfeta = \bY -\bW$ is deterministic. Therefore,
    \begin{align}
        \mathbb E[\bX\mid\bY=\by,\bW] &= \bW + \mathbb E[\tau\bzeta\mid\bfeta=\by-\bW] \\
        &= \bW + \frac{\tau^{2}}{\tau^{2}+\sigma^{2}}(\by-\bW)\\
        &= \frac{\sigma^{2}}{\tau^{2}+\sigma^{2}}\bW +\frac{\tau^{2}}{\tau^{2}+\sigma^{2}}\by .
    \end{align}
    Given $\bY=\by$, the second term is deterministic, so
    \begin{align}
        \Var_{\by}\bigl(\mathbf{E}[\bX\mid\bY=\by,\bW] \bigr) =\left(\frac{\sigma^{2}}{\tau^{2}+\sigma^{2}}\right)^2 
        \,\Var(\bW\cond \bY=\by) \preccurlyeq \frac{\sigma^4 \varphi_0(\sqrt{\tau^2+\sigma^2})}{(\tau^2+\sigma^2)^2} \bfI_D.
    \end{align}
    Adding the two components gives us
    \begin{align}
        \Var(\bX \cond \bY=\by) \preccurlyeq \left( \frac{\tau^2 \sigma^2}{\tau^2 + \sigma^2}  +  \frac{\sigma^4 \varphi_0(\sqrt{\tau^2+\sigma^2})}{(\tau^2+\sigma^2)^2} \right) \bfI_D,
    \end{align}
    which proves the lemma with $\varphi_\tau(\sigma) = \frac{\tau^2 \sigma^2}{\tau^2 + \sigma^2}  +  \frac{\sigma^4 \varphi_0(\sqrt{\tau^2+\sigma^2})}{(\tau^2+\sigma^2)^2}$.
\end{proof}

\subsection{Convolution of a semi-log-concave and a compactly supported distribution: property (g)} 

\begin{lemma}[Property (g) in \Cref{ssec:2.2}]
    If $P^* = P_{\sf slc} \star P_{\sf cmpct}$, where $P_{\sf slc}$ 
    is an $m$-strongly log-concave distribution with an 
    $M$-Lipschitz score function, and $P_{\sf cmpct}$ is 
    supported on a bounded set with diameter $2\mathfrak D$, then
    $P^*$ satisfies \Cref{ass:1} with 
    \begin{align}\label{eq:phi_slg1}
        \varphi(\sigma) = \frac{\sigma^2}{ 1 + m\sigma^2} +
        \frac{\mathfrak D^2M ^2\sigma^4}{(1 + M\sigma^2)^2},
        \qquad \forall \sigma>0,
    \end{align}    
\end{lemma}

\begin{proof}
    Let $\bW \sim P_{\sf cmpct}$ and $\bzeta\sim P_{\sf slc}$  
    be two independent random vectors so that $\bX = \bW + \bzeta 
    \sim P^*$. This means that for some bounded set $\mathcal K$ 
    with diameter $2\mathfrak D$, we have $\Var(\bW)\leq 4\mathfrak 
    D^2$, and that the density $\pi_{\bzeta}$ is continuously 
    differentiable with a score function $\score_{\bzeta}$ satisfying
    \begin{align}
        m\|\bx - \bx'\|^2\leqslant (\bx - \bx')^\top 
        \big(\score_{\bzeta}(\bx) - \score_{\bzeta}(\bx') \big)
        \leqslant M\|\bx - \bx'\|^2. 
    \end{align}
    For $\bxi\indep(\bW,\bzeta)$ such that $\bxi\sim 
    \gamma^D$, and for $\bY = \bX + \sigma\,\bxi$, we have to 
    prove that
    \begin{align}
        \Var(\bX\cond \bY=\by) \preccurlyeq \Big(\frac{\sigma^2}{ 
        1 + m\sigma^2} + \frac{\mathfrak D^2M ^2\sigma^4}{(1 +
        M\sigma^2)^2}\Big)\, \bfI_D.   
    \end{align}
    As before, to ease notation, we write $\mathbf E_{\by}$ and 
    $\Var_{\by}$ to refer to the conditional expectation 
    and conditional variance given $\bY=\by$, respectively. By the law of total variance, we have
    \begin{align}
        \Var_{\by}(\bX) = \mathbf E_{\by} \left[ \Var(\bX \cond 
        \bY=\by, \bW) \right] + \Var_{\by} \left( \mathbf E[\bX \cond \bY=\by, \bW]\right).
        \label{cond_var2}
    \end{align}
    Since the random vector $\bzeta$ is $m$-strongly log-concave, 
    it follows from \Cref{lem:lgc} that 
    \begin{align}
        \Var(\bzeta \cond \bzeta + \sigma\bxi = \by') \leqslant 
        \frac{\sigma^2}{1+m\sigma^2},\qquad \forall \by'\in\mathbb R^D. 
    \end{align}
    Therefore,
    \begin{align}
        \Var(\bX \cond \bY=\by, \bW=\bw) = \Var(\bzeta \cond 
        \bzeta + \sigma\bxi  =\by-\bw) \leqslant  \frac{\sigma^2}{1+m\sigma^2},\qquad \forall \by,\bw\in\mathbb R^D.
    \end{align}
    Hence,  $\Var(\bX \cond \bY=\by, \bW) \leqslant \frac{\sigma^2}{1+m\sigma^2}$ almost surely. This implies that 
    \begin{align}
        \mathbf E_{\by} \left[ \Var(\bX \cond \bY=\by, \bW) \right]
        \preccurlyeq \frac{\sigma^2}{1+m\sigma^2}\,\bfI_D. 
    \end{align}
    We switch to assessing the second term in \eqref{cond_var2}. 
    It holds that
    \begin{align}
        \mathbf E[\bX \cond \bY=\by, \bW=\bw] &
        \stackrel{\numcircled{1}}{=} \bw + \mathbf E[
        \bzeta \cond \bzeta + \sigma\bxi = \by -\bw, \bW=\bw]\\
        & \stackrel{\numcircled{2}}{=} \bw + \mathbf E[\bzeta \cond \bzeta + \sigma\bxi = 
        \by -\bw]\\
        & \stackrel{\numcircled{3}}{=} \bw + \sigma^2 \nabla \log \pi_{\bzeta+\sigma\bxi}
        (\by-\bw) + \by-\bw\\
        & = \by + \sigma^2 \nabla \log \pi_{\bzeta+\sigma\bxi}(\by-\bw),
    \end{align}
    where \numcircled{1} is a consequence of $\bX = \bW + \bzeta$, 
    \numcircled{2} follows from the independence of $\bzeta$ and $\bW$, \numcircled{3} is obtained by the Tweedie formula recalled in \eqref{eq:grad_as_exp}. Let us set $\psi(\bw) = \nabla \log \pi_{\bzeta+\sigma\bxi}(\by-\bw)$. The second claim of 
    \Cref{lem:lgc} combined with \Cref{prop:1} implies that 
    $\psi$ is Lipschitz-continuous with the constant $M/(1 + 
    M\sigma^2)$. Therefore, 
    \begin{align}
        \Var_{\by}(\mathbf E[\bX\cond \bY=\by,\bW]) = \sigma^4
        \Var_{\by}\big(\psi(\bW)\big) \preccurlyeq 
        \frac{M^2\sigma^4}{(1+M\sigma^2)^2}\,\Var_{\by}\big(\bW\big)
        \leqslant \frac{M^2\sigma^4\mathfrak D^2}{(1+M\sigma^2)^2}\,
        \bfI_D,
    \end{align}
    where in the last step we used  
    \Cref{lem:bounded_supp}.
\end{proof}

\section[Proof of Lemma~\ref{lem:noDPI}]{Proof of \Cref{lem:noDPI}}


    We start by first proving that:
    \begin{align}
         \sup_{P^*\in \mathcal N} \frac{\sfd_{\mathsf{TV}
        }^2(Q_D^{T,\score^*}; P^*)}{\sfd_{\mathsf{TV}
        }^2(\gamma^D;P^*)} \bigvee \frac{\sfd_{\mathsf{KL}
        }(Q_D^{T,\score^*}|\!| P^*)}{\sfd_{\mathsf{KL}
        }(\gamma^D|\!|P^*)} 
        \leqslant e^{-2T}
    \end{align}
    The data processing inequality \cite{polyanskiy-2017} states that: 
    \begin{align}
        \sfd_{\mathsf{TV}}(Q_D^{T,\score^*}; P^*)\leq\sfd_{\mathsf{TV}}(\gamma^D;P_T^*);\qquad \sfd_{\mathsf{KL}}(Q_D^{T,\score^*}; P^*)\leq\sfd_{\mathsf{KL}}(\gamma^D;P_T^*).
    \end{align}

    Combined with the concentration property of Ornstein–Uhlenbeck process \cite{GaoZhu, eberle-2019}:
    \begin{align}
        \sfd_{\mathsf{TV}}(\gamma^D;P_T^*)\leq\sfd_{\mathsf{TV}}(\gamma^D;P^*)e^{-T};\qquad \sfd_{\mathsf{KL}}(\gamma^D;P_T^*)\leq\sfd_{\mathsf{KL}}(\gamma^D;P^*)e^{-2T}.
    \end{align}
    gives the desired result.
    
        We now focus on a subset of $\mathcal{N}'\subset\mathcal{N}$ that contains $D$ dimensional Gaussian distributions with mean $\bm{0}$ and $(1+\sigma^2)\bfI_D$ covariance matrix with $\sigma>0$.
        Clearly
        \begin{align}
            \sup_{P^*\in \mathcal N'} \frac{\wass_2(Q_D^{T,\score^*};
        P^*)}{\wass_2(P^*;\gamma^D)}\leq\sup_{P^*\in \mathcal N} \frac{\wass_2(Q_D^{T,\score^*};
            P^*)}{\wass_2(P^*;\gamma^D)}
        \end{align}
    Let $\bX_t$ be defined by Equation \eqref{eq:ou-forward}, then the distribution of $\bX_t$ is $\mathcal{N}(\bm{0}, (e^{-2t}\sigma^2+1)\,\bfI_D)$. Hence, the true score function is
    \begin{align}
        \wt\score(\bx)=-\bx/\sigma^2(t),
    \end{align}
    where $\sigma^2(t)=e^{-2t}\sigma^2+1$.
    Equation \eqref{backward2} obtains the following form under this score function:
    \begin{align}
        \rmd \wt\bY_t = \brs{\wt\bY_t\brr{1-\frac{2}{\sigma^2(T-t)}}} \,\rmd t + \sqrt{2}\, \rmd 
    \wt{\bB}_t.
    \end{align}

    The integrating factor for the SDE is:
    \begin{align}
           I(t)=& \exp\brr{-\int_0^t1-\frac{2}{\sigma^2(T-u)}\  \rmd u}\\
   =&\exp\brr{-t+\int_0^t\frac{2}{\exp(2(u-T))\sigma^2+1}\ \rmd u} \\
   =&\exp\brr{-t+2t+\ln\brr{\frac{\sigma^2+e^{2T}}{\sigma^2e^{2t}+e^{2T}}}} \\
   =&e^t\frac{\sigma^2+e^{2T}}{\sigma^2e^{2t}+e^{2T}}.
    \end{align}
    From Itô’s product rule applied to $I(t)\wt\bY_t$, we get:
    \begin{align}\label{eq:ito_st_gauss}
        \rmd I(t)\wt\bY_t=I(t)\brs{f(t)\wt\bY_t\ \rmd t+\sqrt{2}\ \rmd \wt\bB_t}-I(t)f(t)\wt\bY_t\ \rmd t=\sqrt{2}I(t)\ \rmd \wt\bB_t,
    \end{align}
    where we have used the fact that $\rmd I(t) = -I(t)\brr{1-\frac{2}{\sigma^2(T-t)}}\ \rmd t$.

Integrating both sides of \eqref{eq:ito_st_gauss} from $0$ to $t$:
$$
I(t)\wt\bY_t=\wt\bY_0+\sqrt{2}\int_0^tI(u) \ \rmd{\wt\bB_u}
$$
from which:
\begin{align}\label{eq:closed_form_y_t}
    \wt\bY_t=\frac{\wt\bY_0+\sqrt{2}\int_0^tI(u)\ \rmd \wt\bB_u}{I(t)}.
\end{align}
Note that $\wt\bY_0\sim\gamma^D$. Combined with the fact that $I(t)$ is a deterministic function, we infer from \eqref{eq:closed_form_y_t} that $\wt\bY_t$ is a zero mean Gaussian random variable. So the Wasserstein distance between $\gamma^D$ and the distribution of $\wt\bY_t$ depends only on the covariance matrices:
\begin{align}\label{eq:w2_base}
    \wass_2(Q_D^{T,\score^*};
            P^*)=\|\sigma_{\wt\bY_t}\bfI_D - \sqrt{\sigma^2+1}\bfI_D\|_{F}=
            \big|\sigma_{\wt\bY_t} - \sqrt{\sigma^2+1}\big|\sqrt{D},
\end{align}
where $\sigma^2_{\wt\bY_t}\bfI_D$ is the covariance of $\wt\bY_t$. 

Let $\bZ_t:=\sqrt{2}\int_0^tI(u)\ \rmd \wt\bB_u$. Hence, $\bZ_t\sim \mathcal{N}(\bm{0}, 2\int_0^tI^2(u)\ \rmd u\,\bfI_D)$ and it is independent of $\wt\bY_0$. The variance of $\bZ_t$ is:
\begin{align}
    \sigma^2_{\bz_t} = 2\int_0^tI^2(u)du=\frac{(e^{2t}-1)(e^{2T} +\sigma^{2})}{e^{2T}+\sigma^{2}e^{2t}}\quad
    \text{and}\quad
        \sigma^2_{\bz_T} = \frac{(1-e^{-2T})(e^{2T} +\sigma^{2})}{\sigma^{2}+1}
\end{align}
    The variance of $\wt\bY_T$ can be computed from Equation \eqref{eq:closed_form_y_t}:
    \begin{align}
        \sigma^2_{\wt\bY_T} & =\frac{1+\sigma^2_{\bZ_T}}{I^2(T)} = \frac{(\sigma^2 + 1)(2\sigma^2e^{2T}-\sigma^2+e^{4T})}{
        (\sigma^{2} + e^{2T} )^2}= (\sigma^2 + 1)
        \Big[1 - \frac{\sigma^2(\sigma^2+1)}{
        (\sigma^{2} + e^{2T} )^2}\Big].
    \end{align} 
    Plugging in the value of $\sigma_{\wt\bY_T}$ into  \eqref{eq:w2_base}  we get:
    \begin{align}
        \wass_2(Q_D^{T,\score^*};
            P^*)=\left(1 - \Big\{1-\frac{\sigma^2(\sigma^2+1)}{(\sigma^2+e^{2T})^2}\Big\}^{1/2} \right)\sqrt{(\sigma^2+1)D}
            \stackrel{\sigma\to\infty}{\sim} \sigma \sqrt{D}.
    \end{align}
    We note that $\wass_2(P^*;\gamma^D)=|\sqrt{\sigma^2+1}-1|\sqrt{D}\stackrel{\sigma\to\infty}{\sim} \sigma\sqrt{D}$, so we have
    \begin{align}
    r(\sigma)=\frac{\wass_2(Q_D^{T,\score^*};
        P^*)}{\wass_2(P^*;\gamma^D)}
        \xrightarrow[\sigma\to\infty]{}
        1.
    \end{align}
    Hence, 
    \begin{align}
        \sup_{P^*\in \mathcal N'} \frac{\wass_2(Q_D^{T,\score^*};
            P^*)}{\wass_2(P^*;\gamma^D)} \geq 1.
    \end{align}
    
   When combined with the established contraction behavior of the backward diffusion—operating with the true score function—in the 2-Wasserstein metric for Gaussian distributions \cite{eberle-2019}, we get:

    \begin{align}
        1\leq\sup_{P^*\in \mathcal N} \frac{\wass_2(Q_D^{T,\score^*};
        P^*)}{\wass_2(P^*;\gamma^D)}\le1.
    \end{align}


\section{Proofs of the main results}
\label{sec:proof}

We recall that $P^*$ is the target distribution and
$P^*_t = \alpha_t P^* + \beta_t \gamma^D$ is the distribution
of the forward process at time $t>0$, with $\alpha_t = e^{-t} = 
\sqrt{1-\beta^2_t}$. We also fix some $T>0$  and define 
$\bY_t = \bX_{T-t}$ and  $Q_t^* = \text{Law}(\bY_t)$;  
$\bY_t$ is the state of the backward process \eqref{eq:backward}. 
We set $\widetilde P_{k}$ to be the law of $\bZ_k$ defined 
by \eqref{discr} so that $P^{\sf DDPM} = \widetilde P_{K+1}$.  
Throughout this proof, we will repeatedly use the following notation:
\begin{align}
    \bar m_2 &= 1\vee (\|\bX\|_{\mathbb L_2}/\sqrt{D}),\\
    \varepsilon_{k}^b &= \|\mathbf{E}[\widetilde\score(T-t_k,\bZ_k)
    \cond \mathcal{F}_{k}]
    - \score(T-t_k,\bZ_k)\|_{\mathbb L_2},\quad \varepsilon^b = \max_k \varepsilon_{k}^b\\ 
    \varepsilon_k^v &= \|\widetilde\score(T-t_k,\bZ_k) - 
    \mathbf{E}[\widetilde\score(T-t_k,\bZ_k)\cond \mathcal{F}_{k}]\|_{\mathbb L_2},\quad \varepsilon^v = \max_k \varepsilon_{k}^v.
\end{align}

\subsection{Main recursion}\label{ssec:6.1}
We set $T = t_{K+1}$ and consider a version of the 
continuous-time process $(\bY_t)_{0\leq t\leq T}$ and the
discrete-time process $(\bZ_k)_{0\leq k\leq K+1}$ defined 
on the same probability space and coupled by the relation
$\bxi_{k+1} = (\wt{\bB}_{t_{k+1}} - \wt{\bB}_{t_k})/ 
\sqrt{h_{k}}$. We then use the definition of the 
Wasserstein distance to infer that
\begin{align}
    \wass_2(P^*,\widetilde P_{K+1}) &=\wass_2 (Q^*_{t_{K+1}},
    \widetilde P_{K+1}) \leqslant \|\bY_{t_{K+1}} 
    - \bZ_{K+1}\|_{\mathbb L_2}. \label{decomp1}
\end{align}
Combining \eqref{eq:backward} and \eqref{discr}, in conjunction
with the relation $\sqrt{h_{k}}\,\bxi_{k+1} =
(\wt{\bB}_{t_{k+1}} - \wt{\bB}_{t_k})$, we get
\begin{align}
    \bY_{t_{k+1}} - \bZ_{k+1}  = (1 + h_k)(\bY_{t_k} - \bZ_k)
    &+ 2h_k\big(\score(T-t_k,\bY_{t_k}) - \score(T-t_k, 
    \bZ_k) \big)\\
    &  - 2h_k\big(\widetilde\score(T-t_k,\bZ_k) 
    - \score(T-t_k,\bZ_k) \big)\\
    & + \int_{t_k}^{t_{k+1}} \Big\{\bY_t-\bY_{t_k}+2\score(T-t,
    \bY_{t}) - 2\score(T-t_k,\bY_{t_k})\Big\}\,\rmd t. 
    \label{delta_k1}
\end{align}
In what follows, we use the notation $\bDelta_k = \bY_{t_k} - 
\bZ_k$ and 
\begin{align} 
    \bU_k &= \score(T-t_k,\bY_{t_k}) - \score(T-t_k,\bZ_k);\\
    \bzeta_k &= \widetilde\score(T-t_k,\bZ_k) - \score
    (T-t_k,\bZ_k)\label{def:Uk}\\
    \bV_k &= \int_{t_k}^{t_{k+1}} \Big\{\bY_t-\bY_{t_k}+2\score(T-t,
    \bY_{t}) - 2\score(T-t_k,\bY_{t_k})\Big\}\,\rmd t. \label{def:Vk}
\end{align}
This allows us to rewrite \eqref{delta_k1} as follows
\begin{align}\label{delta_k2}
    \bDelta_{k+1} = (1+{h_k})\bDelta_k + 2h_k\bU_k 
    -2h_k\bzeta_k + \bV_k.
\end{align}
In view of \eqref{decomp1}, we are interested in bounding
the term
\begin{align}
    x_K:= \|\bDelta_K\|_{\mathbb L_2}.
\end{align}
We will proceed by establishing a recursive inequality
upper bounding $x_{k+1}$ by a simple expression involving
$x_k$, and then by unfolding this recursive inequality.

Let us introduce the filtration $(\mathcal F_k)_{k\in\mathbb N}$. 
The first element of this sequence is the $\sigma$-algebra 
generated by $\bY_0$ and $\bZ_0$. Then, each $\mathcal F_{k+1}$
is obtained by extending $\mathcal F_k$ to the smallest 
$\sigma$-algebra for which both $\bzeta_k$ and the process 
$(\tilde \bB_t - \tilde\bB_{t_k})_{t\in[t_{k};t_{k+1}]}$ are
measurable. Note that $Z_k$ is necessarily $\mathcal F_k$-measurable, 
but the same is not true for $\bzeta_k$. Indeed, the estimator
$\tilde \score (T-t_k,\cdot)$ may depend on some random variables
that are not in $\mathcal F_k$.

It is clear that
\begin{align}
    \mathbf E\big[\|\bDelta_{k+1}\|^2\big] & = 
    \mathbf E\big[\|\mathbf E[\,\bDelta_{k+1} \cond \mathcal F_k]
    \|^2\big] + \mathbf E\big[\|\bDelta_{k+1} - \mathbf E[ 
    \bDelta_{k+1}\cond\mathcal F_k]\|^2\big]\\
    &= \|\mathbf E[\,\bDelta_{k+1} \cond \mathcal F_k]\|^2_{ 
    \mathbb L_2} + \|\bDelta_{k+1} - \mathbf E[ 
    \bDelta_{k+1} \cond \mathcal F_k]\|^2_{ \mathbb L_2}.
    \label{decomp:a1}
\end{align}
From \eqref{delta_k2},  by the triangle inequality, 
\begin{align}
\|\mathbf E[\,\bDelta_{k+1} \cond \mathcal F_k]\|_{\mathbb L_2}
    &\leqslant \|(1+{h_k})\bDelta_{k} + 2h_k\bU_k
    \|_{\mathbb L_2} + 2h_k\|\mathbf E[\,\bzeta_k \cond \mathcal 
    F_k]\|_{\mathbb L_2} + \|\mathbf E[\,\bV_k \cond \mathcal F_k]
    \|_{\mathbb L_2}.\label{decomp:a2}
\end{align}
Furthermore,
\begin{align}
    \|\bDelta_{k+1} - \mathbf E[ 
    \bDelta_{k+1} \cond \mathcal F_k]\|_{ \mathbb L_2}
    &\leqslant 2h_k\|\bzeta_k - \mathbf E[\,\bzeta_k\cond\mathcal F_k]
    \|_{ \mathbb L_2} + \|\bV_k - \mathbf E[ \bV_k\cond \mathcal 
    F_k]\|_{ \mathbb L_2}. \label{decomp:a3}
\end{align}
Combining displays \eqref{decomp:a1}, \eqref{decomp:a2} and
\eqref{decomp:a3}, we arrive at
\begin{align}
    \mathbf E\big[\|\bDelta_{k+1}\|^2\big] & \leqslant 
    \Big(\|(1+{h_k})\bDelta_{k} + 2h_k\bU_k\|_{\mathbb L_2} 
    + 2h_k\underbrace{\|\mathbf E[\,\bzeta_k \cond \mathcal F_k]
    \|_{\mathbb L_2}}_{\varepsilon_k^b:=\textrm{bias of estim.\ 
    score}} + \underbrace{\|\mathbf E[\,\bV_k\cond \mathcal F_k]
    \|_{\mathbb L_2}}_{\mathfrak B_k:=\textrm{bias of discr.\ 
    error}}\Big)^2\\
    &\qquad + \Big(2h_k\underbrace{\|\bzeta_k - \mathbf E 
    [\bzeta_k\cond\mathcal F_k]\|_{ \mathbb L_2}}_{\varepsilon_k^v 
    :=\textrm{variance of estim.\ score}} + \underbrace{\|\bV_k 
    - \mathbf E[ \bV_k\cond \mathcal F_k]\|_{ \mathbb L_2}}_{
    \mathfrak V_k:=\textrm{variance of discr.\ error}}\Big)^2.
    \label{decomp:2}
\end{align}
In what follows, it is convenient to use the following notation:
for every $k\in\mathbb N$, let $\alpha_k = e^{-(T-t_k)}$ and 
$\beta_k^2 = 1-\alpha_k^2$. 

\begin{lemma}\label{lem:contraction}
    If $P^*$ satisfies \Cref{ass:1} with a function $\varphi$ 
    and  
    \begin{align}
         h_k\bigg(\frac{1+\alpha_k^2}{1-\alpha_k^2} + m_k\bigg)
         &\leqslant 2,\quad \text{for}\quad
        m_k = 1+ \frac{2\alpha_k^2}{1-\alpha_k^2}\Big(1 -\frac{\varphi(\beta_k/\alpha_k)}{
        1 - \alpha_k^2}\Big)
        \label{cond:h}
    \end{align}
    then, 
    \begin{align}\label{eq:contraction}
        \|(1+{h_k})\bDelta_k +2h_k\bU_k\|_{\mathbb L_2} 
        \leqslant 
        \big(1 - m_k h_k\big)\|\bDelta_k\|_{\mathbb L_2}.
    \end{align}
\end{lemma}

\Cref{lem:contraction} implies that
\begin{align}\label{eq:recursion1}
    x_{k+1}^2 &\leqslant \big((1-m_kh_k)x_k + 2h_k\varepsilon_k^b
    + \mathfrak B_k\big)^2 + \big(2h_k\varepsilon_k^v + 
    \mathfrak V_k\big)^2. 
\end{align}

The next lemma which can be easily deduced by induction applying
the Minkowski inequality, will be used to derive a global bound
on the error $x_K$ from recursive inequalities upper bounding
the error $x_{k+1}$ at the $(k+1)$th step by the one of the 
$k$th step.
\begin{lemma}\label{lem:recursion}
    Let $(A_k)_{k\in\mathbb N}$, $(B_k)_{k\in\mathbb N}$ and 
    $(C_k)_{k\in\mathbb N}$ be three sequences of real numbers
    such that $B_k\geq 0$ and $C_k\geq 0$ for every $k$.
    If $(x_k)_{k\in\mathbb N}$ satisfies the recursive inequality
    \begin{align}
        x_{k+1}^2 &\leqslant (e^{A_k}x_k + B_k)^2 + C^2_k,
        \qquad\forall k\geqslant 0,
    \end{align}
    then, for $\bar A_k = A_0+\ldots+A_k$,
    \begin{align}
        x_{k+1} \leqslant e^{\bar A_k}x_0 + \sum_{j=0}^k 
        e^{\bar A_k - \bar A_j}B_j + \bigg(\sum_{j=0}^k 
        e^{2(\bar A_k -\bar A_j)}C_j^2 \bigg)^{1/2}.
    \end{align}
\end{lemma}
For the subsequent steps of the proof, we leverage the properties of discretization. We begin with the portion
employing constant step-sizes. This discretization is
applied in the time interval where the inequality from
\eqref{eq:contraction} yields a near-contraction. This
is equivalent to considering the values of $k$ for
which $m_k$ in \eqref{eq:recursion1} is positive and 
bounded away from zero.

\begin{lemma}\label{lem:stepsize1a}
    If $T$ and $a\geqslant 1$ are real numbers such that
    $T\geqslant \tfrac12 \log (6a)$. Let $K_0\in \mathbb N$
    be such that for every $k\in\{0,1,\ldots,K_0\}$,
    \begin{align}
        0 \leqslant t_k\leqslant T -   \tfrac12\log(6a),
        \qquad h_k\leqslant 0.7,\qquad \varphi(\beta_k/
        \alpha_k) \leqslant a.\label{cond1:tk}
    \end{align}
    Then, for $\alpha_k = e^{-(T-t_k)}$,
    we have $\alpha_k^2\leqslant 1/(6a)$ as well as
    \begin{align}
        m_k \geqslant 1 + \frac{2\alpha_k^2}{1-\alpha_k^2}
        \bigg(1 - \frac{a}{1-\alpha_k^2}\bigg) \geqslant 1/3,
        \qquad 
        \text{and}\qquad 
        h_k \bigg(\frac{1+\alpha_k^2}{1-\alpha_k^2} + m_k\bigg)
         \leqslant 2,
    \end{align}
    for all $k=0,\ldots,K_0$.
\end{lemma}

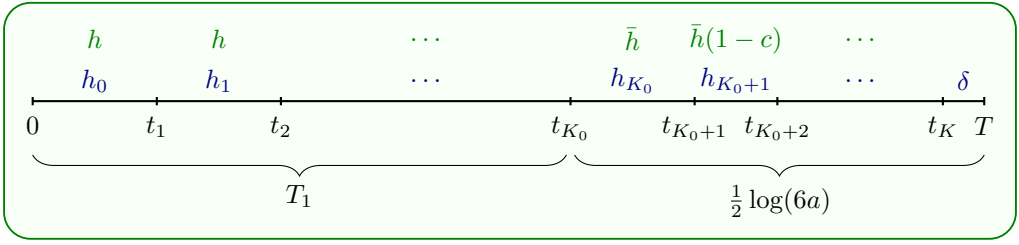
\begin{figure}[h]
    \centering
    \begin{tikzpicture}[scale=0.55,boxstyle/.style={
        draw=green!50!black, 
        fill=green!3, 
        rounded corners=10pt,
        line width=0.7pt,
        inner sep=5pt
    }] 

    \begin{scope}[local bounding box=myfigure]
        \draw[thick] (0,0) -- (23,0);
        
        \foreach \x in {0,3,6,13,16,18,22,23} {
            \draw[thick] (\x,-0.1) -- (\x,0.1);
        }
        
        \node[below] at (0,-0.15) {$0$};
        \node[below] at (3,-0.15) {$t_1$};
        \node[below] at (6,-0.15) {$t_2$};
        \node[below] at (13,-0.15) {$t_{K_0}$};
        \node[below] at (16,-0.15) {$t_{K_0+1}$};
        \node[below] at (18,-0.15) {$t_{K_0+2}$};
        \node[below] at (22,-0.15) {$t_{K}$};
        \node[below] at (23,-0.15) {$T$};
        
        \draw[decorate, decoration={brace, amplitude=8pt, mirror}] 
            (0,-1.3) -- node[below=8pt] {$T_1$} (12.9,-1.3);
        \draw[decorate, decoration={brace, amplitude=8pt, mirror}] 
            (13.1,-1.3) -- node[below=8pt] {$\tfrac{1}{2}\log(6a)$} (23,-1.3);
        
        \definecolor{darkblue}{rgb}{0,0,0.5}
        \definecolor{darkgreen}{rgb}{0,0.5,0}
        
        \node[color=darkblue] at (1.5,0.5) {$h_0$};
        \node[color=darkblue] at (4.5,0.5) {$h_1$};
        \node[color=darkblue] at (9.5,0.5) {$\ldots$};
        \node[color=darkblue] at (14.5,0.5) {$h_{K_0}$};
        \node[color=darkblue] at (17,0.5) {$h_{K_0+1}$};
        \node[color=darkblue] at (20,0.5) {$\ldots$};
        \node[color=darkblue] at (22.5,0.5) {$\delta$};
        
        \node[color=darkgreen] at (1.5,1.5) {$h$};
        \node[color=darkgreen] at (4.5,1.5) {$h$};
        \node[color=darkgreen] at (9.5,1.5) {$\ldots$};
        \node[color=darkgreen] at (14.5,1.5) {$\bar h$};
        \node[color=darkgreen] at (17,1.5) {$\bar h(1-c)$};
        \node[color=darkgreen] at (20,1.5) {$\ldots$};
    \end{scope}

    \begin{scope}[on background layer]
        \node[boxstyle, fit=(myfigure)] {};
    \end{scope}
        
\end{tikzpicture}
    \caption{Notations corresponding to the discretization schedule.}
    \label{fig:stepsizes}
\end{figure}
We set $h_k = h$ for $k=0,\ldots,K_0$. Then, 
\eqref{eq:recursion1}, \Cref{lem:recursion} and $1-h_km_k 
\leqslant 1- h/3\leqslant e^{-h/3}$ imply that
\begin{align}
    x_{K_0} &\stackrel{\numcircled{1}}{\leqslant} e^{-K_0h/3} x_0 
    + \sum_{k=0}^{K_0-1} (1-\tfrac{h}{3})^{K_0-k-1} (2h
    \varepsilon_k^b + \mathfrak B_k)\\ 
    &\qquad + \bigg\{\sum_{k=0}^{K_0-1} (1-\tfrac{h}{3})^{
    2 (K_0-k-1)}(2h\varepsilon_k^v + \mathfrak V_k)^2
    \bigg\}^{1/2}\\
    &\stackrel{\numcircled{2}}{\leqslant}
    e^{-K_0h/3} x_0 + \max_{1\leqslant k< K_0} \Big[ 
    \tfrac{3}{h} (2h\varepsilon_k^b + \mathfrak B_k )\Big]+ \max_{1\leqslant k< K_0} \Big[\sqrt{\tfrac{1.7}{h}} (2h \varepsilon_k^v + \mathfrak V_k) \Big]
    \\
    &\leqslant e^{-K_0h/3} x_0 + 
    \max_{1\leqslant k< K_0} \Big[6\varepsilon_k^b + 3h^{-1}
    \mathfrak B_k\Big] + \max_{1\leqslant k< K_0} 
    \Big[1.35h^{-1/2} (2h\varepsilon_k^v + \mathfrak V_k) 
    \Big], \label{eq:xK2}
\end{align}
where $\numcircled{1}$ comes from applying \Cref{lem:recursion} with $e^{A_k} = (1-m_k h) \leq e^{-h/3}$, form which we get $e^{\bar{A}_{j}} = e^{A_0}\cdot e^{A_1}...e^{A_j} = \prod\limits_{l=0}^j (1-m_l h_l) \leq \left( 1- \frac{h}{3} \right)^{j+1}$ and $e^{\bar{A}_{K_0-1}} = e^{-K_0 h/3}$, 
and $\numcircled{2}$ uses the fact that $\sum_{k=0}^{K_0-1} (1-\tfrac{h}{3})^{K_0-k-1} = \frac{3}{h}\left[1 - \left(1 - \frac{h}{3} \right)^{K_0} \right] \leq \frac{3}{h}$ and, similarly, $\sum_{k=0}^{K_0-1} (1-\tfrac{h}{3})^{2 (K_0-k-1)} = \frac{9}{h(6-h)} \left[1 - \left(1 - \frac{h}{3} \right)^{2K_0} \right] \leq \frac{9}{h(6-h)} \leq \frac{9}{5.3h} \leq \frac{1.7}{h}$ since $h \leq 0.7$.

The next lemma provides an upper bound for the 
bias and the variance of the discretization error.

\begin{lemma}\label{lem:discr}
    Assume that for some $a>0$ and $k\in \{0,\ldots,K\}$, $P^*$ 
    satisfies \Cref{ass:1}
    with $\varphi$ satisfying $\varphi(\sigma)\leqslant a$ for 
    every $\sigma\in [\beta_{k+1}/\alpha_{k+1}; \beta_k/\alpha_k]$. 
    Assume, in addition, that $\bar m_2 = (\mathbf E[\|\bX\|^2]/ D )
    \vee 1 < \infty$. Then, it holds that
    \begin{align}
        \mathfrak B_k &\leqslant \tfrac12\sqrt{\bar m_2
        D}\, h_k^2, \label{bound:Bk}\\
        \mathfrak V_k &\leqslant \tfrac12 \sqrt{\bar m_2D}
        \, h_k^2 + \tfrac{4\sqrt{2D}}{3} \,h_k^{3/2} 
        \frac{(a\alpha_{k+1}^2) \vee\beta_{k+1}^2
        }{\beta_{k+1}^4}. \label{bound:Vk}
    \end{align}
    If instead of $\varphi(\sigma)\leqslant a$, we have 
    $\varphi(\sigma)\leqslant \bar a\sigma^2$ for some 
    $\bar a\geqslant 1$, then \eqref{bound:Vk} can be 
    strengthened as follows 
    \begin{align}
        \mathfrak V_k &\leqslant \tfrac12 \sqrt{\bar m_2D}
        \, h_k^2 + \tfrac{4\sqrt{2D}}{3} \,\bar a\, 
        \frac{h_k^{3/2}}{\beta_{k+1}^2}. \label{bound:Vk'}
    \end{align}
    Finally, under the same condition, the error $\mathfrak 
    V_K$ of the last iterate can be bounded by
    \begin{align}
        \mathfrak V_K &\leqslant \tfrac12 \sqrt{\bar m_2D}
        \, h_K^2 + \tfrac92\,\bar a\sqrt{D\,h_K}. \label{bound:VK}
    \end{align}
\end{lemma}

\subsection[Proof of Theorem~\ref{th:1}: Strongly log-concave convolved 
with a compactly supported distribution]{Proof of \Cref{th:1}: Strongly log-concave convolved 
with a compactly supported distribution}\label{strong_conc_bound}

We know that 
\begin{align}
    \varphi(\sigma) = \frac{\sigma^2}{1 + m\sigma^2} +
    \frac{bM^2\sigma^4}{(1+M\sigma^2)^2}\leqslant 
    \Big[\frac1m + b\Big]\wedge
    \Big[\sigma^2\Big(1 + \frac{bM}4\Big)\Big].
\end{align}
Therefore, 
we can apply \Cref{lem:contraction}, \Cref{lem:stepsize1a} 
with $a=1\vee [(1/m) + b]$ as well as inequalities
\eqref{bound:Bk} and \eqref{bound:Vk'} of \Cref{lem:discr} 
with $\bar a= 1 + \tfrac14 bM$. In addition, to bound the last term in
\eqref{bound:Vk'}, we use the fact that
\begin{align}
    \frac{1}{\beta_{k+1}^2} &= \frac1{1-e^{2(t_{k+1} -T)}}
    \leqslant \frac{1}{1-e^{-\log(6a)}} = \frac{6a}{
    6a - 1} \leqslant 1.2.
\end{align}
Together with \eqref{eq:xK2}, this  leads to 
\begin{align}
    x_{K_0} 
    &\leqslant e^{-K_0h/3} x_0 + 
    6\varepsilon^b + 3h^{1/2}\varepsilon^v + 
    \sqrt{\bar m_2 D}\,\big(\tfrac32  + 1.35\times (\tfrac12 
    + \tfrac{4\sqrt{2}\,\bar a}{3}\times 1.2)\big)h\\
    &\leqslant e^{-K_0h/3} x_0 + 
    6\varepsilon^b + 3h^{1/2}\varepsilon^v + 
    (5.3 + 0.6bM)h\sqrt{\bar m_2 D}.\label{eq:xK0b}
\end{align}
On the time interval $[T-\tfrac{\log(6a)}2;T]$, we use 
the discretization obtained by geometrically decreasing
stepsizes as previously proposed in the literature:
\begin{align}
    h_{K_0+j} = \tfrac{\log(6a)}2\, c \,(1-c)^{j},\quad 
    j=0,\ldots,K-K_0-1,  
\end{align}
where $c\leqslant 0.6/\log(6a)$. This implies, in particular,
that $c\leqslant 0.6/\log 6\leqslant 0.4$ and that $\bar h:=
\max_{k\in[K_0,K]} h_k \leqslant 0.3$. 
The constants $c$ and $K$ are chosen in such a way that
$t_K = T- h_K$ for some small $h_K\leqslant \tfrac{\log(6a)}2$, 
and $t_{K+1} = T$. This means that
\begin{align}
    T- h_K &= T-\tfrac{\log(6a)}2 + \tfrac{\log(6a)}2c
    \sum_{j=0}^{K-K_0-1} (1-c)^{j} \\
    &= T-\tfrac{\log(6a)}2 + \tfrac{\log(6a)}2(1 - (1-c)^{K-K_0}).
\end{align}
This yields
\begin{align}
    (1-c)^{K-K_0} = \frac{2 h_K}{\log(6a)}\qquad 
    \text{and}\qquad K-K_0 &= \frac{\log\log(6a) - \log(2h_K)}{-\log(1-c)}\\ 
    &\leqslant \frac{ \log\log(6a) - \log(2h_K)}{c}.
\end{align}

For $k\geqslant K_0+1$, we will apply \Cref{lem:contraction}. 
To check that its conditions are fulfilled, note that 
\begin{align}
    \frac{1+\alpha_k^2}{1-\alpha_k^2} + m_k
    \leqslant \frac{1+\alpha_k^2}{1-\alpha_k^2} + 1 +
    \frac{2\alpha_k^2}{1-\alpha_k^2}
    \leqslant \frac{4}{1-e^{2(t_k-T)}}\leqslant 
    \frac{4}{1-e^{-1}}\Big(\frac{1}{2(T-t_k)}\vee 1\Big).
\end{align}
This expression, multiplied by $h_k$, is less than 2 
whenever $h_k\leqslant 0.3$. Indeed, on the one hand, 
\begin{align}
    \frac{4h_k}{1-e^{-1}} \leqslant \frac{1.2}{1-e^{-1}}
    \leqslant 2.
\end{align}
On the other hand, for $k>K_0$,
\begin{align}
    t_k &= T- \tfrac{\log(6a)}2 + h_{K_0}+\ldots + h_{k-1} = T- \tfrac{\log(6a)}2 + \tfrac{\log(6a)}2\,c\sum_{j=0}^{ k-K_0-1}  
    (1-c)^{j}\\
    & = T- \tfrac{\log(6a)}2 + \tfrac{\log(6a)}2\,(1-(1-c)^{k-K_0})
     = T - c^{-1}h_k.\label{eq:tkhk}
\end{align}
This implies that 
\begin{align}
    \frac{4h_k}{2(1-e^{-1})(T-t_k)} &= \frac{2c}{1-e^{-1}}<2
\end{align}
since $c\leqslant 0.6$. In addition, taking $\sigma=\beta_k/\alpha_k$ and using the substitution $\beta_k^2 = 1-\alpha_k^2$, we have 
\begin{align}
    m_k &\stackrel{\numcircled{1}}{=} 1 + \frac{2\alpha_k^2}{1-\alpha_k^2} \left(1 - \frac{\varphi(\beta_k/\alpha_k)}{1-\alpha_k^2} \right) \\
    &\stackrel{\numcircled{2}}{=} 1 + \frac{2\alpha_k^2}{1-\alpha_k^2} \left(1 - \frac{1}{\beta_k^2} \left[ \frac{\sigma^2}{1+m\sigma^2} + \frac{bM^2\sigma^4}{(1+M\sigma^2)^2} \right] \right) \\
    &\stackrel{\numcircled{3}}{=} 1 + \frac{2\alpha_k^2}{\beta_k^2} \left( 1- \frac{1}{\alpha_k^2 + m\beta_k^2} - \frac{\beta_k^2 \cdot bM^2}{(\alpha_k^2 + M\beta_k^2)^2}\right) \\
    &= 1 + \frac{2\alpha_k^2}{\beta_k^2} - \frac{2}{\beta_k^2(1+m\sigma^2)} - \frac{2bM^2 \alpha_k^2}{(\alpha_k^2 + M(1-\alpha_k^2))^2},
\end{align}
where $\numcircled{1}$ comes from the definition of $m_k$ from \eqref{cond:h}, $\numcircled{1}$ is true for any $\varphi(\sigma)$ satisfying \eqref{eq:phi_slg}. Equality $\numcircled{3}$ comes from the fact that 
\begin{align}
    \frac{1}{\beta_k^2}\cdot\frac{\sigma^2}{1+m\sigma^2} &= \frac{1}{\cancel{\beta_k^2}}\cdot\frac{\cancel{\beta_k^2}/\alpha_k^2}{1+m\sigma^2} = \frac{1}{\alpha_k^2(1+m\sigma^2)} = \frac{1}{\alpha_k^2+m\beta_k^2},
\end{align}
and 
\begin{align}
    \frac{1}{\beta_k^2} \cdot \frac{bM^2\sigma^4}{(1+M\sigma^2)^2} = \frac{1}{\beta_k^2}\cdot \frac{bM^2 \cdot \beta_k^4/\alpha_k^4}{(1+M\sigma^2)^2} = \frac{\beta_k^2bM^2}{\alpha_k^4(1+M\sigma^2)^2}=\frac{\beta_k^2bM^2}{(\alpha_k^2+M\beta_k^2)^2}. 
\end{align}
Finally, noting that  
\begin{align}
    1 + \frac{2\alpha_k^2}{\beta_k^2} - \frac{2}{\beta_k^2(1+m\sigma^2)} \geq 1 + \frac{2\alpha_k^2}{\beta_k^2} - \frac{2}{\beta_k^2} = -1,
\end{align}
for any $m, \sigma^2\geq 0$, we arrive at
\begin{align}\label{eq:mk2}
  m_k&\geqslant -1  - \frac{2bM^2\alpha_k^2}{(\alpha_k^2+M(1-\alpha_k^2))^2}. 
\end{align}
Therefore, 
\eqref{eq:recursion1} yields
\begin{align}
    x_{k+1}^2 &\leqslant \big(e^{-m_kh_k}x_k + 2h_k\varepsilon_k^b
    + \mathfrak B_k\big)^2 + \big(2h_k\varepsilon_k^v + 
    \mathfrak V_k\big)^2. 
\end{align}
From this recursion and \Cref{lem:recursion}, using the 
notation $H(k) = -m_{K_0}h_{K_0}-\ldots -m_{k}h_{k}$, we infer that
\begin{align}
    x_{K+1} &\leqslant e^{H({K})}\bigg[x_{K_0} + 
    \sum_{k=K_0}^{K} e^{-H(k)}(2h_k \varepsilon_k^b + 
    \mathfrak B_k) + \bigg\{\sum_{k=K_0}^{K} e^{-2H(k)}
    (2h_k\varepsilon_k^v + \mathfrak V_k)^2\bigg\}^{1/2} 
    \bigg].
\end{align}
Inequality \eqref{eq:mk2} yields
\begin{align}
    H(K) - H(k) &\leqslant \sum_{j=k+1}^K h_j + 2bM\sum_{j=k+1}^K
    \frac{Mh_j\alpha_j^{-2}}{(1 + M(\alpha_j^{-2}-1))^2}\\
    &\leqslant \frac12\log(6a) -\sum_{j=K_0}^k h_j+ 2bM\sum_{j=K_0}^K
    \frac{Mh_je^{2(T-t_j)}}{(1 + M(e^{2(T-t_j)}-1))^2}.
\end{align}
Let us set $y_j = M(e^{2(T-t_j)}-1)$. On the one hand, we have
\begin{align}
    H(K) - H(k) &\leqslant \frac12\log(6a) - \sum_{j=K_0}^k h_j + 2bM\sum_{j=K_0}^K
    \frac{h_j(y_j+M)}{(1 + y_j)^2}.
\end{align}
On the other hand, since $h_j\leqslant 0.3$, we have $e^{-2h_j}-1
\leqslant -1.5\,h_j$. Therefore,
\begin{align}
    y_{j} - y_{j+1} = (y_j+M)(1 - e^{-2h_j})\geqslant 1.5 h_j (y_j+M).
\end{align}
This implies that
\begin{align}
    H(K) - H(k) &\leqslant \frac12\log(6a) -\sum_{j=K_0}^k h_j + bM\sum_{j=K_0}^K
    \frac{4(y_j - y_{j+1})}{3(1 + y_j)^2}\\
    &\leqslant \frac12\log(6a) -\sum_{j=K_0}^k h_j + bM\int_{0}^{\infty} 
    \frac{4}{3(1 + t)^2}\,\rmd t\\
    & \leqslant \frac12\log(6a) -\sum_{j=K_0}^k h_j + \frac{4bM}{3}.
\end{align}
Using the standard inequalities
\begin{align}\label{eq:int_ineq}
    \sum_{k=K_0}^{K} e^{-u(h_{K_0}+\ldots+
    h_k)}h_k \leqslant \int_0^{\infty} e^{-ux}\,\rmd x 
    = 1/u,\qquad \forall u>0,
\end{align}
we arrive at
\begin{align}
    x_{K+1} &\leqslant \sqrt{6a}\,e^{\frac43bM}\Big(x_{K_0} +
     2\varepsilon^b + \max_{K_0<k<K} h_k^{-1}\mathfrak B_k +
     \max_{K_0<k<K} \big[ \sqrt{h_k}\,\varepsilon_k^v
     +  \tfrac12 h_k^{-1/2}\mathfrak V_k\big]\Big) \\
     &\qquad + \mathfrak B_K + 2h_K\varepsilon^v + 
     \mathfrak V_K.
\end{align}
We apply then inequalities \eqref{bound:Bk}, \eqref{bound:Vk'}
and \eqref{bound:VK} of \Cref{lem:discr} with $\bar a = 1 + 
\tfrac14 bM$. This leads to
\begin{align}
    x_{K+1} &\leqslant \sqrt{6a}e^{\frac43bM}\Big(x_{K_0} + 
        2\varepsilon^b  + \sqrt{\bar h}\,\varepsilon^v 
        + \sqrt{D}\Big[\sqrt{\bar m_2}\,\bar h +
        \max_{k<K} \frac{\bar a h_k}{\beta_{k+1}^2}\Big]\Big) + 
        5\bar a\sqrt{\bar m_2Dh_K}.\label{eq:xK+1}
\end{align}
The stepsizes $h_k$ of the geometric grid are much smaller
than the noise levels $\beta_{k+1}^2$, as attested by
the following inequality\footnote{We use the standard 
inequality $1-e^{-x}\geqslant (1-e^{-1})(x\wedge 1)$ 
for every $x>0$.}
\begin{align}
    \frac{h_k}{\beta_{k+1}^2} 
        & = \frac{h_k}{1 - e^{2(t_{k+1} - T)}}
        \leqslant \frac{h_k}{1.2 (T-t_{k+1})\wedge 0.5}
        \leqslant \frac{5h_k}{6(T-t_{k+1})} \vee 
        \frac{5h_k}{3}.
\end{align}
It follows from \eqref{eq:tkhk} that $T-t_{k+1} = c^{-1} 
h_{k+1} = c^{-1}(1-c)h_k\geqslant \tfrac23 c^{-1}h_k$. 
Hence,
\begin{align}
    \frac{h_k}{\beta_{k+1}^2} 
        \leqslant \frac{5c}{4} \vee 
        \frac{5c\log(6a)}{6} = \frac{5c\log(6a)}{6}
        =\frac{\bar h}{3}.\label{eq:hkbk1}
\end{align}
Combining \eqref{eq:xK+1} and \eqref{eq:hkbk1}, we arrive 
at
\begin{align}
    x_{K+1} &\leqslant \sqrt{6a}\,e^{\frac43 bM}\Big(x_{K_0} 
    + 2\varepsilon^b 
    + \bar h^{1/2}\,\varepsilon^v + \tfrac43\,\bar a\sqrt{\bar m_2D}\,
    \bar h\Big) + 5\bar a \sqrt{\bar m_2Dh_K}.
\end{align}
This inequality, in conjunction with \eqref{eq:xK0b}, leads
to
\begin{align}
    x_{K+1} &\leqslant \sqrt{6a}\,e^{\frac43 bM}\Big(x_{0} + 8\varepsilon^b 
    + 4h_{\max}^{1/2}\,\varepsilon^v + 6.7\bar a\sqrt{\bar m_2D}\,
    h_{\max}\Big) + 5\bar a\sqrt{\bar m_2Dh_K},
\end{align}
where $h_{\max} = \max(h,\bar h)$ is the maximal step size 
of the entire discretization grid, comprising the parts
defined through arithmetic and geometric progressions. 
These step sizes should satisfy the inequalities
\begin{align}
    h &\leqslant \frac{T-\tfrac12\log(6a)}{K_0}\leqslant 0.7\quad
    \bar h  = \frac{c\log(6a)}2  \leqslant 
    \frac{\log(6a)\big(\log\log(6a) - \log(2h_K)\big)
    }{K-K_0}\leqslant 0.3.
\end{align}
To bound $x_0$, we note that
\begin{align}
    x_0^2 &\leqslant \mathbf E[\|\alpha_{T}\bX + \beta_T\bxi  
    - \bxi\|^2]
    = \alpha_T^2\|\bX\|_{\mathbb L_2}^2 + (1-\beta_T)^2D\leqslant 
    1.01 \bar m_2^2 D e^{-2T}.
\end{align}
as soon as $T\geqslant \log (6)$. Thus, $x_0 \leqslant 1.01
\sqrt{\bar m_2 D} e^{-T}$. We set $T = \tfrac12\log(6a)  
+ T_1$ and $h_K = \delta = 0.5 e^{-2T_1}$ and $K = 2K_0$. 
This leads to the claim of the theorem. Indeed, $h\leqslant 0.7$
translates into $K_0 \geq (10/7) T_1$ and $\bar h \leq 0.3$ 
translates into 
\begin{align}
    K_0 \geq \frac{10\log(6a)\big(\log\log(6a) + 2T_1\big)}{3}
\end{align}
which is satisfied when $K_0\geqslant 7T_1\log(6a) + 4\log(6a)
\log\log(6a)$. Finally, notice that $h\leqslant T_1/K_0$ 
and 
\begin{align}
    \bar h\leqslant \frac{\log(6a)\big(\log\log(6a) + 2T_1\big)}{K_0}.
\end{align}
These inequalities yield the claimed upper bound on $h_{\max}$.

\subsection[Proof of Theorem~\ref{th:2}: Semi log-concave and 
compactly supported distribution on a subspace]{Proof of \Cref{th:2}: Semi log-concave and 
compactly supported distribution on a subspace}

For $P^*$ satisfying \Cref{ass:1} with the function 
\begin{align}\label{eq:str_conc_b}
    \varphi
    (\sigma) = b \wedge \frac{\sigma^2}{(1 - 
    M\sigma^2)_+}.
\end{align}
we can apply \Cref{lem:stepsize1a} with $a=b\vee1$ and \Cref{lem:discr} with $\bar{a}= bM+1$. Similarly to \Cref{strong_conc_bound}, 
the application of \Cref{lem:contraction} and \Cref{lem:stepsize1a} yields
\begin{align}
    x_{K_0} 
    &\leqslant e^{-K_0h/3} x_0 + 
    6\varepsilon^b + 3h^{1/2}\varepsilon^v + 
    \sqrt{\bar m_2 D}\,\big(\tfrac32  + 1.35\times (\tfrac12 
    + \tfrac{4\sqrt{2}\,\bar a}{3}\times 1.2)\big)h\\
    &\leqslant e^{-K_0h/3} x_0 + 
    6\varepsilon^b + 3h^{1/2}\varepsilon^v + 
    (2.2 + 3.1\bar{a})h\sqrt{\bar m_2 D}.\label{eq:xK0b2}
\end{align}
We again use the discretization with geometrically decreasing stepsize 
on the interval $[T-\tfrac{\log(6a)}2;T]$:
\begin{align}
    h_{K_0+j} = \tfrac{\log(6a)}2\, c \,(1-c)^{j},\quad 
    j=0,\ldots,K-K_0-1,  
\end{align}
where $c\leqslant 0.6/\log(6a)$. Following the discussion in \Cref{strong_conc_bound}, we have that
\begin{align}
    K-K_0 &\leqslant \frac{ \log\log(6a) - \log(2h_K)}{c},
\end{align}
and, for $k>K_0$ 
\begin{align}
    h_k\bigg(\frac{1+\alpha_k^2}{1-\alpha_k^2} + m_k\bigg)
         \leqslant 2 \quad \text{and}\quad t_k=T-c^{-1}h_k.  
\end{align}

Combined with \eqref{eq:str_conc_b}, we get
\begin{align}
    m_k\geq1-2\bar{a}.
\end{align}
Hence, \ref{eq:recursion1} yields
\begin{align}
    x_{k+1}^2 &\leqslant \big(e^{(2\bar{a}-1)h_k}x_k + 2h_k\varepsilon_k^b
    + \mathfrak B_k\big)^2 + \big(2h_k\varepsilon_k^v + 
    \mathfrak V_k\big)^2. 
\end{align}
We denote $H_k=(2\bar{a}-1)\sum_{i=K_0}^{k}h_k$. We note that $H_K\leqslant \frac{2\bar{a}-1}{2}\ln(6a)$. Lemma \ref{lem:recursion} states: 
\begin{align}
    x_{K+1} &\leqslant e^{H_{K}}\bigg[x_{K_0} + 
    \sum_{k=K_0}^{K} e^{-H_k}(2h_k \varepsilon_k^b + 
    \mathfrak B_k) + \bigg\{\sum_{k=K_0}^{K} e^{-2H_k}
    (2h_k\varepsilon_k^v + \mathfrak V_k)^2\bigg\}^{1/2} 
    \bigg].
\end{align}
As $2\bar{a}-1=2bM+1$ which is strictly positive, we may apply \eqref{eq:int_ineq} which results in:
\begin{align}
    x_{K+1} &\leqslant \sqrt{6a}\,e^{2\bar{a}-1}\Big(x_{K_0} + 
     \frac{2\varepsilon^b+\max_{k<K}h_k^{-1}\mathfrak B_k}{2\bar{a}-1} + \frac{1}{\sqrt{2\bar{a}-1}}\max_{k<K} \Big\{\sqrt{h_k}\,\varepsilon_k^v + 
    \frac{\mathfrak V_k}{2 h_k^{1/2}}\Big\}\Big) \\ 
    &+ \mathfrak B_K
    + 2h_K\varepsilon^v +\mathfrak V_K.
\end{align}
We apply then inequalities \eqref{bound:Bk}, \eqref{bound:Vk'}
and \eqref{bound:VK} of \Cref{lem:discr}, which leads to
\begin{align}
    x_{K+1} &\leqslant \sqrt{6a}e^{(2\bar{a}-1)}\Big(x_{K_0} + 
        \frac{2\varepsilon^b}{2\bar{a}-1}  + \frac{\sqrt{\bar h}\,\varepsilon^v}{\sqrt{2\bar{a}-1}} 
        + \frac{\sqrt{D}}{\sqrt{2\bar{a}-1}}\Big[\sqrt{\bar m_2}\,\bar h +
        \max_{k<K} \frac{\bar a h_k}{\beta_{k+1}^2}\Big]\Big) + 
        5\bar a\sqrt{\bar m_2Dh_K}.
\end{align}
The above inequality with \eqref{eq:hkbk1} yields:
\begin{align}
     x_{K+1} &\leqslant \sqrt{6a}e^{(2\bar{a}-1)}\Big(x_{K_0} + 
        \frac{2\varepsilon^b}{2\bar{a}-1}  + \sqrt{\frac{\bar h}{2\bar{a}-1}}\,\varepsilon^v 
        + \frac{4\bar{a}\bar{h}}{3}\sqrt{\frac{D\bar m_2}{2\bar{a}-1}}\Big) + 
        5\bar a\sqrt{\bar m_2Dh_K}.\label{eq:xK+12}
\end{align}

Combining \eqref{eq:xK+12} with \eqref{eq:xK0b2} and noting that $(2\bar{a}-1)\geq1$, we get:
\begin{align}
    x_{K+1} &\leqslant \sqrt{6a}\,e^{2\bar{a}-1}\Big(x_{0} + 8\varepsilon^b 
    + 4h_{\max}^{1/2}\,\varepsilon^v + 6.7\bar a\sqrt{\bar m_2D}\,
    h_{\max}\Big) + 5\bar a\sqrt{\bar m_2Dh_K},
\end{align}

Following the discussion of \Cref{strong_conc_bound}, we complete 
the proof by showing that:
 \begin{align}
        x_{K+1} \leq e^{2\bar a-1}
        \Big\{2e^{-T_1} + 
        7\sqrt{6a}\,h_{\max} + 4\sqrt{6a}\big(2\varepsilon^b_{
        \sf score} + h_{\max}^{1/2}\,\varepsilon^v_{\sf score}\big)
        \Big\}\sqrt{D}.
    \end{align}

\section{Proofs of lemmas used in the proofs of main theorems}

We collect in this section the proofs of the building
blocks of our main results. 

\subsection[Proof of Lemma~\ref{lem:contraction}: the origin of the 
contraction/expansion]{Proof of \Cref{lem:contraction}: the origin of the 
contraction/expansion}

Since $\score$ is continuously differentiable, we have
\begin{align}
    \bU_k = \int_0^1 \mathrm D \score \big(T-t_k,\bZ_k + 
    \theta(\bY_{t_k} - \bZ_k)\big)\,(\bY_{t_k} - \bZ_k )\, 
    \rmd\theta := \int_0^1\bfM_k(\theta)\bDelta_k\, 
    \rmd\theta,
\end{align}
where 
\begin{align}
    \bfM_k(\theta) = \mathrm D\score (T-t_k, \bZ_k + \theta
    \bDelta_k) = \nabla^2 \log \pi(T-t_k, \bZ_k + \theta 
    \bDelta_k).
\end{align}
The matrix $\bfM_k$ is symmetric, and according to 
\Cref{prop:1},  all its eigenvalues satisfy
\begin{align}
    -\frac1{1-\alpha_k^2} \leqslant \lambda_j(\bfM_k(\theta))
    \leqslant -\frac1{1-\alpha_k^2}\bigg( 1- \frac{\alpha_k^2
    \varphi(\beta_k/\alpha_k)}{1-\alpha_k^2}\bigg). 
\end{align} 
We assume that $h_k$ is chosen so that
\begin{align}\label{ineq:hk}
    \frac{2h_k}{1 - \alpha_k^2} - (1+{h_k}) \leqslant 
    (1+{h_k}) - \frac{2h_k}{1 - \alpha_k^2}\Big( 1 - \frac{\alpha_k^2\varphi(\beta_k/\alpha_k)}{1-
    \alpha_k^2}\Big).
\end{align}
This is equivalent to 
\begin{align}\label{ineq:15}
    \frac{h_k}{1 - \alpha_k^2}\Big( 2 - \frac{\alpha_k^2
    \varphi(\beta_k/\alpha_k)}{1-\alpha_k^2}\Big) 
    \leqslant (1+{h_k}).
\end{align}
Regrouping the terms, we get
\begin{align}\label{ineq:13a}
    \frac{h_k}{1 - \alpha_k^2}\Big( 1+\alpha_k^2 - 
    \frac{\alpha_k^2 \varphi(\beta_k/\alpha_k)}{1 - 
    \alpha_k^2}\Big) \leqslant 1.
\end{align}
This inequality can be checked to be the same as \eqref{cond:h}. 
Hence, \eqref{ineq:hk} is indeed satisfied and, therefore, 
\begin{align}
    \|(1+h_k)\,\bfI_D + 2h_k\bfM_k(\theta)\| \leqslant 
    1 + h_k - \frac{2h_k}{1-\alpha_k^2}\Big(1 - 
    \frac{\alpha_k^2 \varphi(\beta_k/\alpha_k)}{1-\alpha_k^2}
    \Big).
\end{align}
Therefore, we have 
\begin{align}
    \|(1+{h_k})\bDelta_k + 2h_k\bU_k\|_{\mathbb L_2} 
    &\leqslant \int_0^1 
    \big\| \big((1+{h_k})\,\bfI_D + 2h_k\bfM_k(\theta)\big) 
    \bDelta_k\big\|_{\mathbb L_2}\,\rmd\theta\\
    &\leqslant \Big\{1+{h_k} - \frac{2 h_k}{1-\alpha_k^2}
    \Big( 1 -\frac{\alpha_k^2 \varphi(\beta_k/\alpha_k)}{
    1-\alpha_k^2}\Big) \Big\} \|\bDelta_k\|_{\mathbb L_2}\\
    &= \Big\{1 - \frac{2h_k}{1-\alpha_k^2}\Big(\frac{1 + 
    \alpha_k^2}{2}  - \frac{\alpha_k^2 \varphi(\beta_k/
    \alpha_k)}{1-\alpha_k^2}\Big) \Big\} \|\bDelta_k
    \|_{\mathbb L_2}\\
    &= (1- m_kh_k)\|\bDelta_k\|_{\mathbb L_2}.
\end{align}
This completes the proof of \Cref{lem:contraction}. 

\subsection[Proof of Lemma~\ref{lem:stepsize1a}: strength of 
the deflation in the contracting regime]{Proof of \Cref{lem:stepsize1a}: strength of 
the deflation in the contracting regime}

Notice that $\varphi(\beta_k/\alpha_k)\leqslant a$
implies 
\begin{align}
    m_k\geqslant 1 + \frac{2\alpha_k^2}{1-\alpha_k^2}\Big(1
    - \frac{a}{1-\alpha_k^2}\Big).
\end{align}
Then, one checks that
\begin{align}
    1 + \frac{2\alpha_k^2}{1-\alpha_k^2}\Big(1
    - \frac{a}{1-\alpha_k^2}\Big)\geqslant \frac13 \quad&\Longleftrightarrow 
    \quad (1-\alpha_k)^4 \geqslant 3\alpha_k^2\big(a - 
    (1-\alpha_k^2)\big)\\
    &\Longleftrightarrow 
    \quad (1-\alpha_k)^4 - 3\alpha_k^2\big(a - 
    (1-\alpha_k^2)\big) \geqslant 0\\
    &\Longleftrightarrow 
    \quad 1-(3a-1)\alpha_k^2 -2\alpha_k^4\geqslant 0\\
    &\Longleftrightarrow 
    \quad 2\alpha_k^2\leqslant \frac{2}{3a-1 + 
    \sqrt{(3a-1)^2 + 2}}
\end{align}
Since $a\geqslant 0$, we have 
\begin{align}
    3a-1 + \sqrt{(3a-1)^2 + 2} \leqslant 6a.
\end{align}
Therefore, $\alpha_k^2\leqslant 1/(6a)$ implies $m_k\geqslant 
1/3$. Let us check that for $t_k$ satisfying \eqref{cond1:tk}, 
we have $\alpha_k^2\leqslant 1/(6a)$. Indeed, $\alpha_k$
being an increasing function of $t_k$, we have 
\begin{align}
    \alpha_k^2 & = e^{2(t_k-T)}\leqslant 
    \exp(-\log(6a)) = 1/(6a). 
\end{align}
For the second inequality, it suffices to notice that 
$\varphi(\sigma)\geqslant 0$ and $a\geqslant 1$ imply 
that
\begin{align}
    \frac{1+\alpha_k^2}{1-\alpha_k^2} + m_k 
    & \leqslant \frac{1+\alpha_k^2}{1-\alpha_k^2} + 1
    +\frac{2\alpha_k^2}{1-\alpha_k^2}= 1 + \frac{1 + 
    3\alpha_k^2}{1- \alpha_k^2} \leqslant 1 + 
    \frac{6a+3}{6a - 1} \leqslant 2.8.
\end{align}
Combining with the condition $h_k\leqslant 0.7$, 
this yields $h_k(\frac{1+\alpha_k^2}{1-\alpha_k^2} 
+ m_k)\leqslant 2$ and completes the proof of the lemma.


\subsection[Proof of Lemma~\ref{lem:discr}: assessing the increments of the drift]{Proof of \Cref{lem:discr}: assessing the increments of the drift}

Let $\drift_t = \bY_t +2\score (T-t,\bY_t)$. To 
prove the first inequality, we recall that 
$\score(T-t,\by) = (\alpha_{T-t}\mathbf E[\,\bX_0 
\cond \bY_{t} = \by] -\by)/\beta_{T-t}^2$. 
Therefore, 
\begin{align}
    \drift_t = \frac{2\alpha}{\beta^2}
    \mathbf E[\,\bX_0\cond \bY_{t}]  + \bY_t\Big(
    1- \frac{2}{\beta^2}\Big).
\end{align}
In addition, $\bY_t = \alpha \bX_0 + \beta \bxi$
with $\bxi\indep\bX_0$ and $\bxi\sim\mathcal N_D
(0,\mathbf I_D)$. It holds that
\begin{align}
    \mathbf{E}[\|\bY_t\|^2] &= \alpha^2 \mathbf{E}
    [\|\bX_0\|^2] + \beta^2 \mathbf{E}[\|\bxi\|^2] = 
    \alpha^2 \bar m_2D + \beta^2 D,
    \label{norm_Y}\\     
    \mathbf{E}[\bX_0^\top \bY_t] &= \alpha \mathbf{E}[
    \|\bX_0\|^2] + \beta \mathbf{E}[\bX_0^\top \bxi] 
    = \alpha \bar m_2 D,
\end{align}
since $\bxi$ is independent of $\bX_0$ and has zero mean.

Let us use the ``local notation'' $\bar\score(t,
\by) = \score(t,\by) + \by$ as well as $\bfH(t,\by) = 
\mathrm{D} \score(t,\by)$. According to \cite[Prop. 2]{
ConfortiDS25}, it holds that 
\begin{align}
    \rmd \bar\score(T-t,\bY_t) = 
    \bar\score(T-t,\bY_t)\,\rmd t + \sqrt{2}\,
    \mathrm{D}\bar\score(T-t,\bY_t)\,\rmd\tilde
    \bB_t.
\end{align}
Since $2\bar\score(T-t,\bY_t) = \drift_t + \bY_t$, and 
${\rm D}\bar\score(T-t,\bY_t) = \bfH(T-t,\bY_t) + \bfI_D$
we get
\begin{align}
    \rmd \drift_t &= -\rmd \bY_t + 2\rmd 
    \bar\score (T-t,\bY_t)\\
    & = -\drift_t\,\rmd t - \sqrt{2}\,\rmd\wt{\bB}_t
    +(\drift_t + \bY_t)\,\rmd t + 2\sqrt{2}\,(\bfH(T-t,\bY_t) 
    + \bfI_D)\,\rmd\wt{\bB}_t\\
    & = \bY_t\,\rmd t + \sqrt{2}\,(2\bfH(T-t,\bY_t) + 
    \bfI_D)\,\rmd\wt{\bB}_t.\label{eq:dSt}
\end{align}
Since $\tilde\bB_t-\tilde\bB_{t_k}$ is independent
of the $\sigma$-algebra $\mathcal F_k$, we get
\begin{align}
    \mathbf E[\,\drift_t-\drift_{t_k}\cond \mathcal F_k] 
    &= \int_{t_k}^{t} \mathbf E[\,\bY_u\cond\mathcal F_k]\,\rmd u
\end{align} 
and, therefore,
\begin{align}
    \|\mathbf E[\,\drift_t-\drift_{t_k}\cond \mathcal F_k]
    \|_{\mathbb L_2} &\leqslant \int_{t_k}^t \|\mathbf E[\,\bY_u 
    \cond \mathcal F_k]\|_{\mathbb L_2}\rmd u \leqslant 
    \int_{t_k}^t \|\mathbf \bY_u\|_{\mathbb L_2}\rmd u\\
    &\leqslant \int_{t_k}^t\sqrt{D(e^{-2(T-u)}\bar m_2 + 
    (1-e^{-2(T-u)}))}\,\rmd u\\
    &\leqslant \sqrt{\bar m_2D}\,(t-t_k).
\end{align}
The definition of $\bV_k$ given in \eqref{def:Vk} implies
that $\bV_k = \int_{t_k}^{t_{k+1}} (\drift_t-\drift_{t_k})
\,\rmd t$. This leads to
\begin{align}
    \|\mathbf E[\,\bV_k\cond\mathcal F_k]\| &
    \leqslant \int_{t_k}^{t_{k+1}} \|\mathbf 
    E[\,\drift_t-\drift_{t_k}\cond\mathcal F_k]\|_{\mathbb 
    L_2}\,\rmd t\\
    &\leqslant \sqrt{\bar m_2D}
    \int_{t_k}^{t_{k+1}} (t-t_k)\,\rmd t 
    = \tfrac12 \sqrt{\bar m_2D}\, h_k^2.
\end{align}
This yields the claim of \eqref{bound:Bk}.

We prove now \eqref{bound:Vk}. The definition of $\drift_t 
= \bY_t + 2\score(T-t,\bY_t)$ leads to
\begin{align}
    \big\|\drift_t -\drift_{t_k} &-\mathbf
    E[\,\drift_t -\drift_{t_k}\cond \mathcal F_k]\big\|_{\mathbb
    L_2} \\
    &=  \bigg\|\int_{t_k}^t
    (\bY_u -\mathbf E[\,\bY_u\cond \mathcal F_k])\,\rmd u 
    + \int_{t_k}^t \sqrt{2}\big(2\bfH(u) + \bfI_D\big)\,
    \rmd\tilde\bB_u\bigg\|_{\mathbb L_2}\\
    &\leqslant \int_{t_k}^t
    \big\|\bY_u -\mathbf E[\,\bY_u\cond \mathcal F_k]\big\|_{
    \mathbb L_2}\,\rmd u + \bigg\|\int_{t_k}^t \sqrt{2}
    \big(2\bfH(u) + \bfI_D\big)\,\rmd\tilde\bB_u
    \bigg\|_{\mathbb L_2}. \label{eq:lem3:a}
\end{align}    
On the one hand, in view of the law of total variance, 
we have $\big\|\bY_u - \mathbf E[\,\bY_u\cond \mathcal F_k] 
\big\|_{\mathbb L_2} \leqslant \big\|\bY_u\big\|_{\mathbb L_2}$. 
Therefore, using \eqref{norm_Y}, we get 
\begin{align}
\int_{t_k}^t
    \big\|\bY_u -\mathbf E[\,\bY_u\cond \mathcal F_k]\big\|_{
    \mathbb L_2}\,\rmd u \leqslant 
    \int_{t_k}^t \sqrt{\bar m_2 D} \,\rmd u = \sqrt{\bar m_2 D}\, (t-t_k).\label{eq:lem3:b}
\end{align}
On the other hand, the properties of the stochastic integral
imply that
\begin{align}
    \bigg\|\int_{t_k}^t \sqrt{2}
    \big(2\bfH(u) + \bfI_D\big)\,\rmd\tilde\bB_u
    \bigg\|_{\mathbb L_2}^2 = 
    2\int_{t_k}^t \mathbf E\big[\big\|2\bfH(u) + 
    \bfI_D\big\|_F^2\big]\,\rmd u.\label{eq:lem3:c}
\end{align}
Combining he definition of $\bV_k$ given in \eqref{def:Vk}
with \eqref{eq:lem3:a}, \eqref{eq:lem3:b} and \eqref{eq:lem3:c}, 
we get
\begin{align}
    \|\bV_k - \mathbf E[\,\bV_k\cond \mathcal F_k\,]\|_{ \mathbb L_2} 
    & = \int_{t_k}^{t_{k+1}} \big\|\drift_t -\drift_{t_k} -\mathbf
    E[\,\drift_t -\drift_{t_k}\cond \mathcal F_k\,]\big\|_{\mathbb
    L_2} \rmd t\\
    &\leqslant  \tfrac12 \sqrt{\bar m_2D}\, h_k^2 +       \int_{t_k}^{t_{k+1}}\!\bigg\{2\int_{t_k}^t\! \mathbf E\big[\big\|2\bfH(u) + 
    \bfI_D\big\|_F^2\big]\,\rmd u\bigg\}^{1/2}\!\!\!\!\rmd t. \,\,\label{eq:lem3:d}
\end{align}
The integral  in \eqref{eq:lem3:c} can be bounded from above 
using \Cref{prop:1} and various assumptions of the function 
$\varphi$ from \Cref{ass:1}.  Indeed, denoting $\sigma_{T-u} 
= \beta_{T-u}/\alpha_{T-u}$, we have $\bfH(u)\preccurlyeq 
\beta_{T-u}^{-2}(\varphi(\sigma_{T-u})\sigma_{T-u}^{-2} - 
1)\,\bfI_D$. Since, in addition $\bfH(u) \succcurlyeq -\beta_{T-u}^{-2}\bfI_D$, we get
\begin{align}
    0\preccurlyeq (2\bfH(u) +\bfI_D)^2 \preccurlyeq 
    4\frac{[\varphi(\sigma_{T-u})/\sigma_{T-u}^2]^2
    \vee 1}{\beta_{T-u}^4}\;\bfI_D. \label{eq:lem3:e}
\end{align}
If we assume that $\varphi(\sigma_{T-u})\leqslant a$,  
we arrive at
\begin{align}
    \bigg\{2\int_{t_k}^t \mathbf E\big[\big\|2\bfH(u) + 
    \bfI_D\big\|_F^2\big]\,\rmd u\bigg\}^{1/2} &\leqslant 
    2\sqrt{2D(t-t_k)}\; \frac{(a\alpha_{T-t}^2)\vee\beta_{T-t}^2
    }{\beta_{T-t}^4}. 
\end{align}
In view of \eqref{eq:lem3:d}, this yields
\begin{align}
    \|\bV_k - \mathbf E(\bV_k\cond \mathcal F_k)\|_{ \mathbb L_2} 
    & \leqslant \tfrac12 \sqrt{\bar m_2D}\, h_k^2 + 
    \tfrac{4\sqrt{2D}}{3} \,h_k^{3/2} \frac{(a\alpha_{T-t_{k+1}}^2)
    \vee\beta_{T-{t_{k+1}}}^2}{\beta_{T-t_{k+1}}^4}.
\end{align}
This completes the proof of the second claim of the lemma. 

If instead of the assumption $\varphi(\sigma)\leqslant a$, we
use the assumption $\varphi(\sigma)\leqslant \bar a \sigma^2$ 
with $\bar a\geqslant 1$, inequality \eqref{eq:lem3:e}, the fact 
that $u\mapsto \beta_{T-u}$ is decreasing, and inequality \eqref{eq:lem3:d} imply that
\begin{align}
    \|\bV_k - \mathbf E(\bV_k\cond \mathcal F_k)\|_{ 
    \mathbb L_2} & \leqslant \tfrac12 \sqrt{\bar m_2D}
    \, h_k^2 + \tfrac{4\sqrt{2D}}{3} \,h_k^{3/2} 
    \frac{\bar a}{\beta_{T-t_{k+1}}^2}.
\end{align}

For the last claim, we use \eqref{eq:lem3:d} and \eqref{eq:lem3:e}
as follows
\begin{align}
    \int_{t_{K}}^{T}\!\bigg\{\int_{t_K}^t\! \mathbf E\big[\big\|2\bfH(u) + \bfI_D\big\|_F^2\big]\,\rmd u\bigg\}^{1/2}
    \!\!\!\!\rmd t&\leqslant \sqrt{D} 
    \int_{t_{K}}^{T}\!\bigg\{\int_{t_K}^t\! \frac{4\bar a^2}{
    (1-e^{-2(T-u)})^2}\,\rmd u\bigg\}^{1/2}\!\!\!\!\rmd t\\
    &\leqslant \sqrt{D} 
    \int_0^{T-t_{K}}\!\bigg\{\int_{t}^{T-t_K}\! \frac{4\bar a^2}{
    (1-e^{-2u})^2}\,\rmd u\bigg\}^{1/2}\!\!\!\!\rmd t\\
    &\leqslant \sqrt{D} \int_0^{T-t_{K}}\!\bigg\{\int_t^{T-t_K}
    \! \frac{4\bar a^2}{u^2}\,\rmd u\bigg\}^{1/2}
    \!\!\!\!\rmd t\\
    &= \sqrt{D} \int_0^{T-t_{K}}\!\bigg\{ \frac{4\bar a^2
    (T-t_k-t)}{t(T-t_K)}\bigg\}^{1/2}\!\!\!\!\rmd t\\
    & =\pi\bar a\sqrt{D(T-t_K)}.
\end{align}
Thus, from \eqref{eq:lem3:d}, we infer that
\begin{align}
    \|\bV_K - \mathbf E(\bV_K\cond \mathcal F_K)\|_{ \mathbb L_2} 
    & \leqslant \tfrac12 \sqrt{\bar m_2D}\, h_K^2 + 
    \tfrac92\bar a\sqrt{Dh_K}.
\end{align}
This completes the proof.

\section{Numerical Experiments}\label{App:E}

Our experiments follow the standard DDPM sampling procedure as described in the original DDPM paper by \cite{HoJA20}, specifically the pseudocode presented in their Algorithm 2. 

\subsection{Implementation Details}
For clarity, we re-state their algorithm below.

\begin{center}
    \begin{minipage}{0.85\textwidth}
    \begin{tcolorbox}[
    colback=cyan!05,
    colframe=cyan!80!black,
    arc=4mm,
    boxrule=0.8pt,
    left=2mm, right=2mm, top=-1.7mm, bottom=-1mm
    ]
    \vspace{-10pt}
        \begin{algorithm}[H]
        \caption{DDPM Sampling \cite{HoJA20}} \label{algo:ddpm_sampling}
        \begin{algorithmic}[1]
        \STATE $\bx_T \sim \mathcal{N}(\boldsymbol{0}, \bfI)$
        \FOR{$t = T$ \TO $1$}
            \STATE $\bz \sim \mathcal{N}(\boldsymbol{0}, \bfI)$ if $t>1$, else $\bz =\boldsymbol{0}$
            \qquad\hspace*{46pt}
            \STATE $\bmu_{\btheta}(\bx_t,t) =  \frac{1}{\sqrt{\alpha_t}} \left( \bx_t - \frac{1-\alpha_t}{\sqrt{1-\bar{\alpha}_t}}\boldsymbol{\epsilon}_{\theta}(\bx_t,t) \right)$
            \STATE $\bx_{t-1} = \bmu_{\btheta}(\bx_t,t) + \sigma_t \bz$ \qquad
        \ENDFOR
        \RETURN $\bx_0$
        \end{algorithmic}
        \end{algorithm}
    \end{tcolorbox}
    \end{minipage}
\end{center}
To better explain the correspondence between notation used in our
paper and that of \cite{HoJA20}, we provide the following table:
\begin{table}[h]
    \centering
    \begin{tabular}{c| c}
    \toprule
        Notation in \cite{HoJA20} & Our notation\\
        \midrule    
         $\bx_T,\ldots,\bx_0$ & $\bZ_0,\ldots,\bZ_{K+1}$  \\[5pt]
         $\bz$ &  $\bxi_{k+1}$\\[5pt]
         $\sigma_t$ &  $\sqrt{2h_{k}}$\\[5pt]
         $\alpha_t$ & $(1+h_k)^{-2}\approx e^{-2h_k}\approx 1-2h_k$\\[5pt]
         $\bar\alpha_t$ & $\prod_{j=0}^k (1+h_k)^{-2} \approx 
         e^{-2t_{K+1}}$\\[5pt]
         $\displaystyle\frac{\boldsymbol{\epsilon}_{\theta}(\bx_t,t)}{\sqrt{1-\bar\alpha_t}}$
         & $-2\wt\score(T-t_k,\bZ_k)$\\
         \bottomrule
    \end{tabular}
    \label{tab:my_label}
\end{table}

To evaluate the robustness of the generative process under perturbed score estimates, we had to isolate the score estimation component within the sampling loop. In the formulation of \cite{HoJA20}, this corresponds to the rescaled neural network output $-0.5\boldsymbol{\epsilon}_\theta(\bx_t, t)/\sqrt{1-\bar{\alpha}_t}$. In our experiments, we added various forms of noise (Gaussian, Uniform, Laplace, and Student's-$t$) directly to this term, simulating inaccurate or noisy score predictions. This modification allows us to assess the impact of score perturbations on the quality of generated samples, both visually and quantitatively.

We know that in our formulation of the problem, the 
conditional expectation of the next state given that the
current state is $\bx$ is given by
$\bmu_\theta(\bx, t) = (1+h)\bx + 2\score(t, \bx)h$. 
Therefore, adding $\bzeta$ to $\score(t, \bx)$ implies adding $2h\bzeta$ to $\bmu_\theta(\bx_t, t)$, and thus adding $
\displaystyle\frac{\sqrt{\alpha_t(1-\bar{\alpha}_t)}}{1-\alpha_t}
\times 2h_k\bzeta\approx 2\sqrt{1-\bar\alpha_t}\,\bzeta$ to $\boldsymbol{\epsilon}_\theta(\bx, t)$.

If the paper is accepted, we will open-source the complete inference and evaluation code, including scripts to reproduce all figures and tables.

\subsection{Additional Figures}

\paragraph{Qualitative results.}
 \Cref{fig:qual_cifar} and \Cref{fig:qual_celeba} extend the main-paper image grids.
 For each dataset (CIFAR-10, CelebA-HQ, and LSUN-Churches) we display samples generated with Gaussian, Laplace, and Student's-$t$ score noise at two strengths, $\sigma=0.5$ or $\sigma=1$ (moderate) and $\sigma=2$ (severe).  
Rows share the same latent seed as the baseline to enable direct visual comparison.
\vspace*{-7pt}

\paragraph{Quantitative trends.}
\Cref{fig:fid_steps} tracks FID on the CIFAR-10 dataset as we truncate the 1 000-step DDPM schedule at \{250, 500, 750, 1000\} steps for the \emph{clean} score and the i.i.d. $\mathcal N(\boldsymbol{0},\bfI_D)$ noise contaminated score. 
\begin{figure}[h!]
    \centering
    \vspace{-10pt}
    \begin{subfigure}[t]{0.33\linewidth}
    \includegraphics[width=\linewidth]{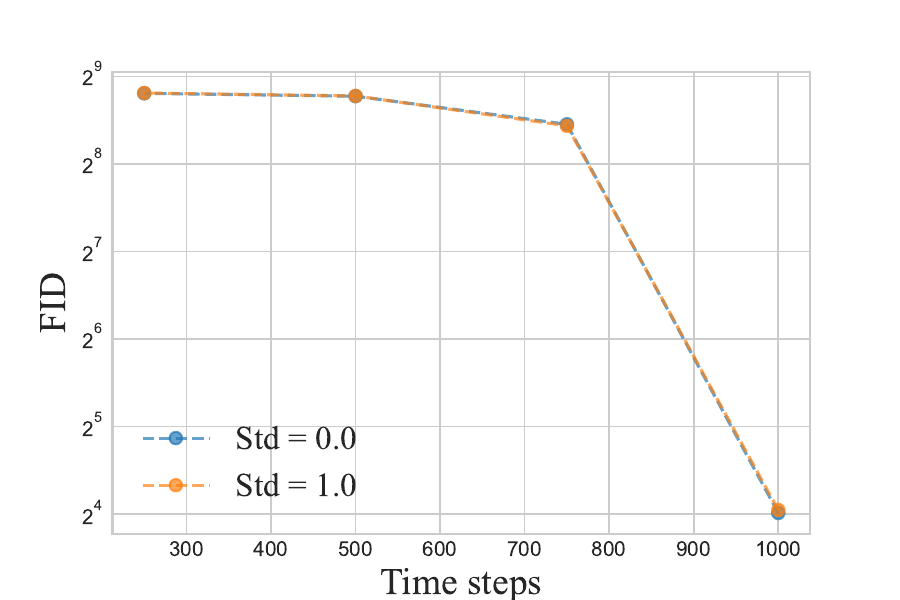}
    \caption{CIFAR-10}
    \end{subfigure}\hfill
    \begin{subfigure}[t]{0.33\linewidth}
    \includegraphics[width=\linewidth]{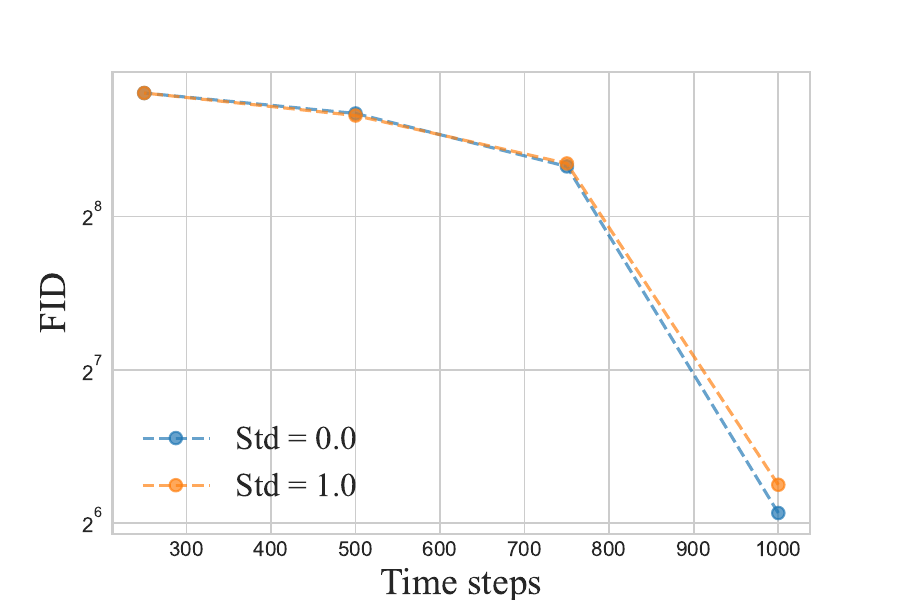}
    \caption{CelebA-HQ}
    \end{subfigure}\hfill
    \begin{subfigure}[t]{0.33\linewidth}
    \includegraphics[width=\linewidth]{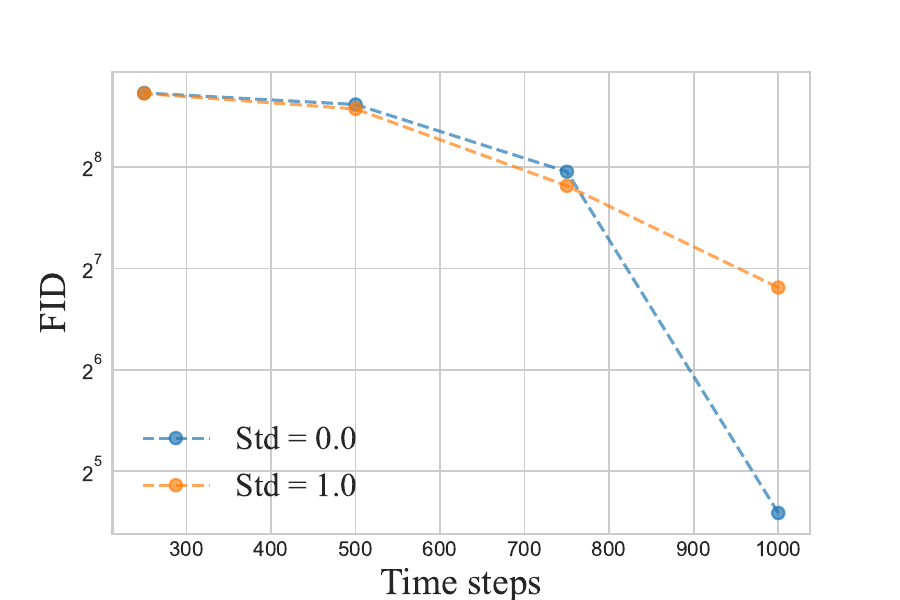}
    \caption{LSUN-Churches}
    \end{subfigure}
    
    \caption{FID as a function of time steps. Blue:  
    standard DDPM inference.  
    Orange: same sampler with i.i.d.  $\mathcal N(\boldsymbol{0},\bfI_D)$ noise added to the score at each step.}    
    \vspace*{-9pt}
    \label{fig:fid_steps}
\end{figure}

We observe that performance increases at a similar rate with the number of steps for both clean and noisy score estimates.

Additionally, \Cref{fig:fid_evolution} illustrates the ``deterioration'' of three distinct pictures for each of the different models (datasets) that we have --- each starting with a fixed random noise, generating the corresponding image after 1000 diffusion steps with the noise contaminated score, as described before, parametrized by different $\sigma$. We observe that datasets with higher-resolution images and, respectively, deeper noise (alternatively, score) predicting neural networks exhibit higher deterioration than those with low-resolution images.
\begin{figure}[h!]
    \centering
\includegraphics[width=\linewidth]{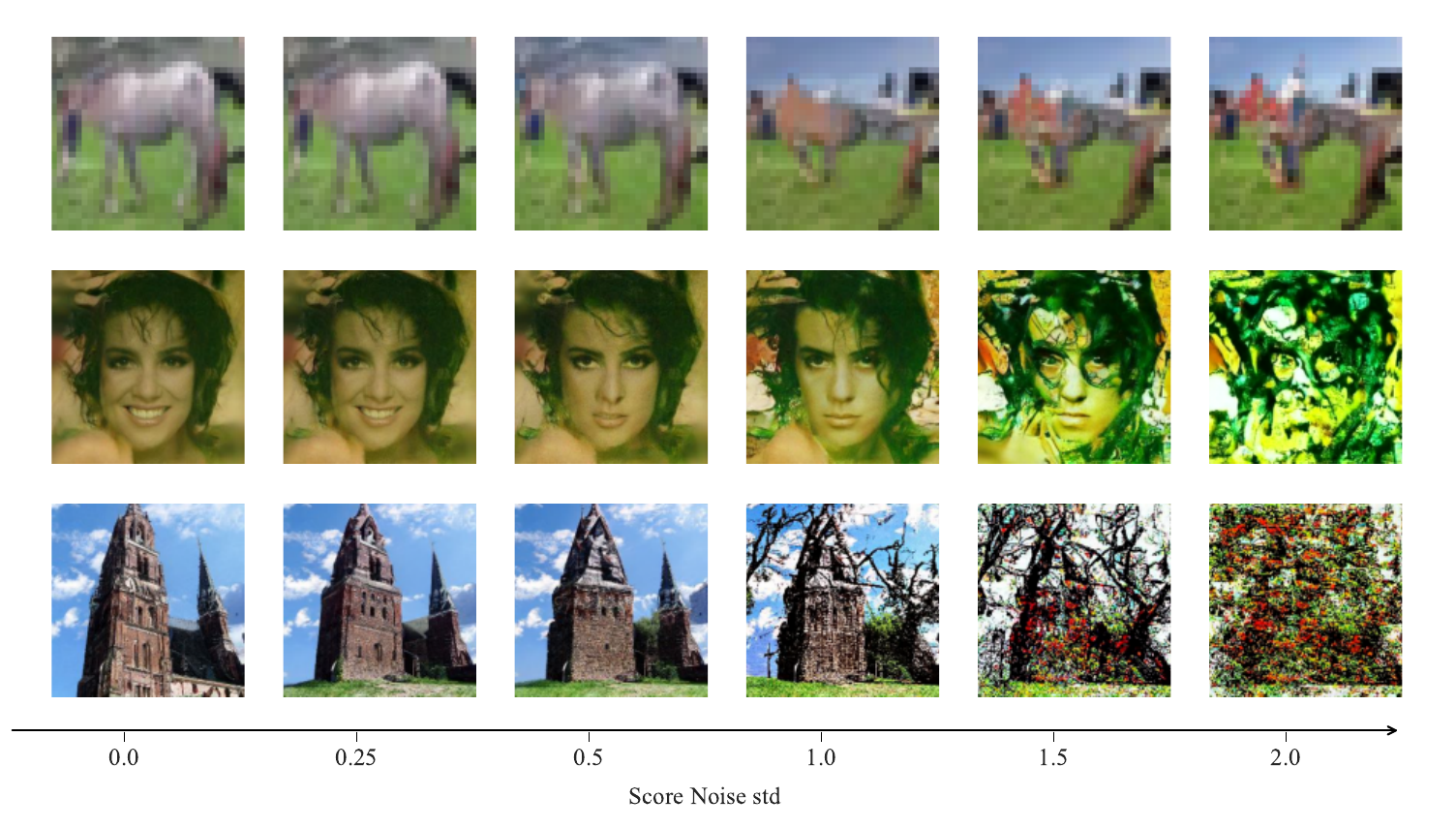}
    \caption{A single example of CIFAR-10 (top), CelebA-HQ (middle) and LSUN-Churches (bottom) generated data, respectively, over different standard deviations.}
    \label{fig:fid_evolution}
\end{figure}

\subsection{Computational Resources}
\label{compute_resources}
This project was provided with computer and storage resources 
by GENCI at IDRIS thanks to the grant 20XX-104565 on the 
supercomputer Jean Zay's A100 partition. 

Some of the experiments were run on two additional GPU nodes: 
one with AMD EPYC 7V12 64-Core Processor, 1TB of RAM, and 
with 8xA100 40GB VRAM version NVIDIA GPUs. The other one 
with AMD EPYC 9005 192-Core Processor, 0.5TB of RAM, and 
with 2xH100 NVIDIA GPUs.

Sampling 8192 CIFAR‑10 images or 512 CelebA-HQ or 512 
LSUN-Churches images takes 1.5 GPU‑hours. FID evaluation 
for all the scale values of a single noise distribution 
takes 0.2 GPU-hours. 

\begin{figure}[p]
  \centering
  \begin{subfigure}[t]{0.3\linewidth}
    \centering
    \includegraphics[width=\linewidth, height=1.5\linewidth]{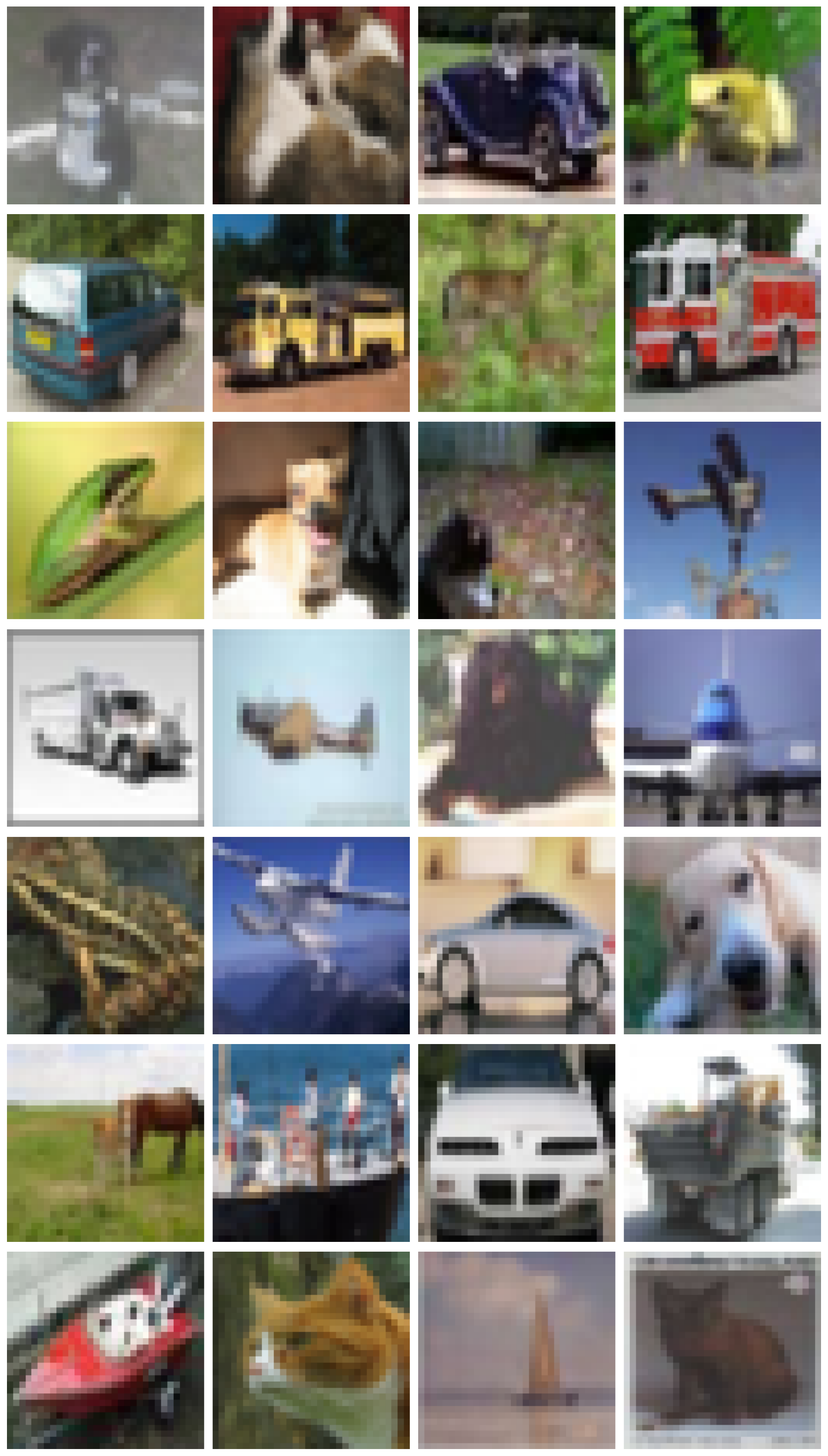}
    \caption{No noise}
  \end{subfigure}\hfill
  \begin{subfigure}[t]{0.3\linewidth}
    \centering
    \includegraphics[width=\linewidth, height=1.5\linewidth]{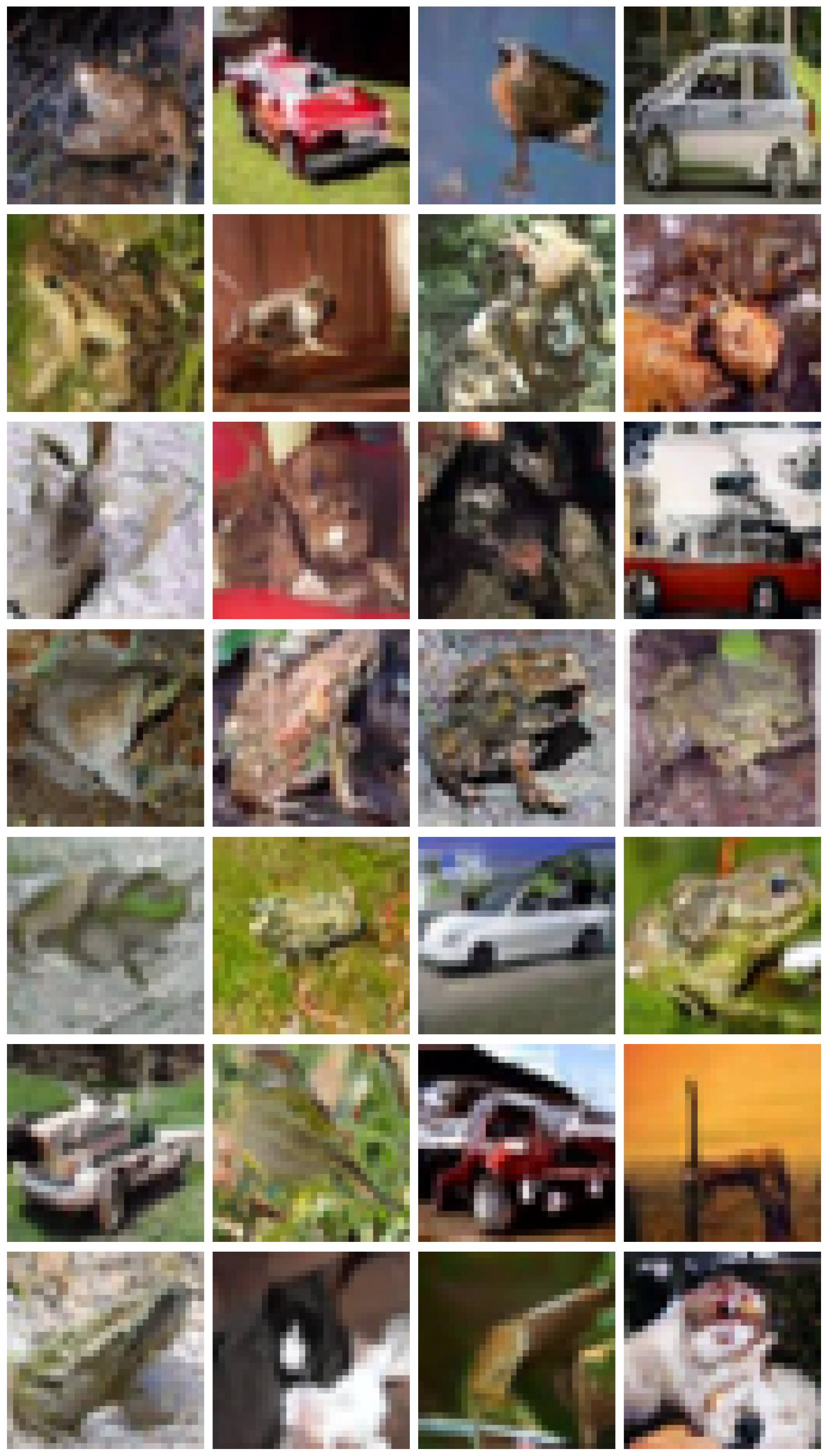}
    \caption{Gaussian noise, $\sigma=1$}
  \end{subfigure}\hfill
  \begin{subfigure}[t]{0.3\linewidth}
    \centering
    \includegraphics[width=\linewidth, height=1.5\linewidth]{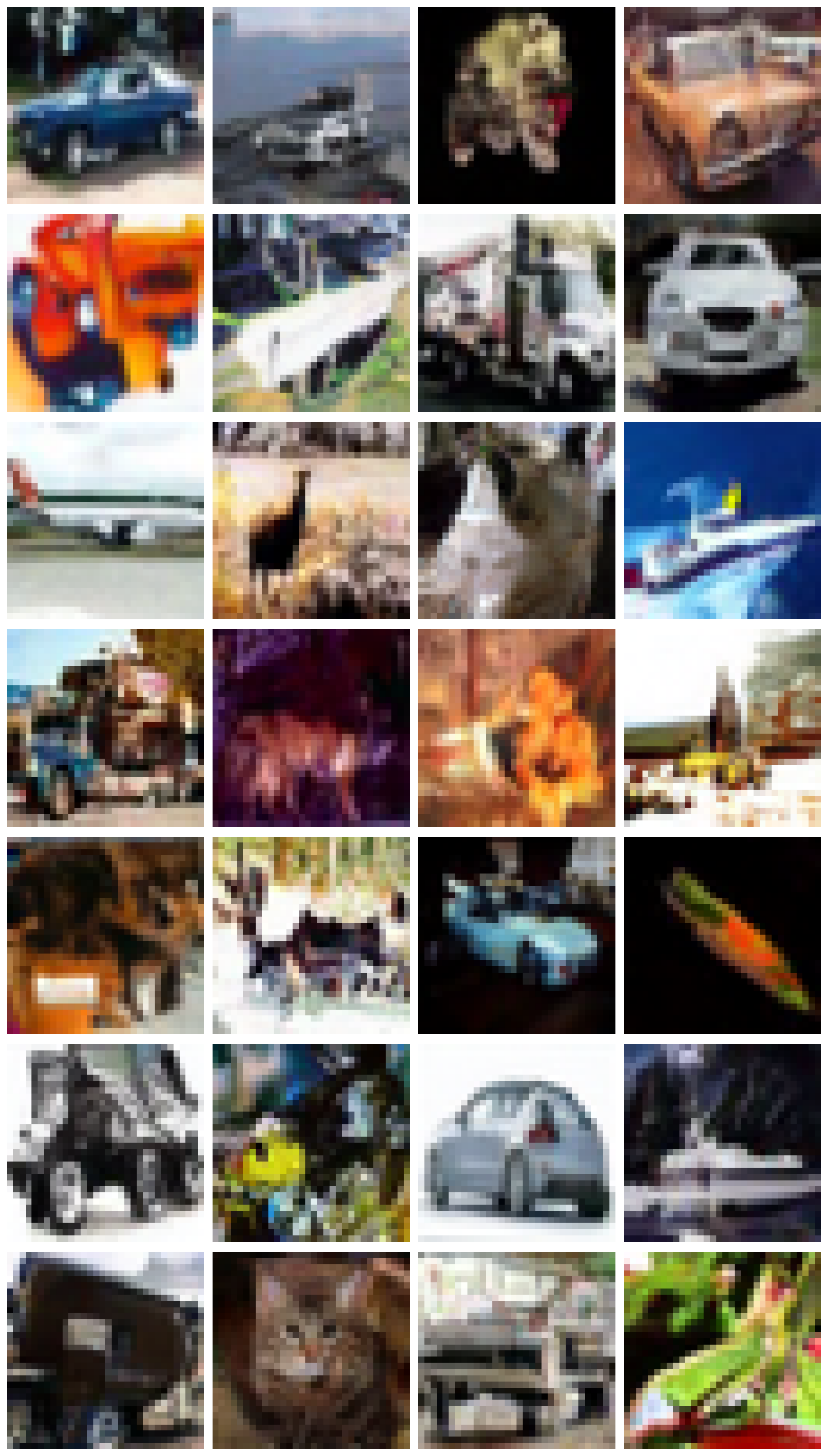}
    \caption{Gaussian noise, $\sigma=2$}
      \vspace{0.5em} 

  \end{subfigure}
  \vspace{0.5em} 
  \begin{subfigure}[t]{0.3\linewidth}
    \centering
    \includegraphics[width=\linewidth, height=1.5\linewidth]{figs/cifar_no_noise.pdf}
    \caption{No noise}
  \end{subfigure}\hfill
  \begin{subfigure}[t]{0.3\linewidth}
    \centering
    \includegraphics[width=\linewidth, height=1.5\linewidth]{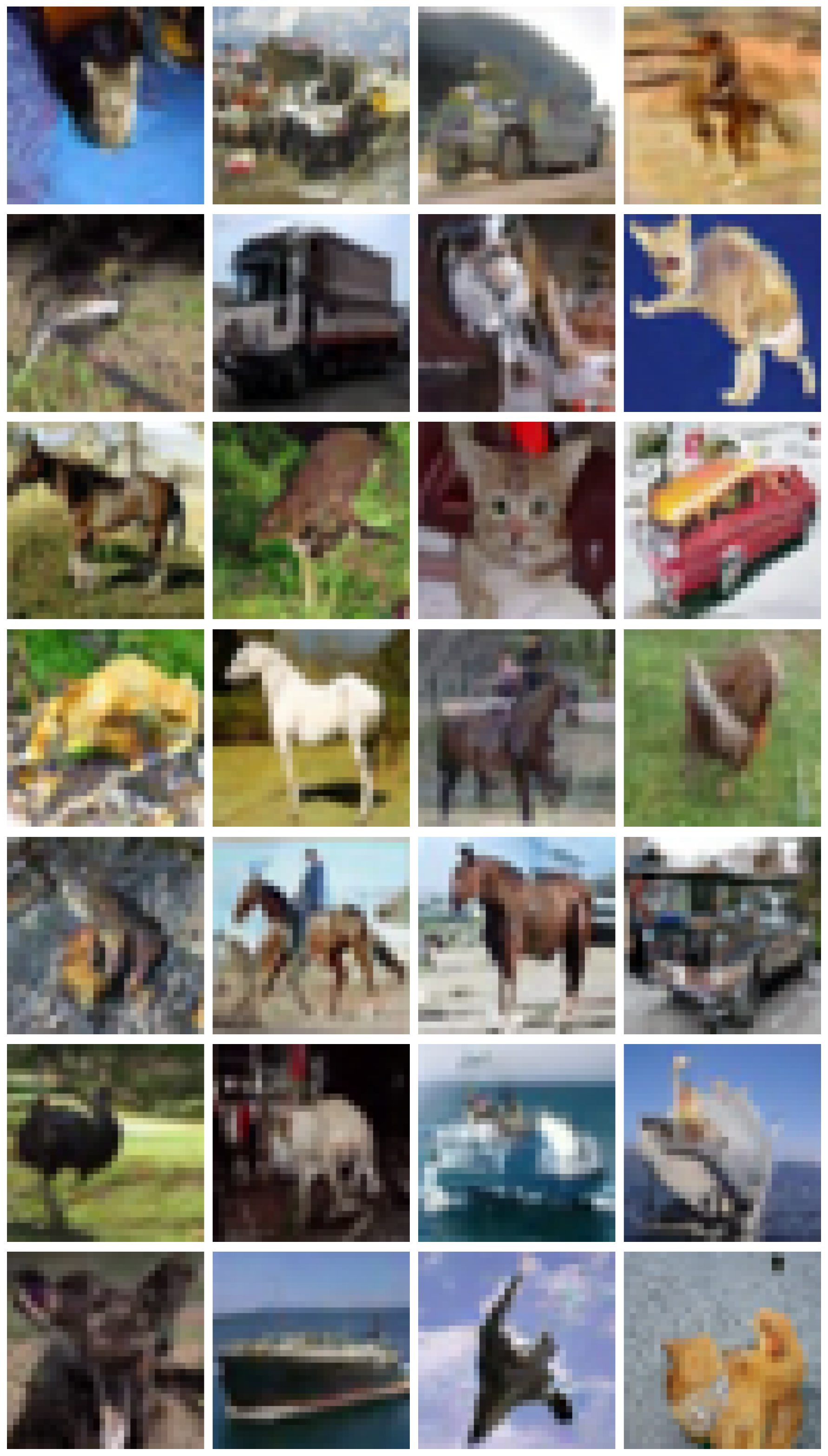}
    \caption{Laplace, $\sigma=1$}
  \end{subfigure}\hfill
  \begin{subfigure}[t]{0.3\linewidth}
    \centering
    \includegraphics[width=\linewidth, height=1.5\linewidth]{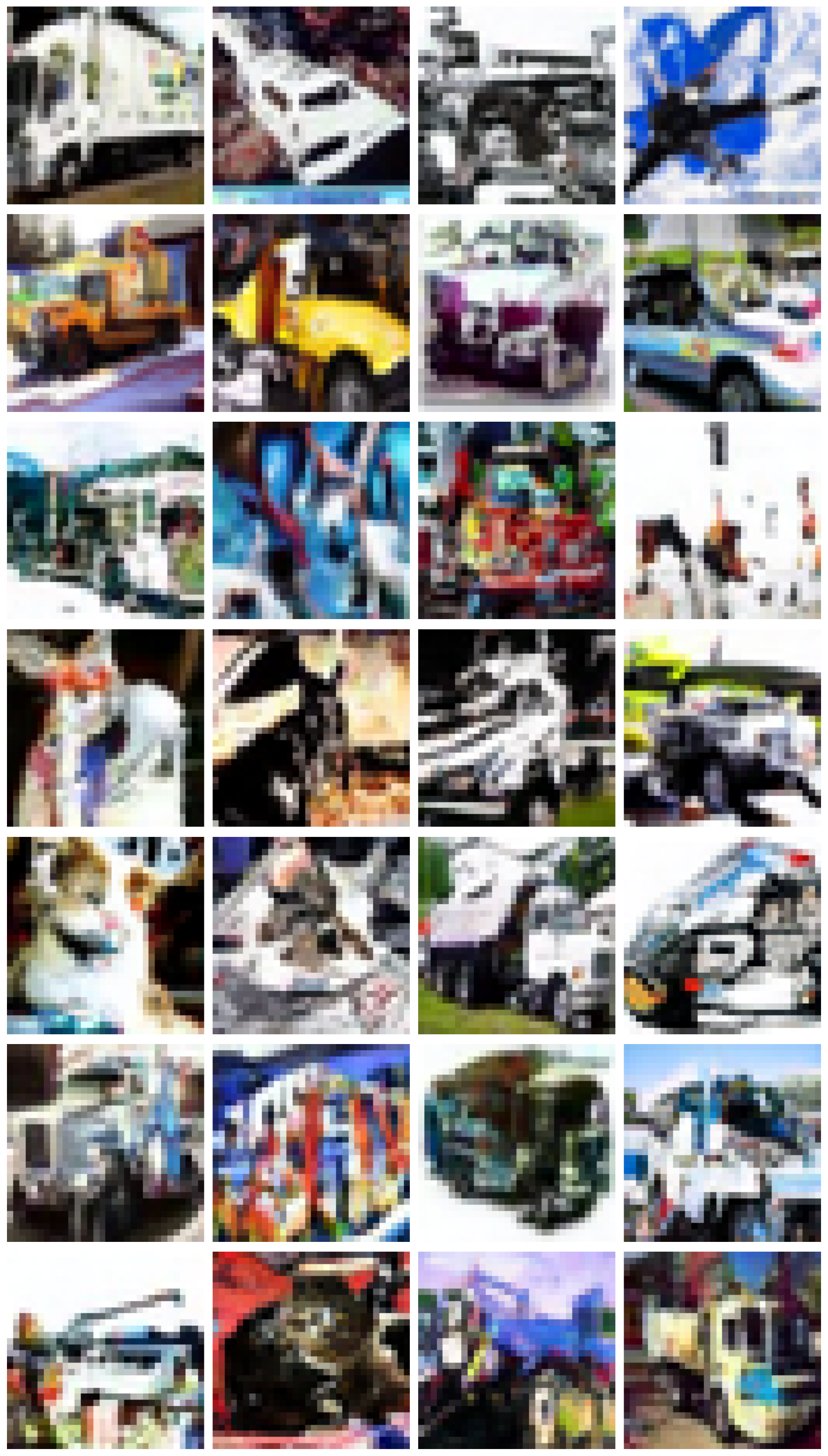}
    \caption{Laplace, $\sigma=2$}
  \end{subfigure}
  \vspace{0.5em} 
  \begin{subfigure}[t]{0.3\linewidth}
    \centering
    \includegraphics[width=\linewidth, height=1.5\linewidth]{figs/cifar_no_noise.pdf}
    \caption{No noise}
  \end{subfigure}\hfill
  \begin{subfigure}[t]{0.3\linewidth}
    \centering
    \includegraphics[width=\linewidth, height=1.5\linewidth]{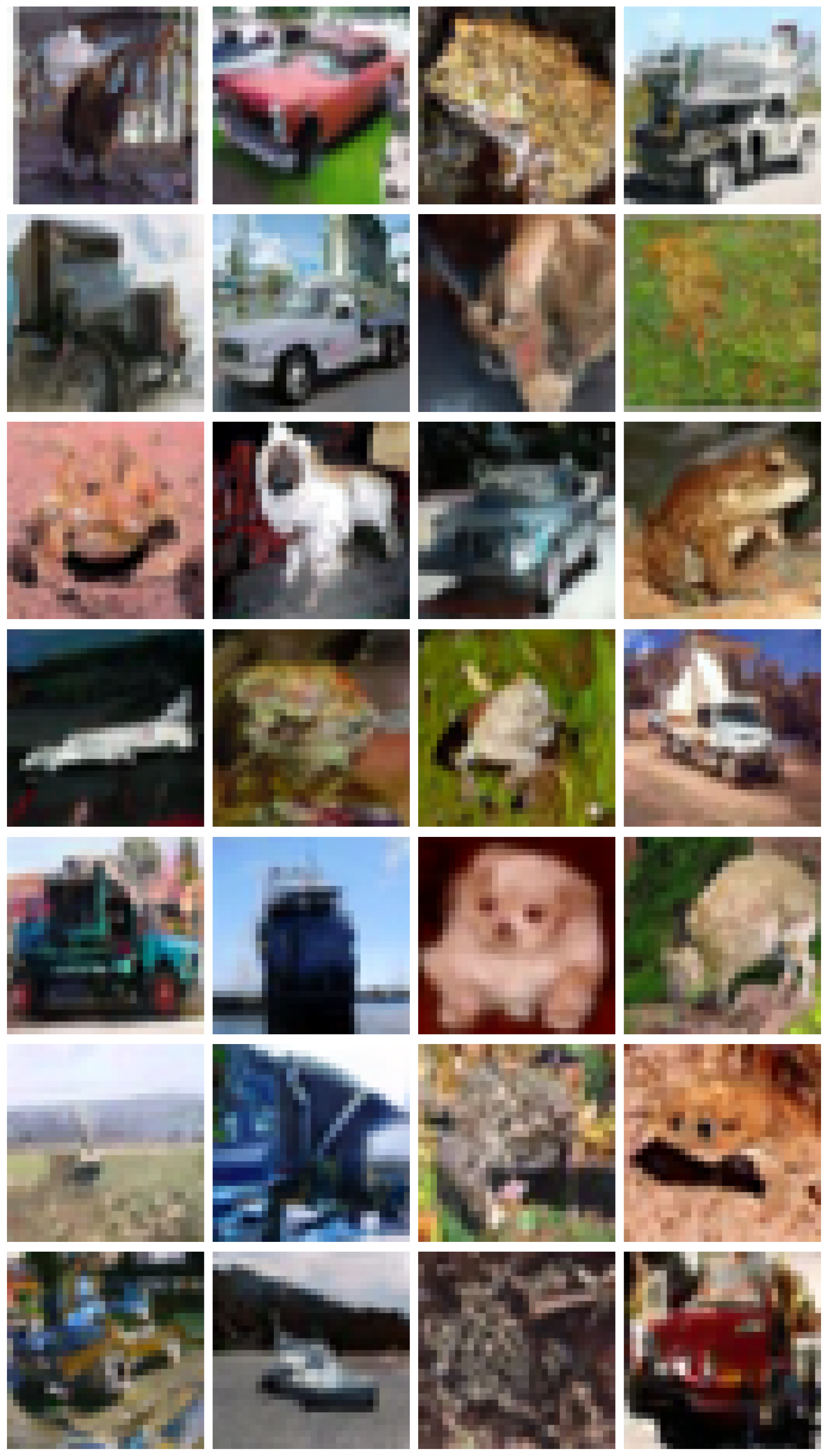}
    \caption{Student's $t$, $\sigma=1$}
  \end{subfigure}\hfill
  \begin{subfigure}[t]{0.3\linewidth}
    \centering
    \includegraphics[width=\linewidth, height=1.5\linewidth]{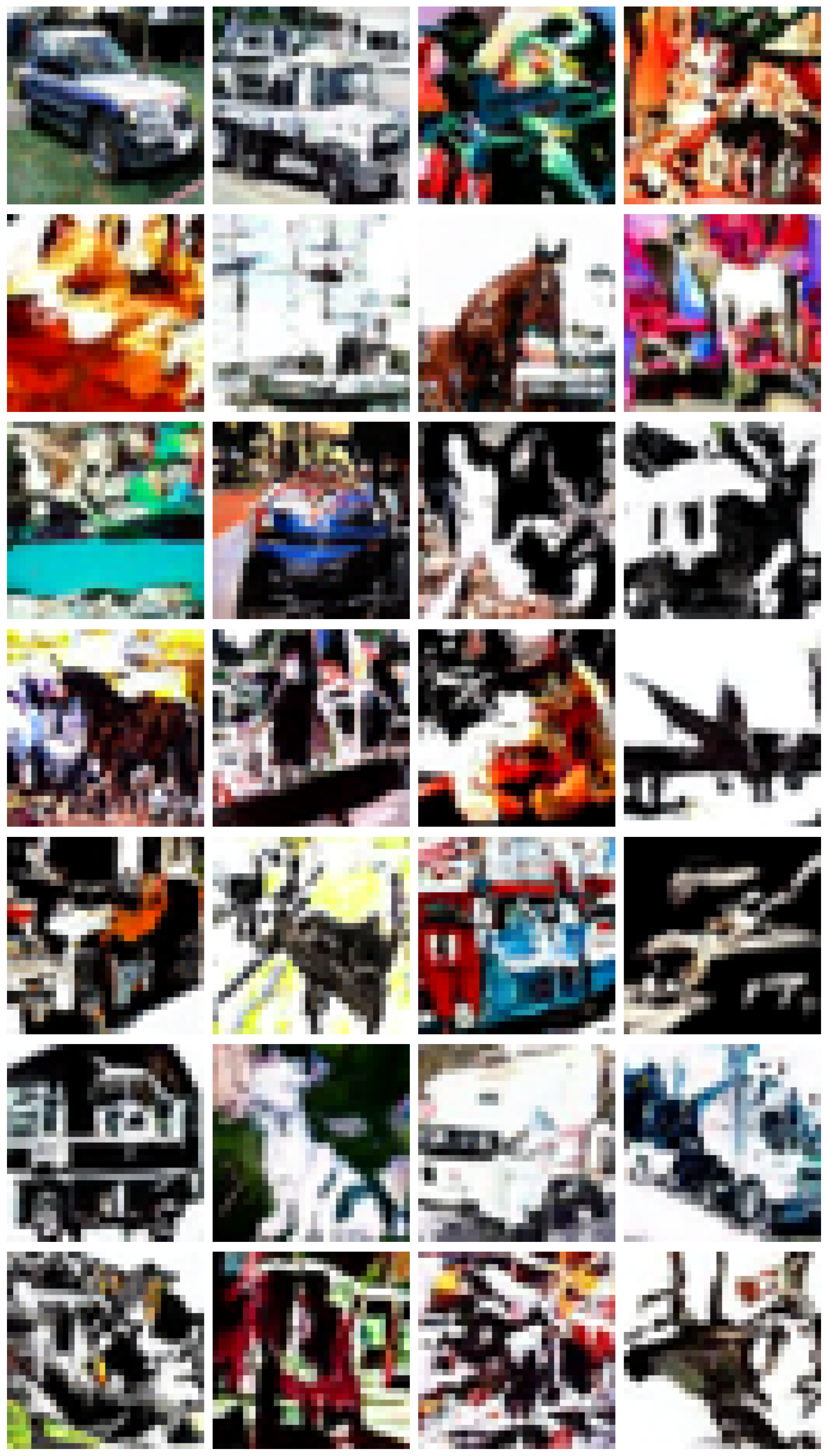}
    \caption{Student's $t$, $\sigma=2$}
  \end{subfigure}

  \caption{Additional CIFAR-10 generations for 3 noise families 
  (rows) and 2 noise levels (columns).}
  \label{fig:qual_cifar}
\end{figure}

\begin{figure}[p]
  \centering
  \begin{subfigure}[t]{0.3\linewidth}
    \centering
    \includegraphics[width=\linewidth, height=1.5\linewidth]{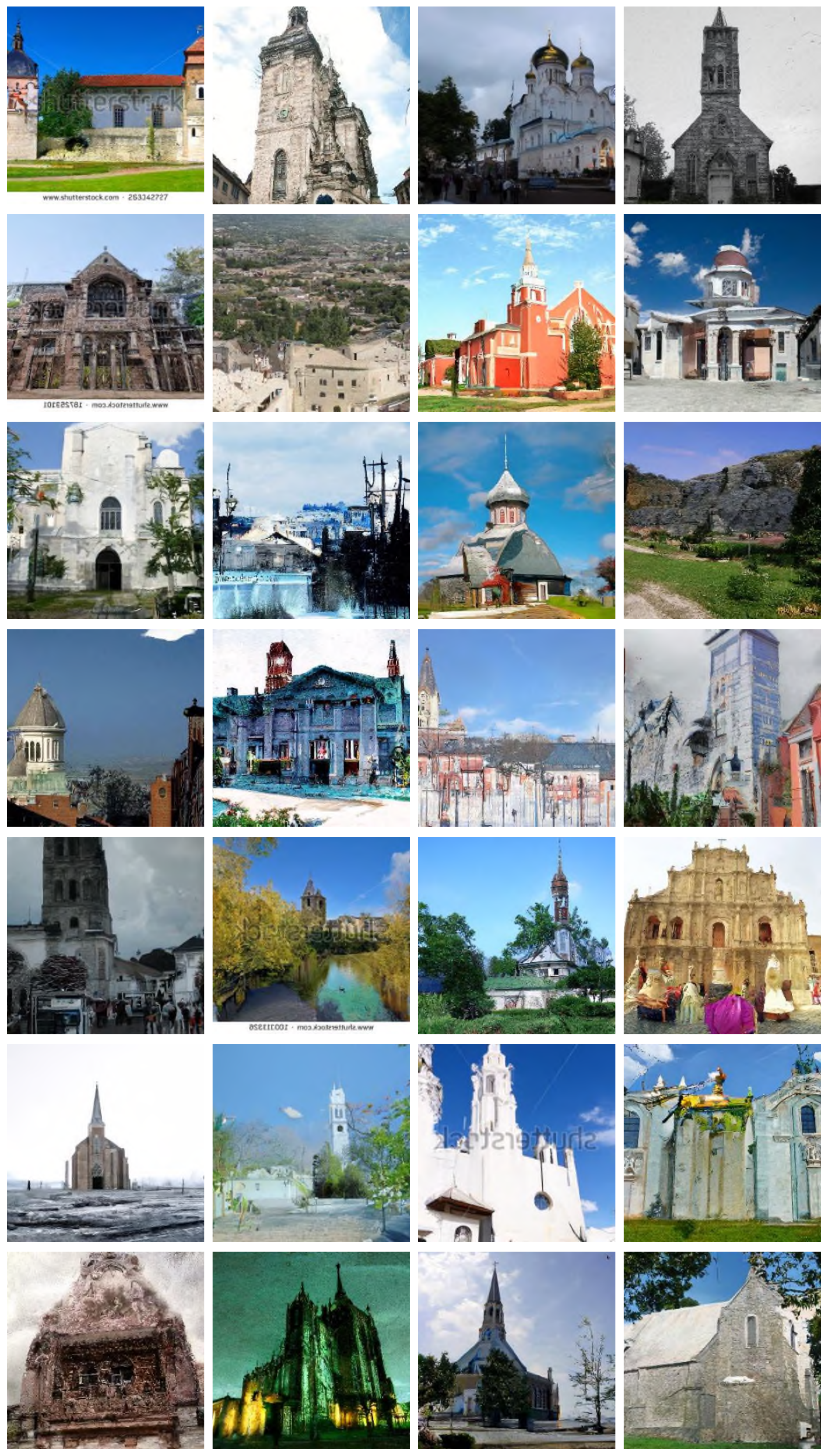}
    \caption{No noise}
  \end{subfigure}\hfill
  \begin{subfigure}[t]{0.3\linewidth}
    \centering
    \includegraphics[width=\linewidth, height=1.5\linewidth]{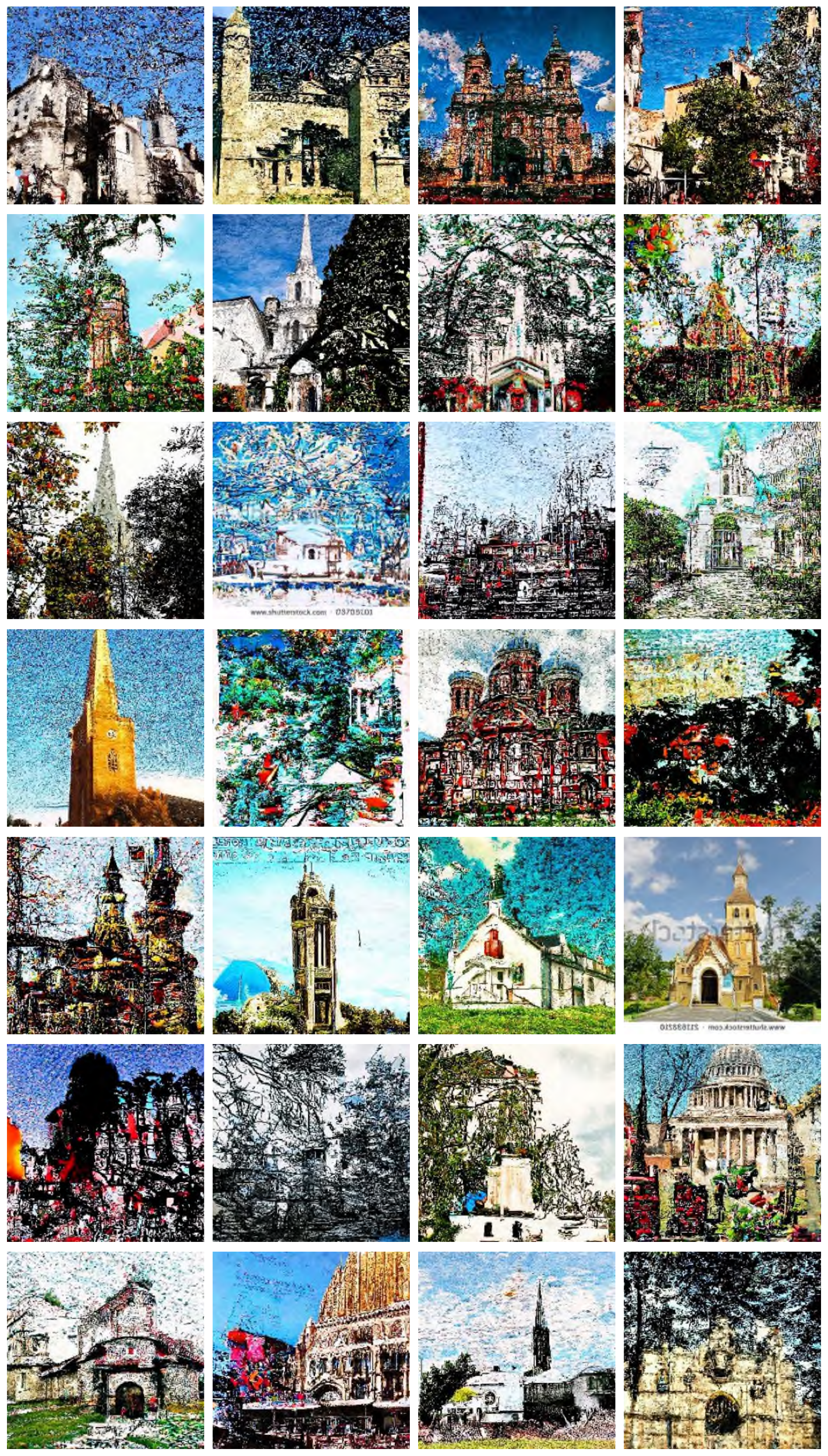}
    \caption{Gaussian noise, $\sigma=1$}
  \end{subfigure}\hfill
  \begin{subfigure}[t]{0.3\linewidth}
    \centering
    \includegraphics[width=\linewidth, height=1.5\linewidth]{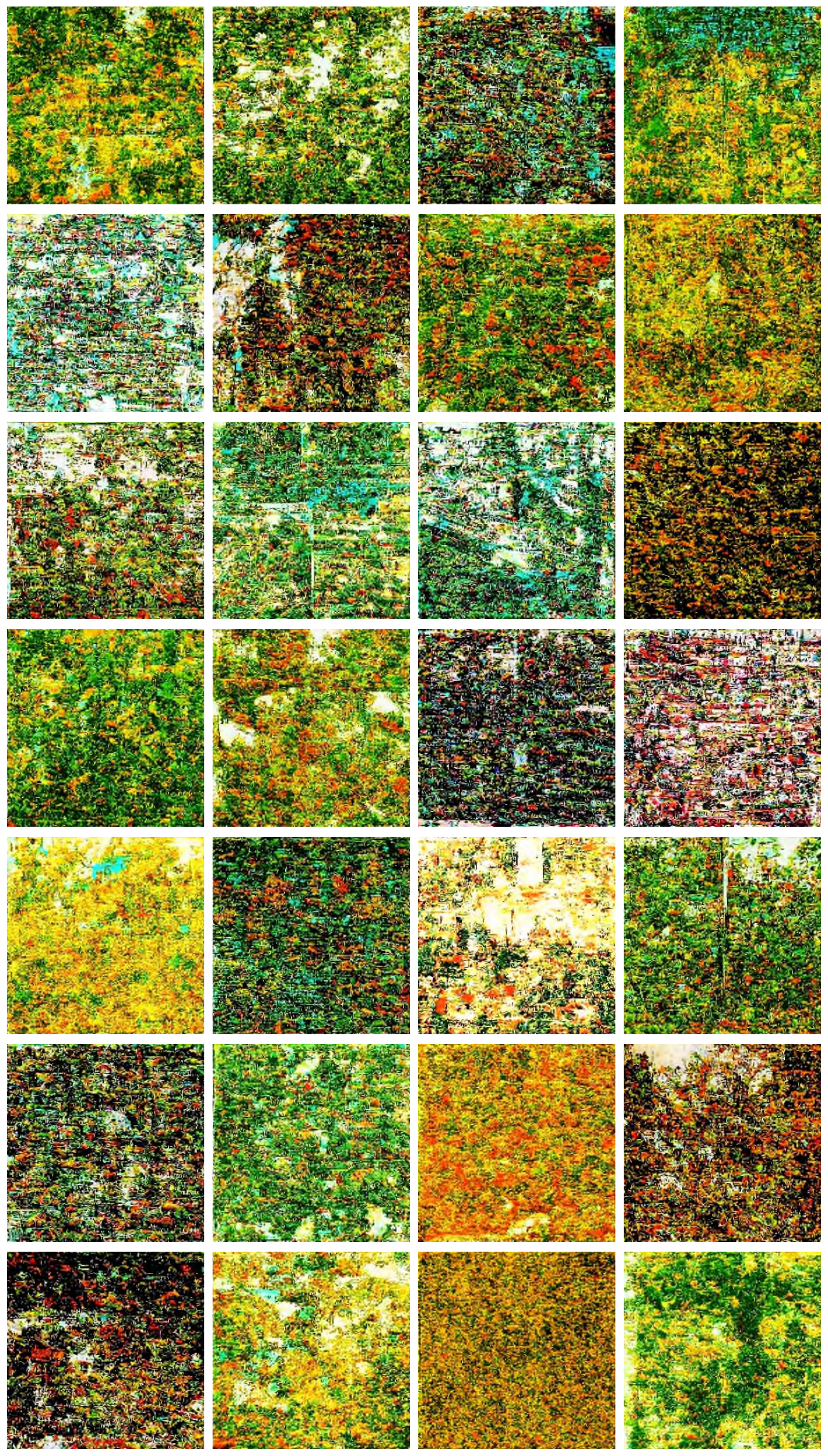}
    \caption{Gaussian noise, $\sigma=2$}
      \vspace{0.5em} 

  \end{subfigure}
  \vspace{0.5em} 
  \begin{subfigure}[t]{0.3\linewidth}
    \centering
    \includegraphics[width=\linewidth, height=1.5\linewidth]{figs/church_no_noise_small.pdf}
    \caption{No noise}
  \end{subfigure}\hfill
  \begin{subfigure}[t]{0.3\linewidth}
    \centering
    \includegraphics[width=\linewidth, height=1.5\linewidth]{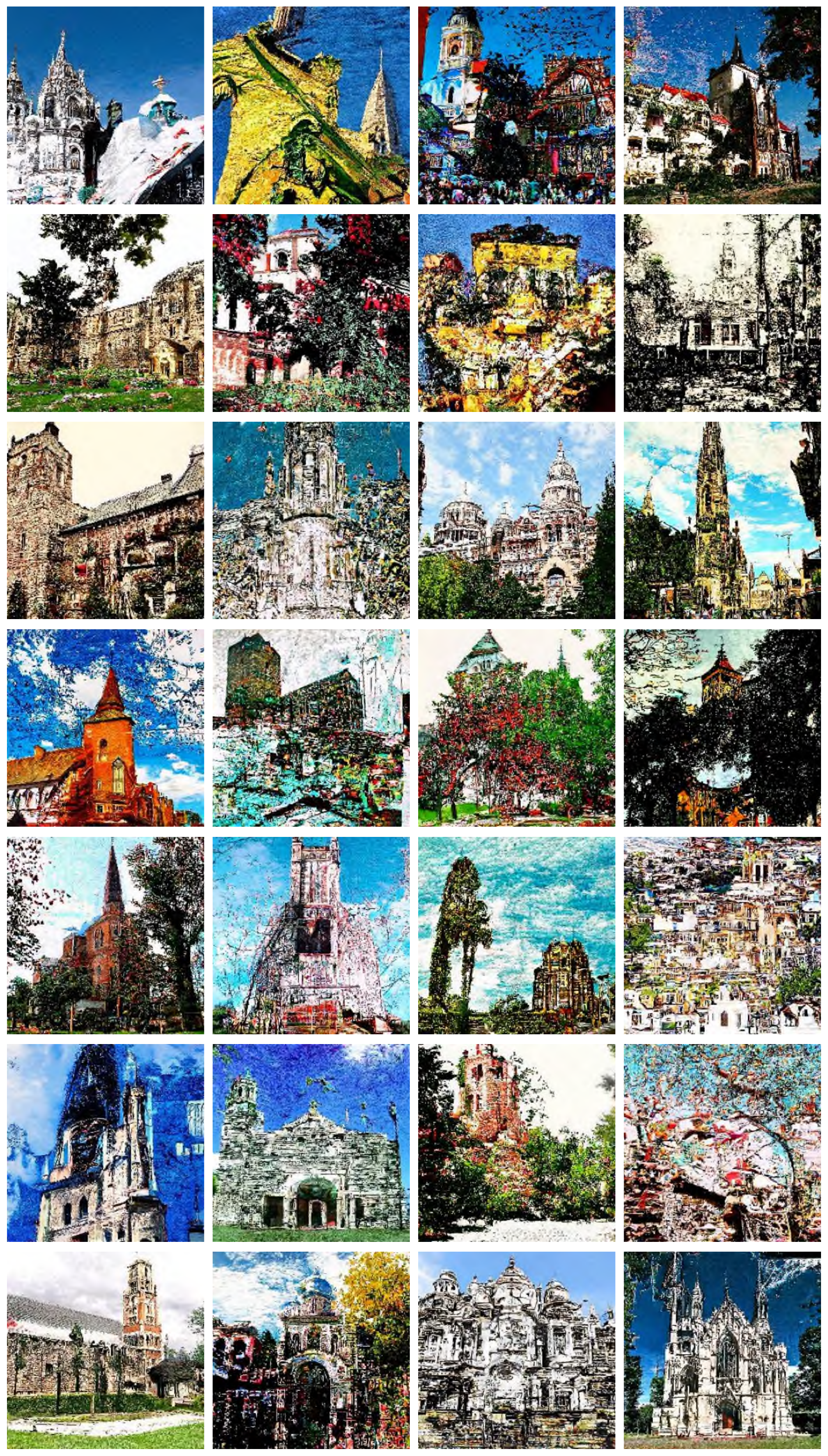}
    \caption{Laplace, $\sigma=1$}
  \end{subfigure}\hfill
  \begin{subfigure}[t]{0.3\linewidth}
    \centering
    \includegraphics[width=\linewidth, height=1.5\linewidth]{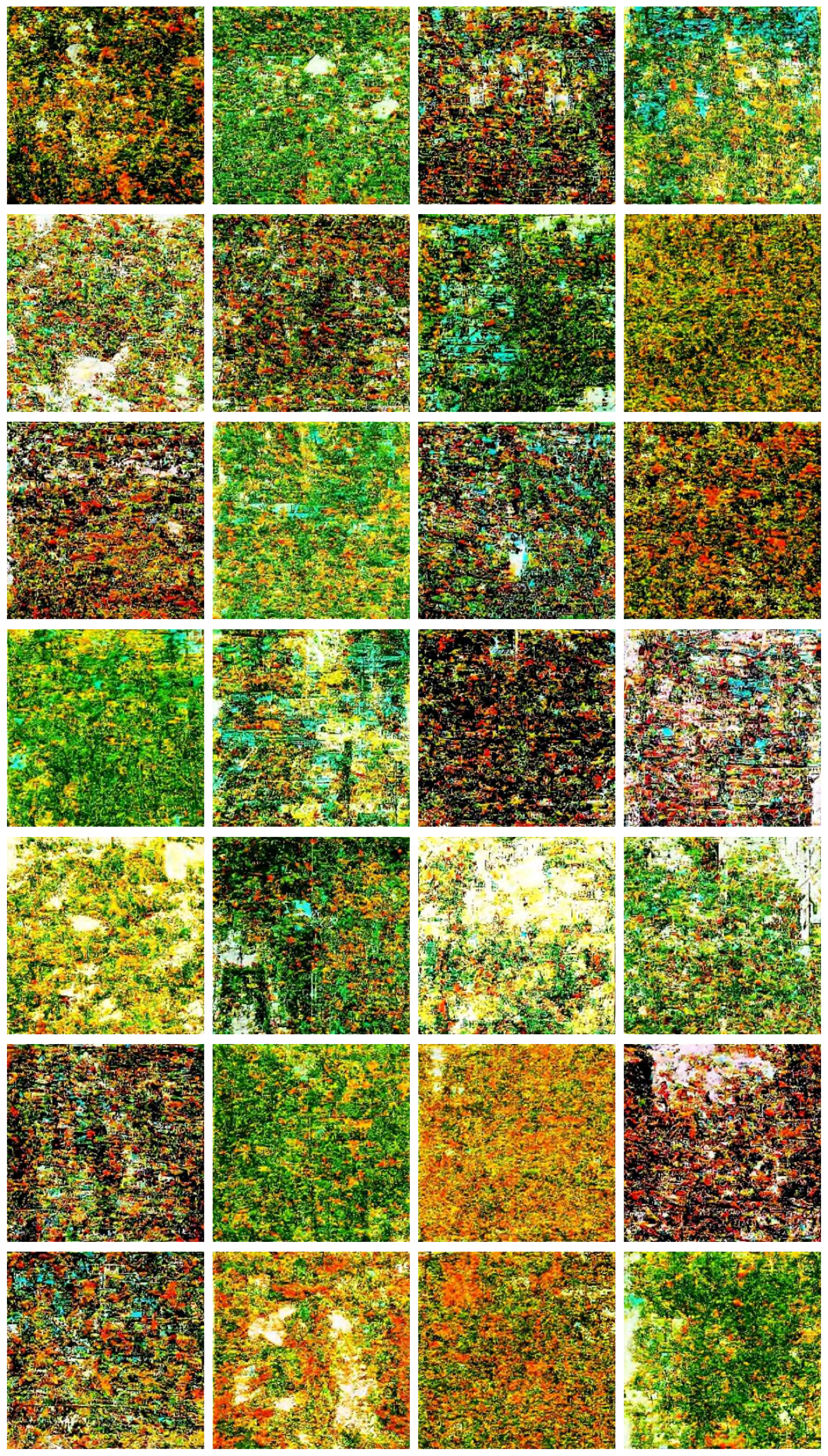}
    \caption{Laplace, $\sigma=2$}
  \end{subfigure}
  \vspace{0.5em} 
  \begin{subfigure}[t]{0.3\linewidth}
    \centering
    \includegraphics[width=\linewidth, height=1.5\linewidth]{figs/church_no_noise_small.pdf}
    \caption{No noise}
  \end{subfigure}\hfill
  \begin{subfigure}[t]{0.3\linewidth}
    \centering
    \includegraphics[width=\linewidth, height=1.5\linewidth]{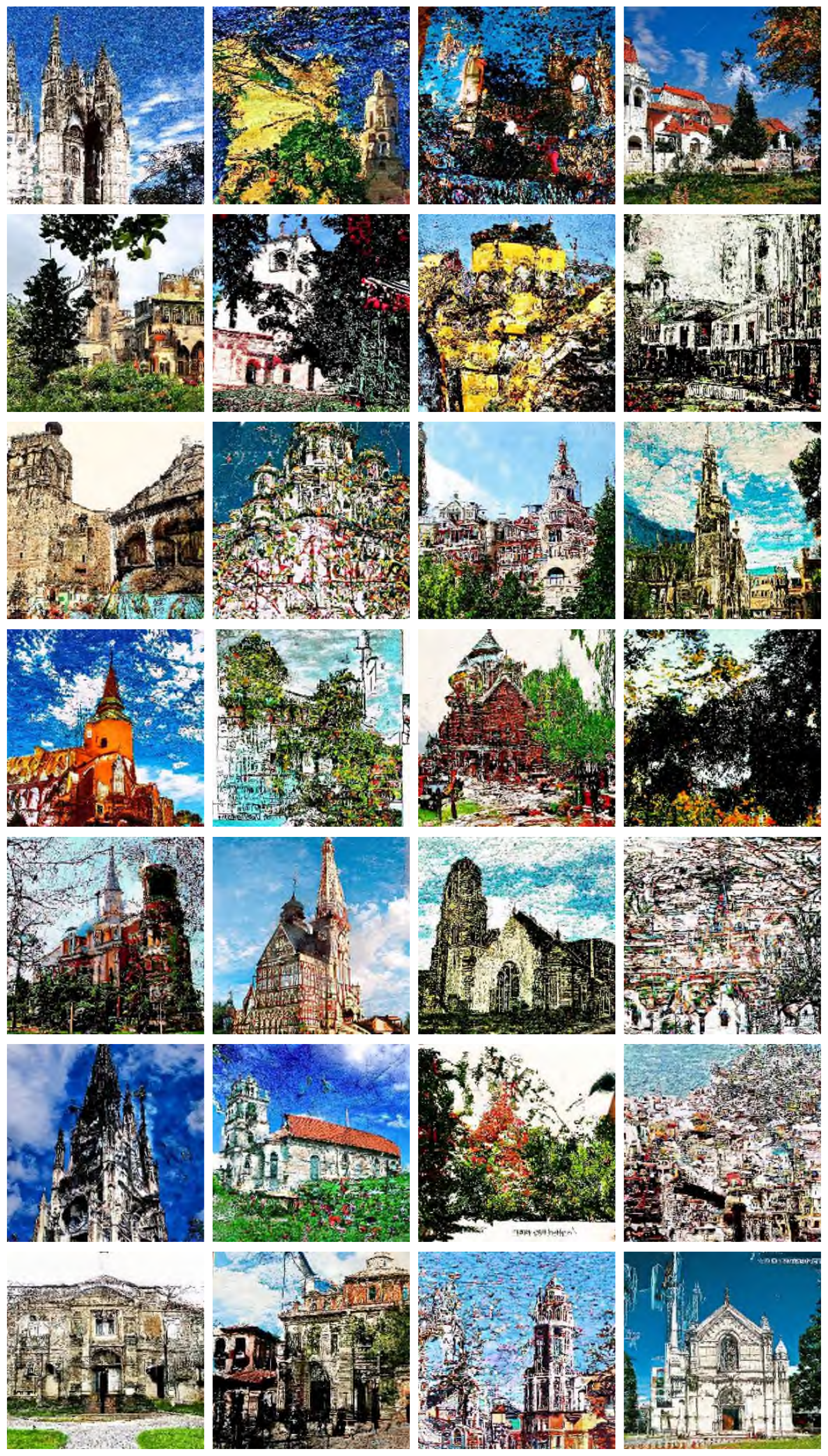}
    \caption{Student's $t$, $\sigma=1$}
  \end{subfigure}\hfill
  \begin{subfigure}[t]{0.3\linewidth}
    \centering
    \includegraphics[width=\linewidth, height=1.5\linewidth]{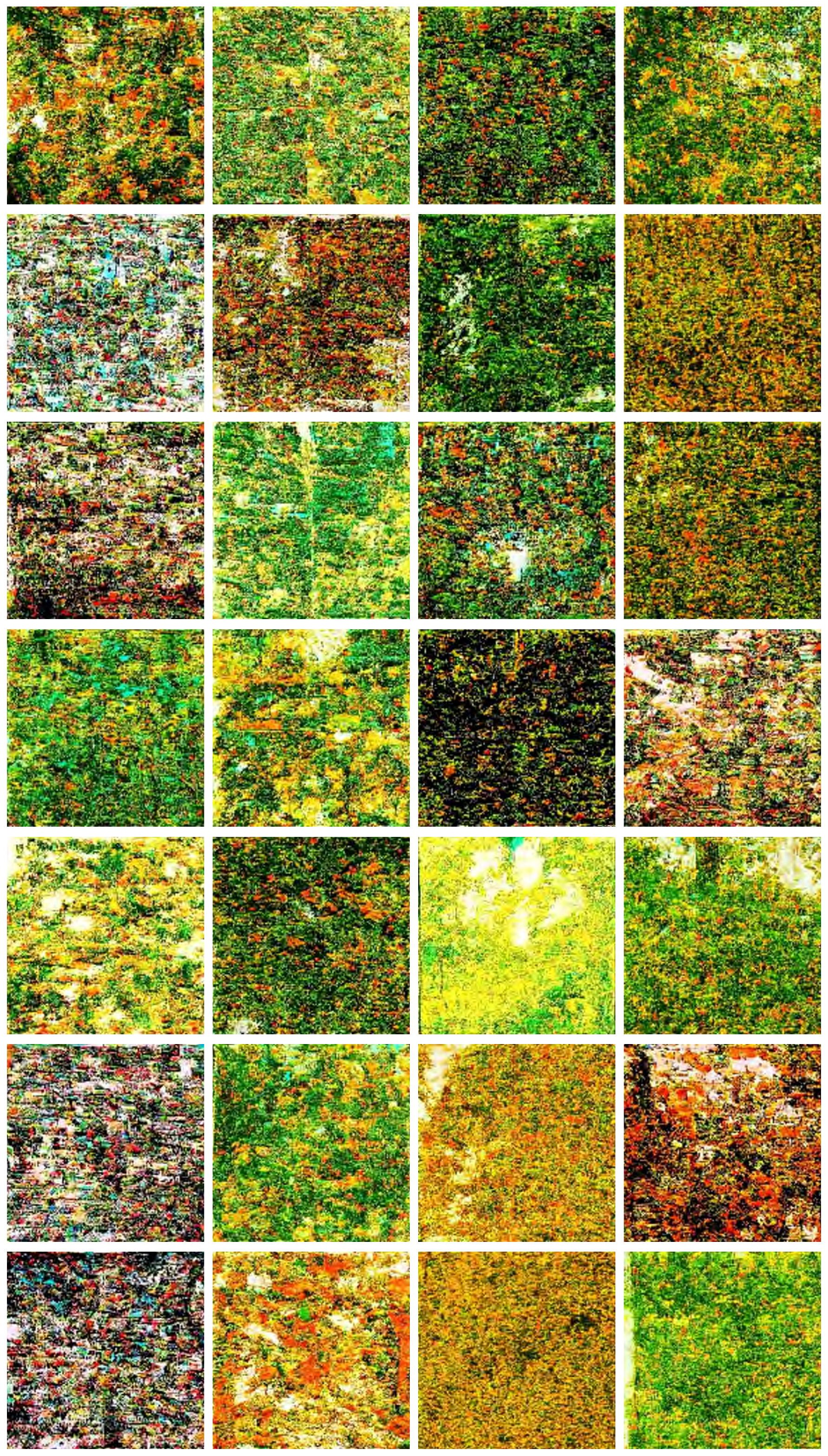}
    \caption{Student's $t$, $\sigma=2$}
  \end{subfigure}

  \caption{Additional LSUN-Church generations for 3 noise 
  families (rows) and 2 noise levels (columns).}
  \label{fig:qual_LSUN}
\end{figure}

\begin{figure}[p]
  \centering
  \begin{subfigure}[t]{0.3\linewidth}
    \centering
    \includegraphics[width=\linewidth, height=1.5\linewidth]{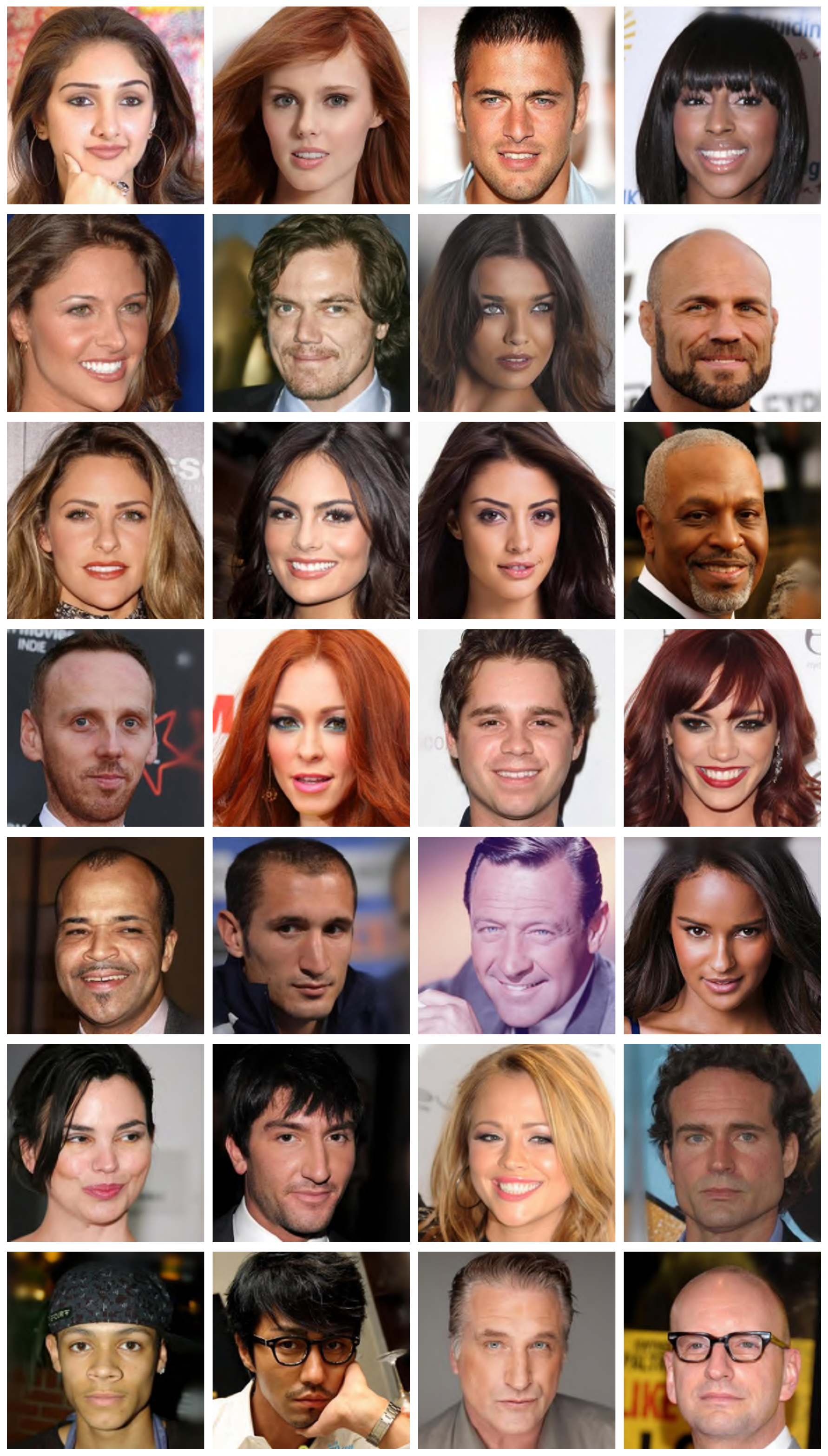}
    \caption{No noise}
  \end{subfigure}\hfill
  \begin{subfigure}[t]{0.3\linewidth}
    \centering
    \includegraphics[width=\linewidth, height=1.5\linewidth]{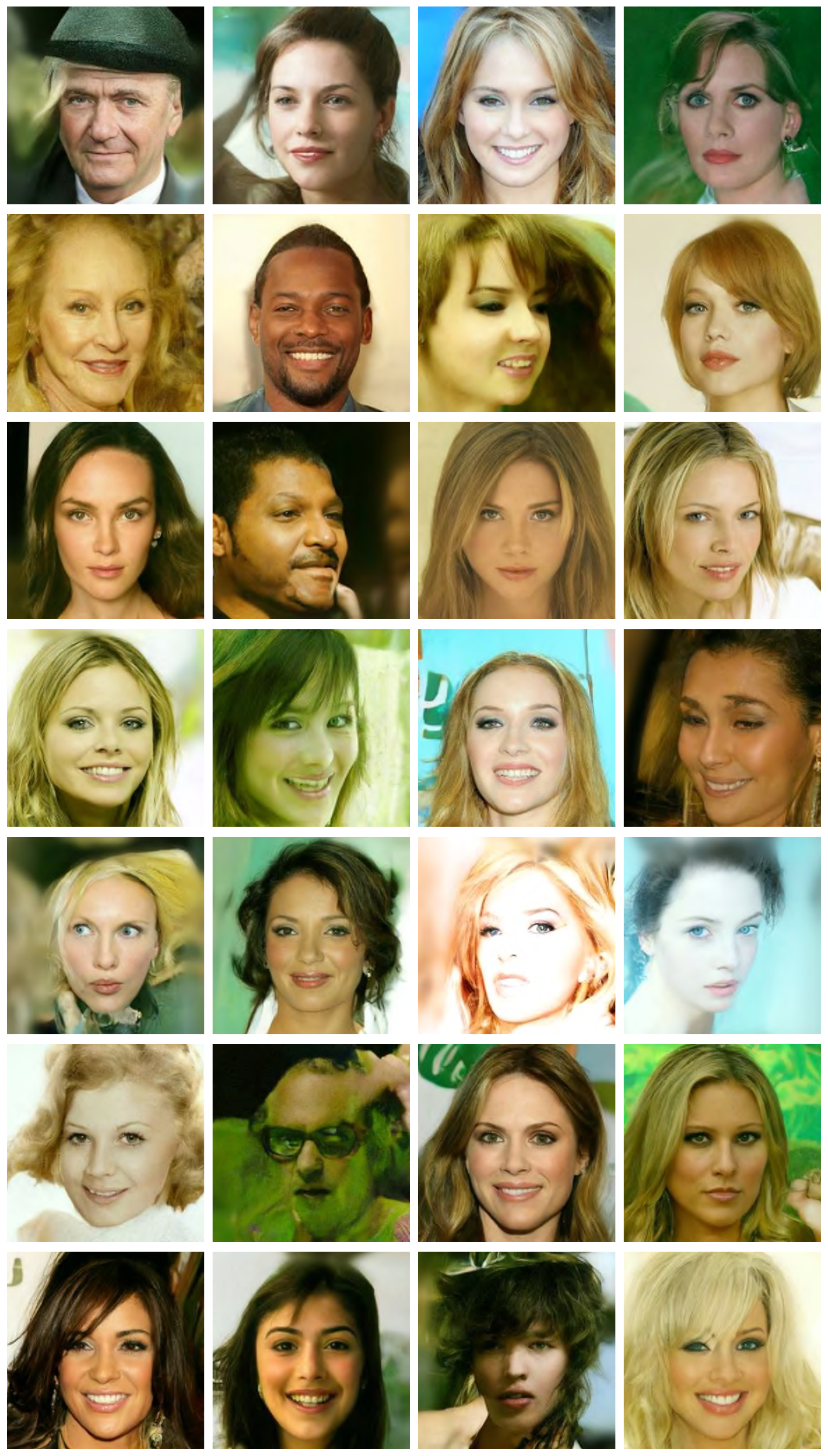}
    \caption{Gaussian noise, $\sigma=0.5$}
  \end{subfigure}\hfill
  \begin{subfigure}[t]{0.3\linewidth}
    \centering
    \includegraphics[width=\linewidth, height=1.5\linewidth]{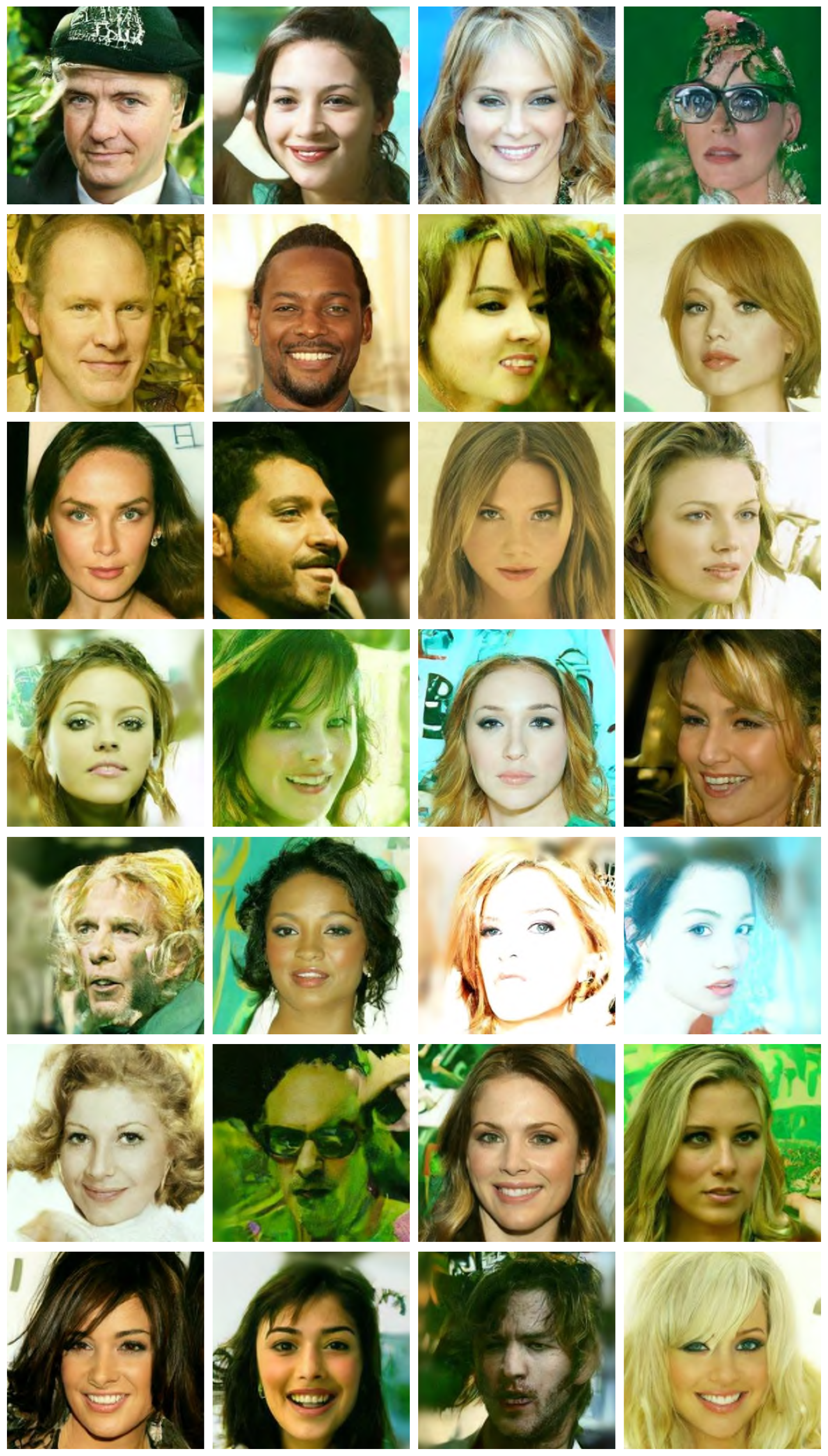}
    \caption{Gaussian noise, $\sigma=1$}
      \vspace{0.5em} 

  \end{subfigure}
  \vspace{0.5em} 
  \begin{subfigure}[t]{0.3\linewidth}
    \centering
    \includegraphics[width=\linewidth, height=1.5\linewidth]{figs/celeba_no_noise_small.pdf}
    \caption{No noise}
  \end{subfigure}\hfill
  \begin{subfigure}[t]{0.3\linewidth}
    \centering
    \includegraphics[width=\linewidth, height=1.5\linewidth]{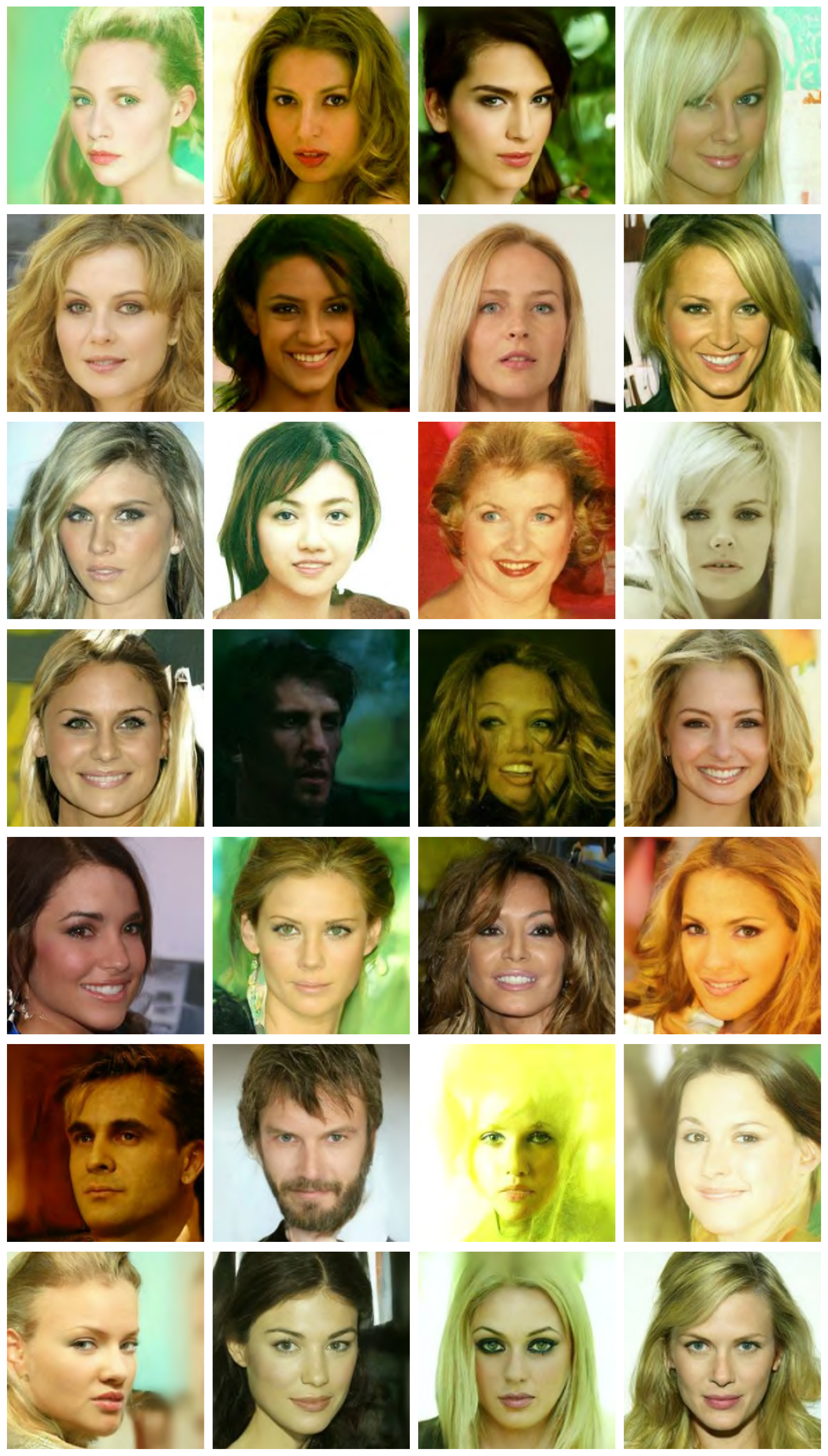}
    \caption{Laplace, $\sigma=0.5$}
  \end{subfigure}\hfill
  \begin{subfigure}[t]{0.3\linewidth}
    \centering
    \includegraphics[width=\linewidth, height=1.5\linewidth]{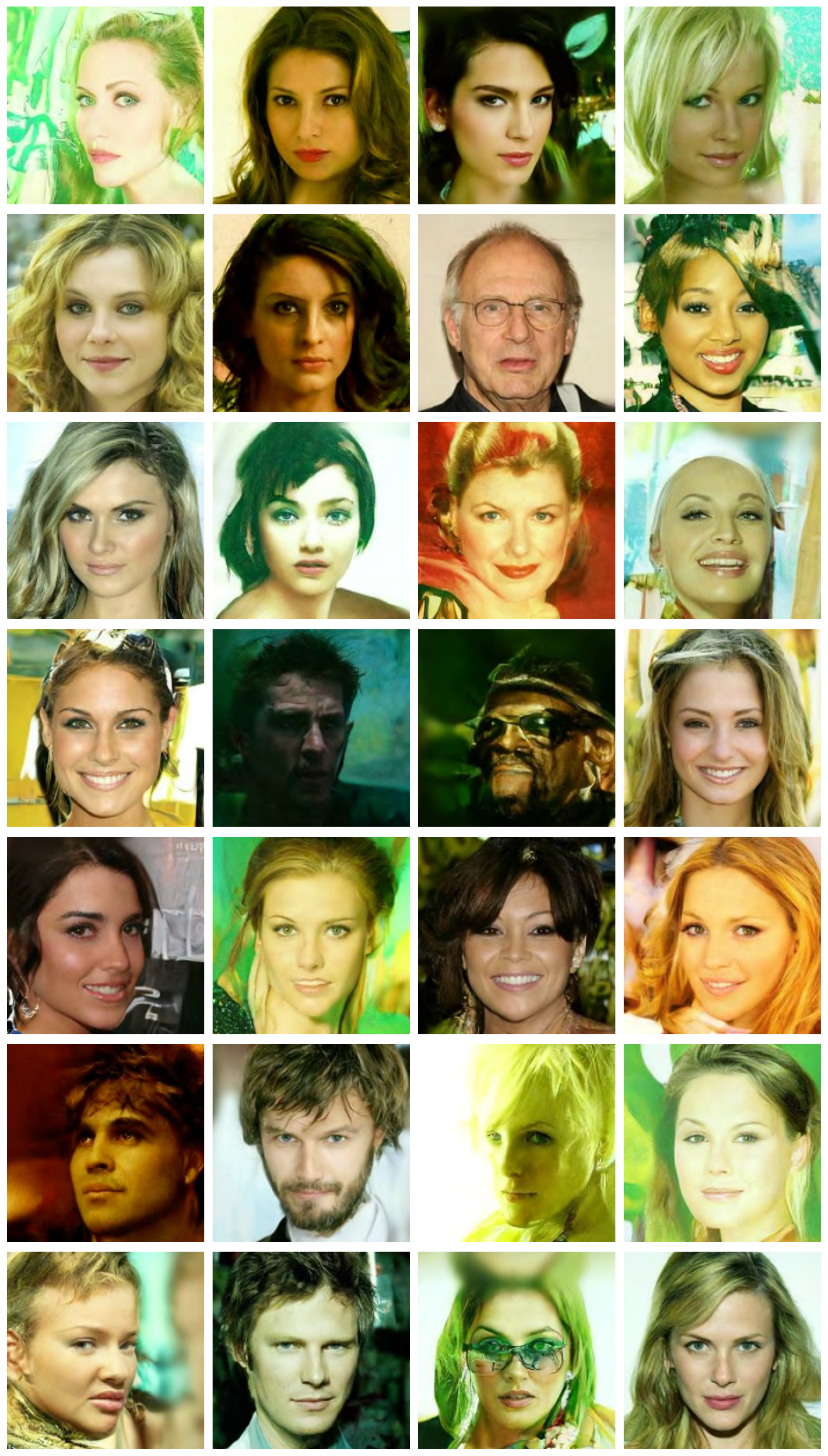}
    \caption{Laplace, $\sigma=1$}
  \end{subfigure}
  \vspace{0.5em} 
  \begin{subfigure}[t]{0.3\linewidth}
    \centering
    \includegraphics[width=\linewidth, height=1.5\linewidth]{figs/celeba_no_noise_small.pdf}
    \caption{No noise}
  \end{subfigure}\hfill
  \begin{subfigure}[t]{0.3\linewidth}
    \centering
    \includegraphics[width=\linewidth, height=1.5\linewidth]{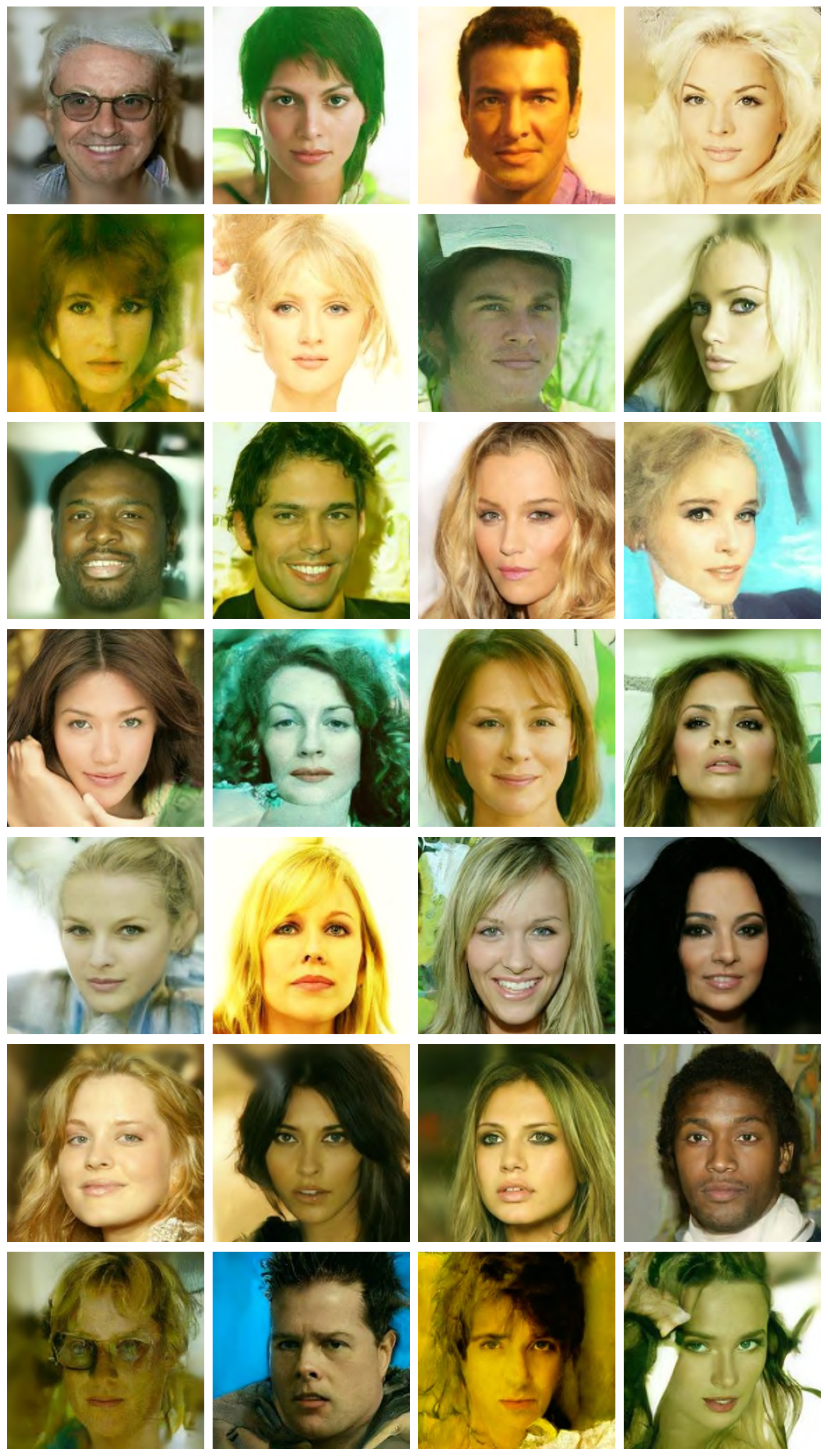}
    \caption{Student's $t$, $\sigma=0.5$}
  \end{subfigure}\hfill
  \begin{subfigure}[t]{0.3\linewidth}
    \centering
    \includegraphics[width=\linewidth, height=1.5\linewidth]{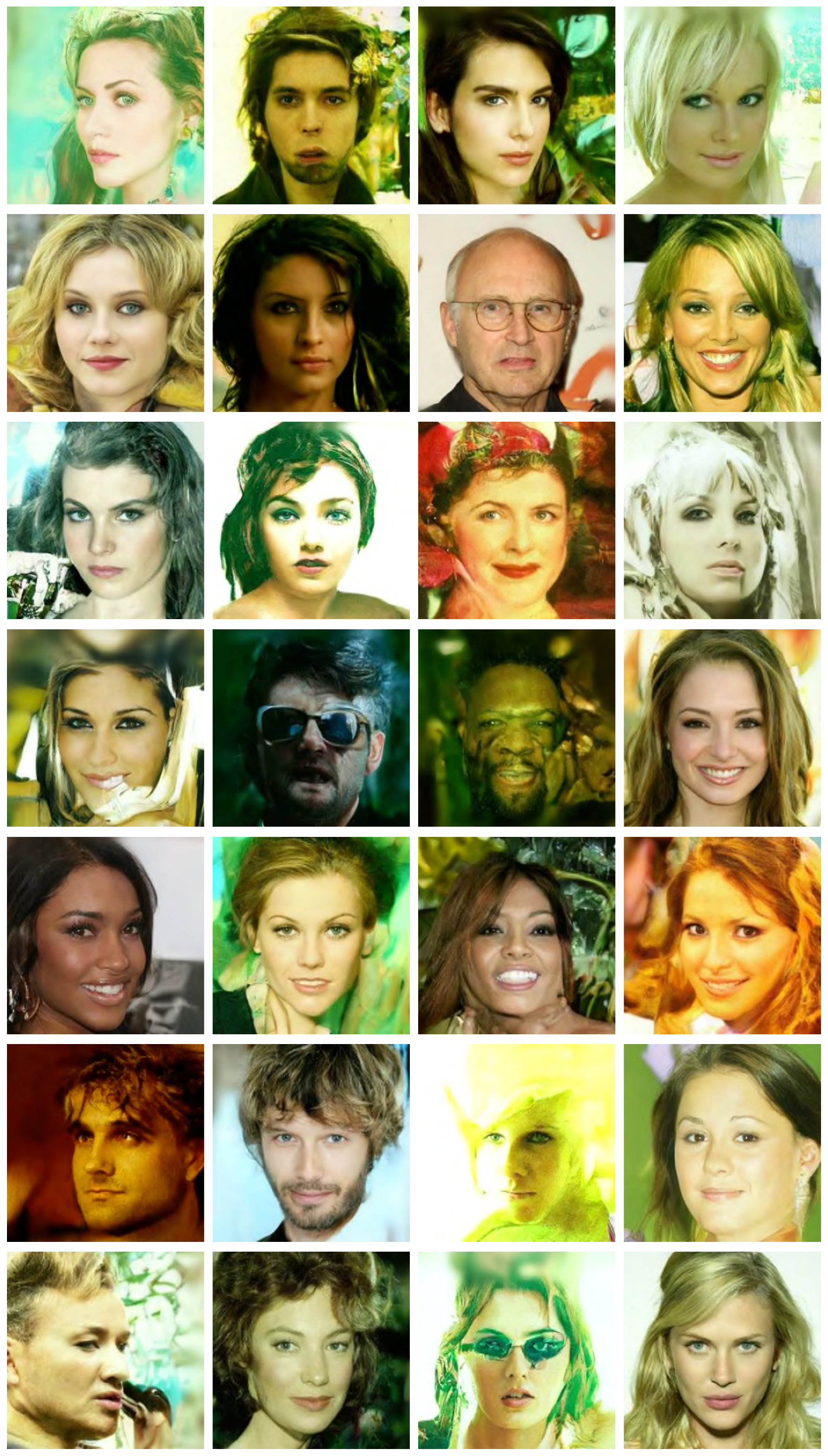}
    \caption{Student's $t$, $\sigma=1$}
  \end{subfigure}

  \caption{Additional CelebA-HQ generations for 3 noise families (rows) and 2 noise levels (columns).}
  \label{fig:qual_celeba}
\end{figure}

\end{document}